%% file: main.tex
\documentclass[letterpaper]{article} 
\usepackage{aaai23} 
\usepackage{times}  
\usepackage{helvet}  
\usepackage{courier}  
\usepackage[hyphens]{url}  
\usepackage{graphicx} 
\usepackage{natbib}  
\usepackage{caption} 
\usepackage{algorithm}
\usepackage{algorithmic}

\usepackage{newfloat}
\usepackage{listings}
\DeclareCaptionStyle{ruled}{labelfont=normalfont,labelsep=colon,strut=off} 
\floatstyle{ruled}
\newfloat{listing}{tb}{lst}{}
\floatname{listing}{Listing}
\usepackage{cite}
\usepackage{amsmath,amssymb,amsfonts}
\usepackage[linguistics]{forest}
\usepackage{tabularx}
\usepackage{textcomp}
\usepackage{xcolor,colortbl}
\usepackage{amsthm}
\usepackage{booktabs}
\usepackage{makecell}
\usepackage{subcaption}
\usepackage{listings}
\def\x{\mathbf{x}}
\def\y{\mathbf{y}}
\def\w{\mathbf{w}}

\def\X{\mathbf{X}}

\newcommand{\capt}{{\textrm{cap}}}
\newcommand{\fix}{{\textrm{\rm fix}}}
\newcommand{\splita}{{\textrm{\rm split}}}
\newcommand{\lf}{{\textrm{left}}}
\newcommand{\ri}{{\textrm{right}}}
\newcommand{\modify}{{\textrm{mod}}}

\newtheorem{theorem}{Theorem}[section]

\newtheorem{definition}{Definition}[section]

\newcommand{\ourmethod}{{\textrm{OSRT}}}
\DeclareMathOperator*{\argmin}{\arg\!\min}

\title{Optimal Sparse Regression Trees
}

\author{
    Rui Zhang\textsuperscript{\rm 1}\equalcontrib,
    Rui Xin\textsuperscript{\rm 1}\equalcontrib,
    Margo Seltzer\textsuperscript{\rm 2},
    Cynthia Rudin\textsuperscript{\rm 1}
}

\affiliations{

    \textsuperscript{\rm 1} Duke University\\
    \textsuperscript{\rm 2} University of British Columbia\\
    
     \{r.zhang, rui.xin926, cynthia.rudin\}@duke.edu, \{mseltzer\}@cs.ubc.ca

}

\begin{document}
\maketitle

\begin{abstract}
Regression trees are one of the oldest forms of AI models, and their predictions can be made without a calculator, which makes them broadly useful, particularly for high-stakes applications. Within the large literature on regression trees, there has been little effort towards full provable optimization, mainly due to the computational hardness of the problem.  This work proposes a dynamic-programming-with-bounds approach to the construction of provably-optimal sparse regression trees. We leverage a novel lower bound based on an optimal solution to the k-Means clustering algorithm on one dimensional data. We are often able to find optimal sparse trees in seconds, even for challenging datasets that involve large numbers of samples and highly-correlated features. 

\end{abstract}

\section{Introduction}

\textit{Regression trees} are one of the oldest and most popular forms of machine learning model, dating back to the 1963 AID algorithm of \citet{MorganSo1963}. Since then, there has been a vast amount of work on regression trees, the overwhelming majority of which involves greedy tree induction and greedy pruning \cite{Breiman84,Quinlan93,payne1977algorithm,loh2002regression}. In these approaches, trees are grown from the top down, with greedy splitting at each branch node, and greedy pruning afterwards. These techniques are easy and fast,
but their trees have no notion of global optimality. Greedily-grown trees can be much larger than necessary, sacrificing interpretability, and their performance suffers when compared to other machine learning approaches. Thus, questions remain -- is it possible to create optimal  regression trees? Would they be competitive with other machine learning algorithms if they were fully optimized? Certainly there would be many uses for sparse interpretable regression trees if we could create them with accuracy comparable to that of other machine learning approaches.

While the quest for fully-optimal decision trees began in the mid-90's with the work of  \citet{Bennett96optimaldecision}, fully optimal decision tree learning was rarely attempted over the last three decades, owing to the computational hardness of the problem. Works that did attempt it \cite{dobkininduction,FarhangfarGZ08,narodytska2018learning,janota2020sat,shati2021sat, hu2020learning,avellaneda2020efficient} had strong constraints, such as shallow depth or perfect classification accuracy.
For classification (rather than regression), scientists have had recent success in producing fully optimal trees \cite{MenickellyGKS18,blanquero2020sparsity,HuRuSe2019,verwer2019learning,AngelinoLaAlSeRu17-kdd,lin2020generalized,McTavishZhongEtAl2022,FarhangfarGZ08,nijssen2007mining, NijssenFromont2010,aghaei2020learning,verhaeghe2019,nijssen2020,NijssenFromont2010, demirovic2022murtree} using mathematical programming or dynamic programming. However, building sparse optimal classification trees is a much easier problem, since the 0-1 loss has natural discrete lower bounds, and binary integer programming can be used; this is not true of regression, which uses (real-valued) mean squared error as its loss function.

Let us discuss the few works that do address challenges resembling optimal regression trees. The works of  \citet{blanquero2022sparse} and  \citet{BertsimasEtAl21201} do not construct traditional sparse trees with constant predictions in the leaves; their leaf nodes contain linear or polynomial classifiers, thus the formula for producing predictions is quite complex. The former \citep{blanquero2022sparse} uses $\ell_\infty + \ell_1$ regularization for the linear models within the 
nodes and the latter \cite{BertsimasEtAl21201} uses $\ell_2$ regularization for polynomial models in the leaves. Neither of these regularize the number of leaves. The evtree algorithm \citep{grubinger2014evtree} claims to construct globally optimal trees, but since it is purely an evolutionary method (no bounds are used to reduce the search space), there is no guarantee of reaching optimality, and one never knows whether optimality has already been reached.
\citet{Dunn2018} and \citet{verwer2017learning} provide mathematical programming formulas for optimal regression trees, but no open source code is available; regardless, mathematical programming solvers are generally slow. \citet{InterpretableAI} provides proprietary software that requires a license, but it is not possible to ascertain whether it uses local search or mathematical programming; we suspect it uses local search heuristics, despite the claim of optimal solutions.
In other words, as far as we know, there is no other prior peer-reviewed work that directly produces \textit{sparse, provably-optimal regression trees} with publicly available code.

Our goal is to design optimal sparse regression trees in the classical sense, with a small number of leaves, a single condition at each split, and a constant prediction in each leaf. This makes the predictions easy to understand and compute, even for people who cannot understand equations. Given a trained tree, one can print it on an index card and compute a prediction without adding or multiplying any numbers, which makes these models easy to troubleshoot and use -- even in high-stakes settings. An example tree for the \textbf{seoul bike} dataset \cite{ve2020rule, sathishkumar2020using, Dua:2019} constructed by our method is shown in Figure \ref{fig:example-tree}.

\newcommand{\eq}{=}

\begin{figure}
    \centering
    \begin{forest} [ {$is Functional = Yes$}  [ {$Seasons=Winter$} ,edge label={node[midway,left,font=\scriptsize]{True}}  [ $\mathbf{225.52}$  ]  [ {$-17.9 < Temperature(^\circ C) \leq -3.5$}   [ $\mathbf{521.49}$  ]  [ {$11:30< Hour \leq 17:15$}   [ $\mathbf{1575.10}$  ]   [ $\mathbf{836.90}$  ]  ] ] ]  [ $\mathbf{0.00}$ ,edge label={node[midway,right,font=\scriptsize]{False}} ]  ] \end{forest}

    \caption{Optimal regression tree for \textit{seoul bike} dataset with $\lambda=0.05, \textit{max depth}=5$. This dataset predicts the number of bikes rented in an hour. It is binarized by splitting each feature into four categories. }
    \label{fig:example-tree}
\end{figure}
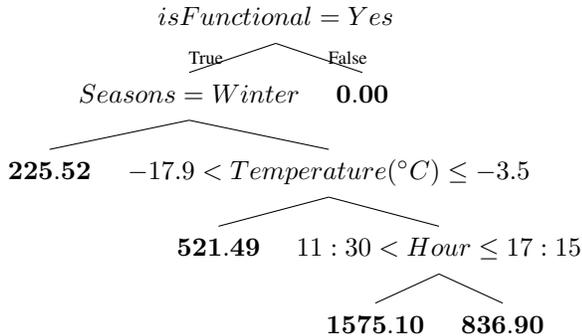


Our formulation is a dynamic-programming-with-bounds approach, where the search space is either reduced or searched methodically. Such approaches have been highly successful for  classification trees \citep{AngelinoLaAlSeRu17-kdd,lin2020generalized,nijssen2020} but have not been previously used for regression trees. An important novel element of our formulation is a lower bound that we call the ``k-Means equivalent points lower bound.'' To reduce the search space, we need as tight a bound as possible on the objective. Our bound makes use of the observation that any high-quality decision tree of $C$ leaves will perform as bad or worse than the performance of fully-optimal C-Means clustering on the labels alone (without any features). We discuss this in Section \ref{sec:bounds}.

Our main results are: (1) The first algorithm with publicly available code for optimal sparse regression trees in the classical sense, with a \textit{proof of optimality}. We call this algorithm Optimal Sparse Regression Trees (\ourmethod{}). (2) A substantial speedup over evtree, owing to our analytical bounds that reduce the search space. Evtree globally optimizes models, but does not provide a proof of optimality as \ourmethod{} does.

\section{Notation and Objective}
We denote the training dataset $(\X, \y)$ as $\{(\x_i, y_i)\}^N_{i=1} $, where $\x_i \in \{0,1\}^M$ is a binary feature vector and $y_i \in \mathbb{R} $ is a target variable. (Real-valued features in the raw dataset can be transformed into binary features in many different ways, e.g., splitting the domain of the feature into equal-sized buckets, splitting between every two realized values of the variable in the training set, using splits from a reference model as in \citet{McTavishZhongEtAl2022}; we use the first technique.) 

We denote $\mathcal{L}(t,\X,\y) := \frac{1}{N}\sum_{i=1}^N \left(y_i -\hat{y}_i\right)^2$ as the loss of tree $t$ on the training dataset, where $\hat{y}_i$ is the prediction of $\x_i$ by tree $t$, i.e., we use mean squared error (MSE) as the loss function. We define the objective function of tree $t$, $R(t,\X,\y)$ as a combination of \textbf{tree loss} and \textbf{penalty on complexity}:
\[\mathcal{L}(t,\X,\y) + \lambda \cdot \text{complexity}(t)\] where the complexity penalty is  $H_t$, the number of leaves in tree $t$:
\begin{equation}\label{eq:obj1}
     R(t,\X,\y) := \mathcal{L}(t,\X,\y) + \lambda H_t.
\end{equation}
Computationally, it is easier when a depth constraint is added:
\begin{equation}\label{eq:obj2}
     \mathcal{L}(t,\X,\y) + \lambda \cdot \text{complexity}(t), \text{ s.t. } \textrm{depth}(t) \leq d.
\end{equation}
Adding a depth constraint dramatically reduces the search space, but it can lead to suboptimal values of the objective if the depth constraint is smaller than the depth of the optimal solution. 
Unlike all previous approaches, our algorithm can find provably-optimal trees that globally minimize Equation \eqref{eq:obj1} without a depth constraint.

\section{Bounds}\label{sec:bounds}
\input{bounds}
\section{Algorithm} \label{sec:algorithm}
\input{algorithm}

\section{Comparison of Regression Tree Optimization Methods}
\textit{Unlike other methods, OSRT can optimize regression trees without a hard depth constraint and support mean absolute error (L1 loss).} Table \ref{tb:summaryofmethods} summarizes the comparison of different regression tree optimization methods. Blue cells are comparative advantages, and red cells are comparative disadvantages.
\renewcommand\tabularxcolumn[1]{m{#1}}
\begin{table*}[ht]
\small
\centering
\begin{tabularx}{\linewidth}{|X|c|c|c|c|c|c|c|}
\hline
 & OSRT & IAI & Evtree & GUIDE & CART & ORT & DTIP \\ \hline \hline 
Guarantee optimality &
   Yes &
  No &
   No &
  No &
   No &
  Yes &
  Yes \\ \hline
Optimization strategy & DPB  & Local Search & Evolutionary & Greedy search & Greedy Search & MIO & MIO  \\ \hline
Can optimize without depth constraint &
  Yes&
  No &
  No& Yes & Yes & No & No \\ \hline
Support (weighted) least absolute deviation &
Yes & No & No & No & Yes &
  Unknown &
  Unknown \\ \hline
Implementation available &
  Yes & \makecell{Yes (Executable\\ Only)} & Yes &  \makecell{Yes (Executable \\ Only)} & Yes & No & No \\\hline
\end{tabularx}
\caption{Comparison of OSRT, IAI \cite{InterpretableAI}, Evtree \cite{grubinger2014evtree}, GUIDE \cite{loh2002regression}, CART \cite{Breiman84}, ORT \cite{Dunn2018} and DTIP \cite{verwer2017learning}.  Executables for IAI and GUIDE are available, but their source code is not. DPB is dynamic programming with bounds, MIO is mixed integer optimization.}
\label{tb:summaryofmethods}
\end{table*}

\section{Experiments}\label{sec:exp}
We ran experiments on 12 datasets; the details are described in Appendix C.1. Our evaluation answers the following:
\begin{enumerate}
    \item Are trees generated by existing regression tree optimization methods truly optimal? How well do optimal sparse regression trees generalize? How far from optimal are greedy-approach models?  (\S \ref{sec:optimality})
    \item Does each method yield consistently high-quality results? (\S \ref{sec:controllability})
    \item How fast does \ourmethod{} converge, given that it guarantees optimality? (\S \ref{sec:speed})
    \item How much do our novel bounds contribute to the performance of \ourmethod? (\S \ref{sec:valueof_kmeans})
    \item What do optimal regression trees look like? (\S \ref{sec:opt_tree})
\end{enumerate}
 \subsection{Optimality and Generalization}\label{sec:optimality}
We compare trees produced by CART \citep{Breiman84}, GUIDE \citep{loh2002regression}, IAI \citep{InterpretableAI}, Evtree \citep{grubinger2014evtree} and \ourmethod, trained on various datasets. For each method, we swept a range of hyperparameters to illustrate the relationship between loss and sparsity (IAI, Evtree, and OSRT all penalize the number of leaves). Optimization experiments in Appendix D and cross-validation experiments in Appendix H, along with a demonstration of these results in Figure \ref{fig:main:lvs:train_and_test} show: (1) trees produced by other methods are usually \textit{sub-optimal} even if they claim optimality (they do not \textit{prove} optimality), and \textit{only our method can consistently find the optimal trees}, which are the most efficient frontiers that optimize the trade-off between loss and sparsity, (2) OSRT has the best generalization performance among methods, and (3) \textit{we can now quantify how far from optimal other methods are}. 
\begin{figure*}[htbp]
    \centering
    \includegraphics[width=0.45\textwidth]{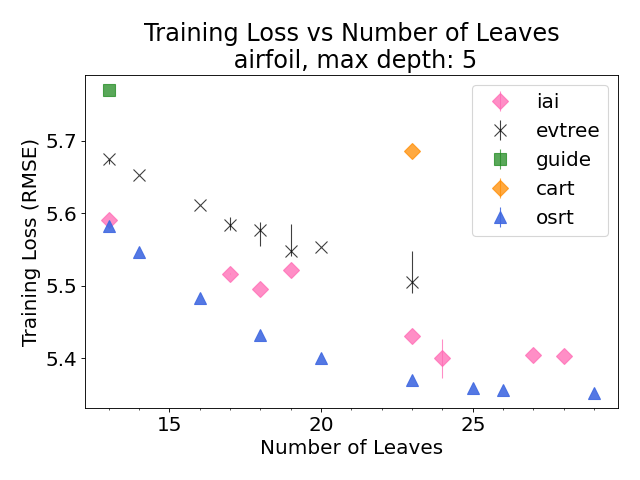}
    \includegraphics[width=0.45\textwidth]{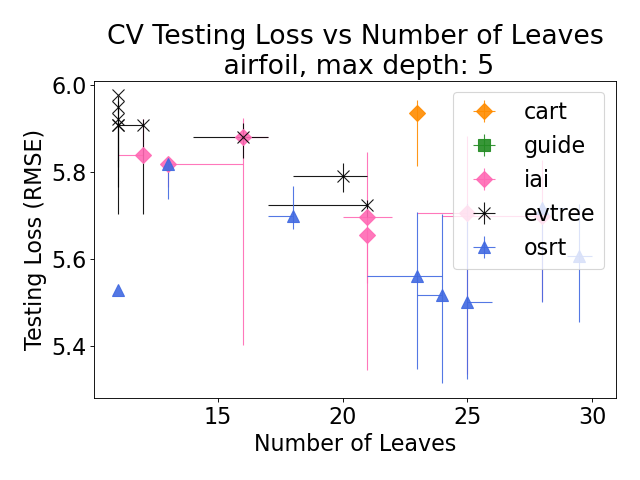}
    \caption{Training and testing loss achieved by IAI, Evtree, GUIDE, CART, OSRT on dataset \textit{airfoil}, $d=5$.}
    \label{fig:main:lvs:train_and_test}
\end{figure*}

\subsection{Controllability}\label{sec:controllability}
Unlike IAI and Evtree, our method does not rely on random seeds. \textit{The results returned by OSRT are consistently high quality, while those of IAI and Evtree are not.} Figure \ref{fig:main:variance} shows the stochasticity of various methods.
Trees produced by IAI and Evtree have large variance in complexity and accuracy if we do not fix the random seed. High variance of loss and sparsity can result in inaccuracy and overfitting. Details and results of this experiment can be found in Appendix F.
\begin{figure}[htbp]
    \centering
    \includegraphics[width=0.45\textwidth]{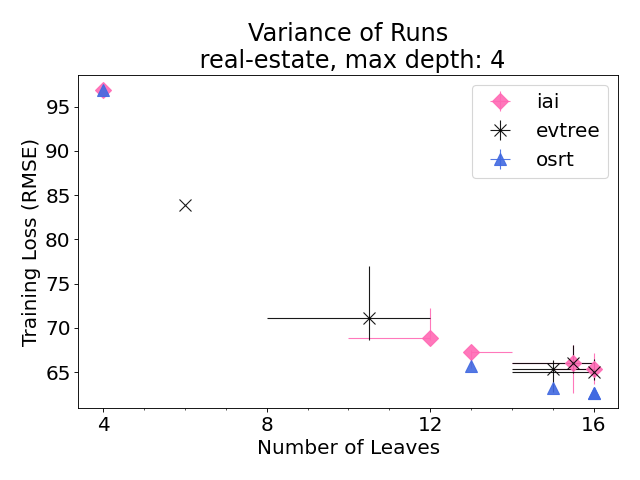}
    \caption{Variance (horizontal and vertical lines) of trees generated by IAI, Evtree, OSRT using 10 different random seeds on dataset \textit{real-estate}.}
    \label{fig:main:variance}
\end{figure}
\subsection{Speed and Scalability}\label{sec:speed}
Our method is one of the \textit{fastest} regression tree optimization methods and the \textit{only one} that also guarantees optimality. Figure \ref{fig:main:tvs} shows that OSRT performs well in run time, and Figure \ref{fig:main:scalability} shows its outstanding scalability when tackling a large dataset with over 2 million samples. As the number of sample increases, Evtree slows down more than other methods and cannot converge within a 30-minute time limit when the sample size exceeds 50,000. More results are shown in Appendices G and I.    
\begin{figure}[htbp]
    \centering
    \includegraphics[width=0.45\textwidth]{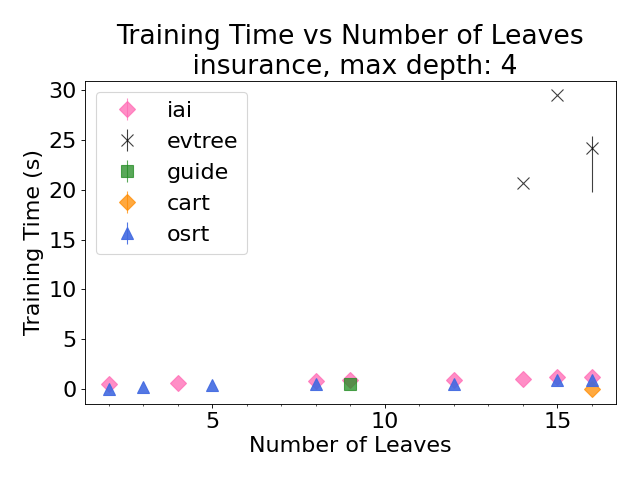}
    \includegraphics[width=0.45\textwidth]{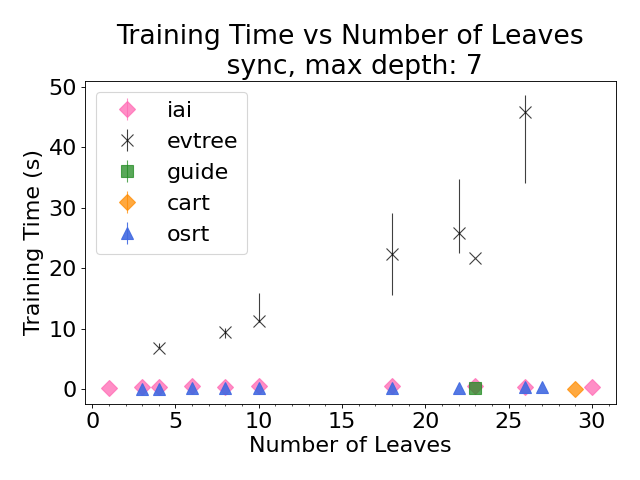}
    \caption{Training time of trees generated by CART, GUIDE, IAI, Evtree, OSRT.}
    \label{fig:main:tvs}
\end{figure}

\begin{figure}[htbp]
    \centering
    \includegraphics[width=0.45\textwidth]{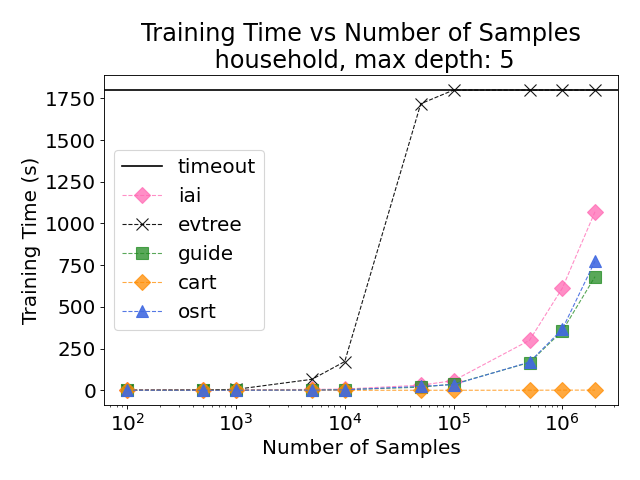}
    \caption{Training time of CART, GUIDE, IAI, Evtree and OSRT as a function of sample size on dataset \textit{household}, $d=5,\lambda=0.035$. (30-minutes time limit; Evtree timed out when sample size is beyond 50,000)}
    \label{fig:main:scalability}
\end{figure}

\subsection{Value of k-Means Lower Bound}\label{sec:valueof_kmeans}
The squared error used in regression tasks tends to make the equivalent points lower bound loose, preventing us from pruning more of the search space. The novel k-Means lower bound allows us to \textit{aggressively prune the search space}, and Figure \ref{fig:main:valueof_kmeans} shows that for the airfoil data set, the k-Means lower bound converged in less than one-fourth the time it took the equivalent points bound to converge. More results can be found in Appendix J.1.  

\begin{figure}[ht]
    \centering
    \includegraphics[width=0.45\textwidth]{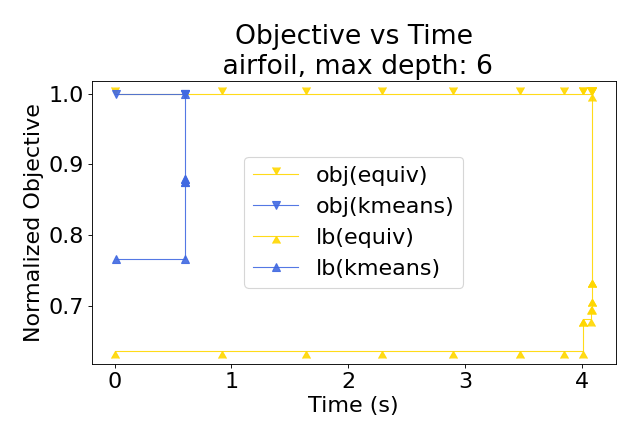}
    \caption{The time saved by k-Means lower bound (blue) over equivalent points bound (yellow), using $\lambda = 0.005$. The optimal solution is found when the lower bound equals objective. The k-Means bound converges in under a second.}
    \label{fig:main:valueof_kmeans}
\end{figure}

\subsection{Optimal Trees}\label{sec:opt_tree}
Figure \ref{fig:opt_trees} presents two optimal trees generated by OSRT on dataset \textit{servo}, with and without a depth constraint respectively, using the same regularization parameter. It shows that imposing a depth constraint sacrifices the global optimality of Equation \ref{eq:obj1}. More results regarding the ablation study of depth limit can be found in Appendix J.2, and Appendix L compares optimal trees generated by OSRT and sub-optimal trees generated by other methods.
\begin{figure}[ht]
    \centering
    \begin{subfigure}[b]{0.45\textwidth}
    \centering
    \begin{forest} [ {$pgain = 4$}   [ $\mathbf{0.67}$ ,edge label={node[midway,left,font=\scriptsize]{True}} ]  [ {$pgain = 5$} ,edge label={node[midway,right,font=\scriptsize]{False}}  [ $\mathbf{0.53}$  ]  [ {$pgain = 6$}   [ $\mathbf{0.52}$  ]  [ {$motor = D$}   [ $\mathbf{1.40}$  ]   [ $\mathbf{3.68}$  ]  ] ] ] ] \end{forest}
    \caption{(Max depth 4) Optimal tree with 5 leaves, $R^2 = 69.63\%$.}
    \end{subfigure}
    \begin{subfigure}[b]{0.45\textwidth}
    \centering
    \begin{forest} [ {$pgain = 4$}   [ $\mathbf{0.67}$ ,edge label={node[midway,left,font=\scriptsize]{True}} ]  [ {$pgain = 5$} ,edge label={node[midway,right,font=\scriptsize]{False}}  [ $\mathbf{0.53}$  ]  [ {$pgain = 6$}   [ $\mathbf{0.52}$  ]  [ {$motor = D$}   [ $\mathbf{1.40}$  ]  [ {$motor = E$}   [ $\mathbf{2.40}$  ]   [ $\mathbf{4.10}$  ]  ] ] ] ] ] \end{forest}
    \caption{(No depth limit) Optimal tree with 6 leaves, $R^2 = 75\%$.}
    \end{subfigure}
    \caption{Optimal trees generated by OSRT on dataset \textit{servo} with (a) depth limit 4 and (b) no depth limit. Tree (b) has only one more leaf but explains 5\% more training data variance than Tree (a).}
    \label{fig:opt_trees}
\end{figure}
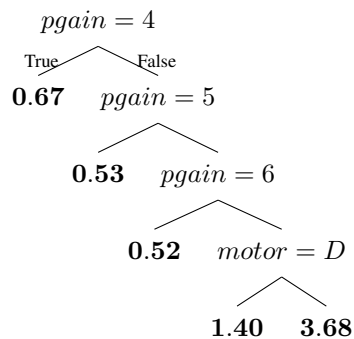
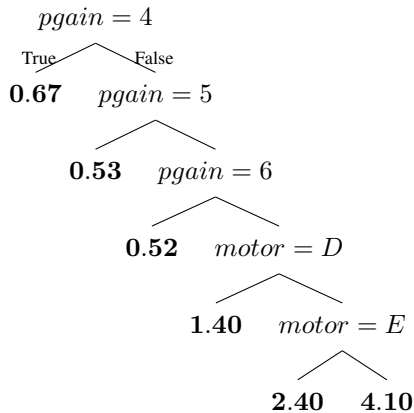

\section{Conclusion}
We provide the first method to find provably-optimal regression trees within a reasonable time. Our method quickly and consistently finds an optimal sparse model that tends to generalize well. Our method also scales well even for large datasets. OSRT provides a naturally human-interpretable option for solving regression problems in contrast to other, uninterpretable methods such as ridge regression, support vector regression, ensemble methods and neural networks.
\section*{Code Availability}
The implementation of OSRT is available at \url{https://github.com/ruizhang1996/optimal-sparse-regression-tree-public}.

\noindent Our experiment code is available at \url{https://github.com/ruizhang1996/regression-tree-benchmark}.
\section*{Acknowledgments}
We acknowledge the support
of the Natural Sciences and Engineering Research Council of
Canada (NSERC), the National Institute on Drug Abuse (NIDA) under grant DA054994, and the National Science Foundation (NSF) under grant IIS-2130250.

\bibliography{references}
\bibstyle{aaai23}
\appendix
\onecolumn
\include{supplement}

\end{document}

%% file: bounds.tex
Following \citet{HuRuSe2019, lin2020generalized}, we represent a tree as a set of leaves. Trees with identical leaves, regardless of different internal branching nodes, are considered equivalent. This representation allows us to save memory and avoid duplicate computation during tree construction. 

Our algorithm, like that of \citet{lin2020generalized} for classification, is a dynamic-programming-with-bounds algorithm. This algorithm searches the whole space of trees systematically from smaller to larger trees. If the algorithm determines (through the use of bounds) that the current partial tree it is constructing can never be extended to form an optimal full tree, it stops exploring that part of the search space. Thus, the tighter the bounds, the more the algorithm reduces the search space and the more quickly it converges to the optimal solution.  Thus, we present a series of tight bounds that reduce computation by reducing the search space. 

We start with notation. A tree $t$ is represented as a set of $H_t$ distinct leaves:
$t = \{l_1, l_2, \dots, l_{H_t}\}.$
It can also be written as:
\[t = (t_{\textrm{fix}}, \delta_{\textrm{fix}}, t_{\textrm{split}}, \delta_{\textrm{split}}, K, H_t)\]
where $t_{\textrm{fix}}=\{l_1, l_2, \dots, l_K\}$ are a set of $K$ \textbf{fixed leaves} that are not allowed to be further split in this part of the search space, $\delta_{\fix} = \{\hat{y}_{l_1},\hat{y}_{l_2}, \dots, \hat{y}_{l_K}\} \in \mathbb{R}^K $ are predicted targets for the fixed leaves, $t_{\splita} = \{l_{K+1}, l_{K+2}, \dots, l_{H_t}\}$ are $H_t - K$ \textbf{splitting leaves} that can be further split in this part of the search space, and their predicted targets are $\delta_{\textrm{split}} = \{\hat{y}_{l_{K+1}},\hat{y}_{l_{K+2}}, \dots, \hat{y}_{H_t}\} \in \mathbb{R}^{H_{t}-K}$. 

We generate new trees by splitting different subsets of splitting leaves in tree $t$. We define $t' = (t'_{\fix}, \delta'_{\fix}, t'_{\splita}, \delta'_{\splita}, K', H_{t'})$ as a \textbf{child tree} of $t$ \textit{if and only if }$t'_{\fix}$ is a superset of $t_{\fix}$, and $t'_{\splita}$ is generated through splitting a subset of $t_{\splita}$. We denote $\sigma(t)$ as the set of all child trees of $t$.

The following bounds start out analogous to those of \citet{lin2020generalized} for classification and diverge entirely when we get to the new k-Means Lower Bound.

\subsection{Lower Bounds}
The loss of a tree has contributions from its two parts: fixed leaves and splitting leaves. Since the fixed leaves cannot be further split in this part of the search space, their contribution provides a lower bound for tree $t$ and all of its child trees. Define the objective lower bound of tree $t$ as 
\[R(t,\X,\y) \geq \mathcal{L}(t_{\fix},\X,\y) + \lambda H_t, \]
where $\mathcal{L}(t_{\fix},\X,\y)$ is the sum of losses for fixed leaves:
\begin{equation}\label{eq:lb}
    \mathcal{L}(t_{\fix},\X,\y) = \frac{1}{N} {\sum_{i=1}^N }\left(y_i -\hat{y}_i\right)^2 \cdot \mathbf{1}_{\capt(t_{\fix}, \x_i)}
\end{equation}
$\mathbf{1}_{\capt(t_{\fix}, \x_i)}$ is 1 when one of the leaves in $t_{\fix}$ captures $\x_i$, 0 otherwise. ($t_{\fix}$ captures $\x_i$ when $\x_i$ falls into one of the fixed leaves of $t$.) If splitting leaves have 0 loss, then the tree's loss is equal to the lower bound.

We denote the current best objective we have seen so far as $R^c$. If the objective lower bound of $t$ is worse than $R^c$, i.e., $\mathcal{L}(t_{\fix},\X,\y) + \lambda H_t > R^c$, then $t$ cannot be an optimal tree, nor can any of its children, and the search space can be pruned. To show this, we need the following bound, stating that the child trees of $t$ all obey the same lower bound from the fixed leaves. \textit{Note that all proofs are in the appendix.}
\begin{theorem}\label{thm:hierarchical_lb}
(Hierarchical Objective Lower Bound). Any tree $t' = (t'_{\fix}, \delta'_{\fix}, t'_{\splita}, \delta'_{\splita}, K', H_{t'}) \in \sigma(t)$ in the child tree set of tree $t = (t_{\fix}, \delta_{\fix}, t_{\splita}, \delta_{\splita}, K, H_t)$ obeys: \[R(t',\X,\y) \geq \mathcal{L}(t_{\fix},\X,\y) + \lambda H_t. \]
\end{theorem}
That is, the objective lower bound of the parent tree holds for all its child trees. This bound ensures that we do not further explore child trees if the parent tree can be pruned via the lower bound.

The next bound removes all of a tree's child trees from the search space, even if the tree itself could not be eliminated by the previous bound. 

\begin{theorem}\label{thm:look_ahead}
(Objective Lower Bound with One-step Lookahead). Let $t = (t_{\fix}, \delta_{\fix}, t_{\splita}, \delta_{\splita}, K, H_t)$ be a tree with $H_t$ leaves. If $\mathcal{L}(t_{\fix},\X,\y) + \lambda H_t + \lambda > R^c$, even if its objective lower bound obeys  $\mathcal{L}(t_{\fix},\X,\y) + \lambda H_t \leq R^c$, then for any child tree $t' \in \sigma(t)$, $R(t', \X, \y) > R^c$.
\end{theorem}
That is, even if a parent tree cannot be pruned via its objective lower bound, if $\mathcal{L}(t_{\fix},\X,\y) + \lambda H_t + \lambda > R^c$, all of its child trees are sub-optimal and can be pruned (and never explored).

\subsection{Equivalent Points}
Before making the lower bound of the objective tighter, let us introduce equivalent points. We define \textbf{equivalent points} as samples with \textit{identical features} but possibly different target values. It is impossible to partition these samples into different leaves in any tree; a leaf that captures a set of equivalent points that have different targets can \textbf{never} achieve zero loss. Our bound exploits this fact.

Let $u$ be a set of equivalent points where samples have exactly the same feature vector $\x$, such that $\forall j_1,j_2,...j_{|u|}\in u$:
\[\x_{j_1} = \x_{j_2} = \dots = \x_{j_{|u|}}.\]
We define the \textbf{equivalence loss} $\mathcal{E}_u$ as the sum of squares error for set $u$ when the estimate of the leaf is the best possible, namely the mean of targets for points in $u$. Define $\bar{y}_u=\frac{1}{|u|}\sum_{(\x_i, y_i) \in u} y_i$:
\begin{equation}
    \mathcal{E}_u = \frac{1}{N} \sum_{(\x_i, y_i) \in u}\left( y_i - \bar{y}_u \right)^2.
\end{equation}
\begin{theorem}\label{thm:equiv_lb}
(Equivalent Points Lower Bound). Let $t = (t_{\fix}, \delta_{\fix}, t_{\splita}, \delta_{\splita}, K, H_t)$ be a tree with $K$ fixed leaves and $H_t - K$ splitting leaves. For any child tree $t'= (t'_{\fix}, \delta'_{\fix}, t'_{\splita}, \delta'_{\splita}, K', H_{t'}) \in \sigma(t)$:
\begin{equation}\label{eq:equiv_lb}
    R(t', \X, \y) \geq  \mathcal{L}(t_{\fix},\X,\y) + \lambda H_{t} + \sum_{u\in U}  \mathcal{E}_u \cdot \mathbf{1}_{\capt(t_{\splita}, u)},
\end{equation}
where $U$ is the set of equivalent points sets in training dataset $(\X, \y)$ and $\mathbf{1}_{\capt(t_{\splita}, u)}$ is 1 when $t_{\splita}$ captures set $u$, 0 otherwise. \end{theorem}
\noindent\textit{Combining with the idea of Theorem \ref{thm:look_ahead}, we have}:
\begin{equation}\label{eq:equiv_lb+lookahead}
    R(t', \X, \y) \geq  \mathcal{L}(t_{\fix},\X,\y) + \lambda H_{t} + \lambda + \sum_{u\in U}  \mathcal{E}_u \cdot \mathbf{1}_{\capt(t_{\splita}, u)}.
\end{equation}

The bound we introduce next, one of the main novel elements of the paper, is much tighter than the Equivalent Points Lower Bound.

\subsection{k-Means Lower Bound}
Let us consider the points within each leaf of a regression tree. The smallest possible losses within a leaf are achieved when the label values within the leaf are all similar to each other. If we know we will construct a tree with $C$ leaves and we could rearrange the points into any of the leaves, how would we arrange them to minimize loss?
The best loss we could possibly achieve would come from grouping points with the most similar targets together in the same leaf. This procedure is equivalent to computing an  optimal clustering of the targets (in 1 dimension) that minimizes
the sum of squared errors between each point and the position of its cluster center (the mean of the cluster). The solution to this clustering problem gives the lowest loss we can possibly achieve for any regression tree with $C$ leaves. We can use this as a  lower bound on the loss for $t_\splita$ by setting $C$ equal to the $H_t - K$ number of unsplittable leaves.
There exists a deterministic algorithm that takes linear time for computing the optimal k-Means loss on one dimensional data, which takes advantage the fact that the number line is totally ordered \cite{song2020efficient}.  

\begin{definition}\label{def:kmeans_1d}
(k-Means Problem for 1D targets)
Given a set of $N'$ 1D points $\y'$ and a number of clusters $C$, the goal is to assign points into $C$ clusters so that the sum of squared Euclidean distances between each point and its cluster mean is minimized. 
Define $\text{k-Means}(C, \y')$ to be the optimal objective of the k-Means algorithm for clustering 1D points $\y'$ of size $N'$ into $C$ clusters ($C\geq 1$):
\begin{equation}
    \text{k-Means}(C,\y') := \min_{z,A} \sum_{i=1}^{N'} (y'_i - z_{A(y'_i)})^2.
\end{equation}
$A(y'_i)$ is a function that specifies the cluster assignment of $y'_i$ among ${c_1, c_2, \dots, c_C}$, and $z_c$ is the centroid of cluster $c$, which is the mean of the points assigned to that cluster.
\begin{equation}\label{eq:cluster_mean}
    z_{c} := \frac{\sum_{A(y'_i) = c} y'_i}{\sum_{A(y'_i) = c} \mathbf{1}}.
\end{equation}
\end{definition}
We note here that for an assignment, $A$, of points to a tree's $C$ leaves, choosing the mean $z_c$ as the predicted label in each leaf $c$ yields the following for the k-Means objective, which is optimized over $z$ for a fixed $A$:
\begin{equation}\label{eq:kmeanspartial}
    \textit{k-Means-obj}(C,\y',A) := \min_{z} \sum_{i=1}^{N'} (y'_i - z_{A(y'_i)})^2.
\end{equation}
That is, minimizing the regression loss (sum of squares to the mean target in each leaf) also yields the k-Means' choice of cluster center as the mean of the targets for points belonging to a leaf. Clearly $\text{k-Means-obj}(C,\y',A)\geq \text{k-Means}(C,\y')$ since the latter is minimized over the assignment of points to clusters without regard to the tree structure at all. This logic is used in our lower bound.

\begin{theorem}\label{thm:kmeans_lb}
(k-Means Lower Bound).
Consider tree $t = (t_{\fix}, \delta_{\fix}, t_{\splita}, \delta_{\splita}, K, H_t)$.
and any child tree $t' =(t'_{\fix}, \delta'_{\fix}, t'_{\splita}, \delta'_{\splita}, K', H_{t'}) \in 
\sigma(t)$. Let $(\X_{t_{\splita}}, \y_{t_{\splita}})$ be samples captured by the splitting leaves $t_{\splita}$. Then,
\begin{align*}\label{eq:kmeans_lb}
    R(t', \X, \y) \geq & \mathcal{L}(t_{\fix},\X,\y) + \lambda K\\
    & + \min_C \left(\frac{1}{N}\text{k-Means}(C,\y_{t_{\splita}}) + \lambda C\right).
\end{align*}

\end{theorem}

\subsection{k-Means Equivalent Points Lower Bound}
We can make the bound from the last section even tighter. In fact, in the k-Means lower bound above, we ignored information inherent to the regression tree problem, because we ignored all of the features $\X$. We can achieve a tighter bound if we leverage our knowledge of $\X$ to again consider equivalent points. Specifically, all points with the same features must be assigned to the same leaf. We first present the definition of a modified k-Means problem and then state our theorem.

\begin{definition}\label{def:cons_kmeans}
(Constrained k-Means Problem for 1D targets)
Given a set of $N'$ 1D target points $\y'$ with feature vector $\X'$ and number of clusters $C$, the goal is to assign points into $C$ clusters so that the sum of squared Euclidean distances between each point and its cluster mean is minimized, under the constraint that all points with the same feature vector $\x'$ must be assigned to one cluster.
\begin{eqnarray}\nonumber
\lefteqn{\text{Constrained\_k-Means}(C,\X',\y')} \\&=& \label{eq:cons_kmeans}
\min_{z,A} \sum_{i=1}^{N'} (y'_i - z_{A(y'_i)})^2\\\nonumber 
&\text{ s.t. }& \textrm{if  } \x_i = \x_{i'} \textrm{, then  } A(y'_i) = A(y'_{i'}). 
\end{eqnarray}


\end{definition}

Adding this constraint makes the k-Means Lower Bound tighter. 


\begin{theorem}\label{thm:kmeans_equiv_lb}
(k-Means Equivalent Points Lower Bound).
Consider tree $t = (t_{\fix}, \delta_{\fix}, t_{\splita}, \delta_{\splita}, K, H_t)$.
and any child tree $t' =(t'_{\fix}, \delta'_{\fix}, t'_{\splita}, \delta'_{\splita}, K', H_{t'}) \in 
\sigma(t)$. Let $(\X_{t_{\splita}}, \y_{t_{\splita}})$ be samples captured by the splitting leaves $t_{\splita}$. Then, 
\begin{eqnarray}\label{eq:kmeans_equiv_lb}
\nonumber
\lefteqn{R(t', \X, y) \geq  \mathcal{L}(t_{\fix},\X,\y) + \lambda K +}\\&\hspace*{-15pt}
\min_C \left(\frac{1}{N}\text{Constrained\_k-Means}(C,\X_{\splita},\y_{t_{\splita}}) + \lambda C\right)&
\end{eqnarray}
where Constrained\_k-Means is defined in Equation \ref{eq:cons_kmeans}.
\end{theorem}

\subsection{Computing k-Means Equivalent Points Bound}

We now define a weighted version of the k-Means problem, where each sample point is associated with a weight. We derive these weights later as sizes of the equivalent sets.

    
\begin{definition}\label{def:weighted_kmeans}
(Weighted k-Means Problem)
Given a set of $N'$ 1D points $\y'$ with weights $\w \in \mathbb{R}^{N'} $ and number of clusters $C$, the goal is to assign points into $C$ clusters so that the weighted sum of squared Euclidean distances between each point and its cluster centroid is minimized. 
Define $\text{Weighted\_k-Means}(C, \y', \w)$ as the optimal objective of the k-Means algorithm clustering 1D points $\y'$ of size $N'$ into $C$ clusters ($C\geq 1$):
\begin{eqnarray}
    \nonumber\text{Weighted\_k-Means}(C,\y',\w )\\ = \min_{z,A} \sum_{i=1}^{N'} w_i\cdot(y'_i - z_{A(y'_i)})^2.
\end{eqnarray}
$A(y'_i)$ is a function that specifies the cluster assignment of $y'_i$ among ${c_1, c_2, \dots, c_C}$, and $z_c$ is the centroid of cluster $c$, which is the weighted mean of the points assigned to that cluster.
The weighted mean for cluster $c_j$ is:\begin{equation}\label{eq:weighted_cluster_mean}
    z_{c_j} = \frac{\underset{A(y'_i) = c_j}{\sum} w_i \cdot y'_i}{\underset{A(y'_i) = c_j}{\sum} w_i}
\end{equation}
which is similar to the one defined by \citet{song2020efficient}.
\end{definition}

\citet{song2020efficient} present an efficient $O(kN)$ solution to this weighted k-Means problem, where $k$ is the number of clusters and $N$ the number of data samples. We leverage this algorithm for the k-Means Equivalent Points Lower Bound (Theorem \ref{thm:kmeans_equiv_lb}). In the following theorem, we show that solving this weighted k-Means problem is equivalent to solving a constrained k-Means problem on a modified dataset.

\begin{theorem}\label{thm:cons_weighted_kmeans}
(Constrained k-Means with Equivalent Points is equivalent to weighted k-Means)
Recall in Definition \ref{def:cons_kmeans},
we have $N'$ 1D target points $\y'$ with features $\X'$and number of clusters $C$. We also have a constraint that all points in any equivalent set $u$  must be assigned to the same leaf.  
Define a modified dataset $(\X_{\modify}, \y_{\modify}, \w_{\modify})$, where all points of equivalent set $u$ in the original dataset $(\X',\y')$ are represented by a single point $(\x_u, y_u, w_u)$, where $\x_u$ is the same as the feature vector of equivalent set $u$,
\begin{equation}
    y_u = \frac{1}{|u|}\sum_{(\x'_i, y'_i) \in u} y'_i,
\end{equation} and the weight is the size of the equivalent set $u$\begin{equation}
    w_u = |u|.
\end{equation}
An optimal clustering of the modified dataset will directly provide an optimal clustering of the original dataset with the equivalent points constraint from Equation \ref{eq:cons_kmeans}. (All points from the original dataset contributing to a weighted point in the modified dataset will be assigned to its cluster.)
\end{theorem}
That is, solving the Weighted k-Means problem produces the same solution(s) as solving the Constrained k-Means problem.
Thus, solving the weighted k-Means problem on the modified dataset provides the same result as solving the constrained k-Means on the original dataset.

In Equation \ref{eq:kmeans_equiv_lb}, we observe that computing the k-Means Equivalent Points Lower bound requires that we find the minimum of Constrained\_k-Means across all possible $C$. One can easily see that it is sufficient to iterate $C$ from $1$ to $|\y_{t_{\splita}}|$, where every data point is in its own cluster. However, this would be costly when dealing with large datasets. 
The following theorem, as proved in \citet{aggarwal1994finding}, shows that the loss improved from adding more clusters decreases as the number of clusters increases. It means we do not need to generate k-Means solutions for all $C$ up to the size of the subproblem, \textit{we can stop as soon as the objective improvement from adding new clusters becomes less than the regularization $\lambda$}.
\begin{theorem}\label{thm:kmeans_k_convex}
\citep[Convexity of Weighted\_k-Means in number of clusters, from][]{aggarwal1994finding}
Recall $\text{Weighted\_k-Means}(C,\y',\w )$ from Definition \ref{def:weighted_kmeans} for number of clusters $C$, 1D points $\y'$, and weights $\w$. We have 
\begin{eqnarray}
\nonumber
&&\text{Weighted\_k-Means}(C-1,\y',\w) \\
\nonumber
&+&\text{Weighted\_k-Means}(C+1,\y',\w) \\
 &\geq&2\times \text{Weighted\_k-Means}(C,\y',\w). 
 \label{eq:kmeans_k_convex}
\end{eqnarray}
\end{theorem}

Other bounds that help reduce the search space (e.g. Leaf Bounds, Splitting Bounds, Permutation Bound, Subset Bound) can be found in Appendix B. 

%% file: algorithm.tex
We implemented \ourmethod{} based on the GOSDT \cite{lin2020generalized} framework, which uses a \textit{dynamic-programming-with-bounds} formulation. Each \textit{subproblem} in this formulation is identified by a support set $s = \{s_1, s_2, \ldots, s_{N}\}$, where $s_i$ is a boolean value indicating whether point $i$ is in the support set $s$. Each leaf and branching node corresponds to a subproblem, recording which 
samples traverse through that node (or leaf). GOSDT records and updates lower and upper bounds of the objective for each subproblem and stores them in a \textit{dependency graph}. The dependency graph summarizes the relationship among subproblems. In dynamic programming formulations, finding  tight bounds is crucial in reducing the runtime of the algorithm, because 
that is the key to eliminating large portions of search space. Our k-Means-based bounds for regression are tight and substantially reduce time-to-optimality, as we show in Section \ref{sec:valueof_kmeans} and Appendix J.1.
Like GOSDT, our method finds the optimal trees when the lower and upper bounds of the objective converge. Algorithm \ref{alg:lowerbound} below is a subroutine of OSRT.


\textbf{Compute\_Lower\_Bound (Algorithm \ref{alg:lowerbound})}:
This algorithm implements the k-Means Equivalent Points Lower Bound as defined in Theorem \ref{thm:kmeans_equiv_lb}. We leveraged a k-Means solver from \citet{song2020efficient}, which is a dynamic programming formulation that fills in a $C$ by $N$ matrix, where $C$ represents the number of clusters and $N$ corresponds to the number of samples.
We do not assume a maximum value for $C$ and instead grow the table one row at a time,
using the \textit{fill\_kmeans\_dp} function from their implementation. Each point $(a, b)$ in the table represents the optimal k-Means loss using $a$ clusters and the first $b$ datapoints.

\textbf{Line 1-3}: Compute equivalent target set by grouping equivalent points together, and gather all of their labels.
\textbf{Lines 4-5}: Compute weight $\w$ and value $\textbf{v}$ that defines the k-Means problem.
\textbf{Lines 6-8}: Initialize current loss, \textit{loss}, previous loss, \textit{loss}$'$, number of clusters used, \textit{nClusters}, and dynamic programming table, \textit{dp\_table}. 
\textbf{Lines 9-17}: Solve weighted k-Means problem by adding clusters one at a time.  
\textbf{Line 11}: Retrieve loss using
\textit{nClusters} clusters from the last entry of the last filled row of dynamic programming table.
\textbf{Lines 12-14}: Terminate algorithm if we can no longer benefit from adding more clusters as the reduction of loss by adding one cluster is monotonically decreasing. See Theorem \ref{thm:kmeans_k_convex}.
\textbf{Line 18}: Compute constant correction term, \textit{correction}, that restores weighted k-Means to constrained k-Means problem (see Theorem \ref{thm:cons_weighted_kmeans}).

\definecolor{comment}{rgb}{0.1,0.1,0.1}
\def\comment#1{\textcolor{comment}{\textit{// #1 }}}

\setlength{\textfloatsep}{0pt}
\begin{algorithm}
\caption{compute\_lower\_bound$(\textit{dataset}, \textit{sub}, \lambda) \\ ~\rightarrow~ \textit{lower\_bound}$\\
\comment{For a subproblem \textit{sub} and regularization $\lambda$, compute its Equivalent k-Means Lower Bound }}\label{alg:lowerbound}
\label{alg:lowerbound2}
\begin{tabbing}
xxx \= xx \= xx \= xx \= xx \= xx \kill
1: Let $U$ = the set of unique samples $\x_i ~ \in \textit{sub}$ \\
\comment{For each unique sample in $U$, create a set of all targets $y$} \\
\comment{corresponding to copies of that sample in \textit{sub}} \\
\comment{(equivalent point sets)}\\
2: $E = \emptyset$ \\
3: $\forall \x_i \in U, E_{\x_i}.\textit{append}(\{y_j\}~|~\forall \x_j \in \textit{sub},  y_j~\textit{if}~\x_i = \x_j)$ \\
\comment{For each unique sample in $U$, compute the number of } \\
\comment{identical samples to it (producing the vector $\w$) and the  } \\
\comment{average of all targets (producing the vector $\textbf{v}$)} \\
4: $\w \leftarrow \{|{E_\x}| ~|~ E_\x \in E\}$ where $E_\x \subset \y$ is a set of \\targets for one unique $\x \in U$. \\
5: $\textbf{v} \leftarrow \{\overline{E_\x} ~|~ E_\x \in E\}$ where  $\overline{E_\x}$ denotes \textit{average} of $E_\x$ \\
6: $\textit{loss}, \textit{loss}' \leftarrow \textrm{inf}$ \\
7: $\textit{nClusters} \leftarrow 1$  \\
\comment{We initialize the dynamic programming table with}\\
\comment{no rows, but one column for each element of U}\\
8: $\textit{dp\_table} \leftarrow [ ][|E|]$ \\
9: \textbf{while} \textit{true} \textbf{do} \\
\>\>\comment{Fill in the $(\textit{nClusters}-1)^{th}$ row of dp\_table}\\
10:\>\> $\textit{dp\_table} \leftarrow \textit{fill\_kmeans\_dp}(\textit{nClusters}-1,w, v,$\\
\>\>\>$\textit{dp\_table})$ \\
11:\>\> $\textit{loss} \leftarrow \textit{dp\_table}[(\textit{nClusters}-1, |E| - 1)]$ \\
12:\>\> \textbf{if} $\textit{loss}' - \textit{loss} \leq \lambda$ \textbf{then} \\
\>\> \comment{Adding this cluster does not reduce loss enough } \\
\>\> \comment{to justify addition of another cluster}\\
13:\>\>\> \textbf{break} \\
14:\>\> \textbf{end} \\
15: \>\>  $\textit{nClusters} \leftarrow \textit{nClusters} + 1$ \\ 
16: \>\> $\textit{loss}'\leftarrow \textit{loss}$ \\
17: \textbf{end} \\
\comment{Correct from weighted k-Means to constrained k-Means}\\
18: $\textit{correction} \leftarrow \sum_{\x_j \in \textit{sub}} y_j^2 - \sum_{\x_i \in U}w_iv_i^2$ \\
19: \textbf{return} $\textit{loss}' + \lambda\times \textit{nClusters} + \textit{correction}$
\end{tabbing}
\end{algorithm}
\setlength{\textfloatsep}{0pt}

%% file: supplement.tex

\section{Theorems and Proofs}
\subsection{Proof of Theorem \ref{thm:hierarchical_lb}}
\textbf{Theorem \ref{thm:hierarchical_lb}}
(Hierarchical Objective Lower Bound). \textit{Any tree $t' = (t'_{\fix}, \delta'_{\fix}, t'_{\splita}, \delta'_{\splita}, K', H_{t'}) \in \sigma(t)$ in the child tree set of $t = (t_{\fix}, \delta_{\fix}, t_{\splita}, \delta_{\splita}, K, H_t)$ obeys: \[R(t',\X,\y) \geq \mathcal{L}(t_{\fix},\X,\y) + \lambda H_t. \]}
That is, the objective lower bound of the parent tree holds for all its child trees. This bound ensures that we do not further explore child trees if the parent tree can be pruned via the lower bound.
\begin{proof}
As we know, $K' \geq K, H_{t'} > H_t$, since $t'$ is a child tree of $t$. The objective lower bound (which holds for all trees) of $t'$ is:
\begin{equation}\label{eq:child_lb}
 R(t',\X,\y) \geq \mathcal{L}(t'_{\fix},\X,\y) + \lambda H_{t'}.   
\end{equation}
Since  $\mathcal{L}(t'_{\fix},\X,\y) = \mathcal{L}(t_{\fix},\X,\y) + \mathcal{L}(t'_{\fix} \setminus t_{\fix},\X,\y)$ and the loss of $K' - K$ fixed leaves in $t$ is nonnegative, i.e., $\mathcal{L}(t'_{\fix} \setminus t_{\fix},\X,\y) \geq 0$, we have:
\begin{equation}\label{eq:fix_loss}
 \mathcal{L}(t'_{\fix},\X,\y) \geq \mathcal{L}(t_{\fix},\X,\y)   
\end{equation}
therefore:
\begin{equation}\label{eq:hierachical_lb}
    R(t',\X,\y) \geq \mathcal{L}(t_{\fix},\X,\y) + \lambda H_{t}. 
\end{equation}
\end{proof}
\subsection{Proof of Theorem \ref{thm:look_ahead}}
\textbf{Theorem \ref{thm:look_ahead}}
\textit{(Objective Lower Bound with One-step Lookahead). Let $t = (t_{\fix}, \delta_{\fix}, t_{\splita}, \delta_{\splita}, K, H_t)$ be a tree with $H_t$ leaves. If $\mathcal{L}(t_{\fix},\X,\y) + \lambda H_t + \lambda > R^c$, even if its objective lower bound $\mathcal{L}(t_{\fix},\X,\y) + \lambda H_t \leq R^c$, then for any child tree $t' \in \sigma(t)$, $R(t', \X, \y) > R^c$.
That is, even if a parent tree cannot be pruned via its objective lower bound, if $\mathcal{L}(t_{\fix},\X,\y) + \lambda H_t + \lambda > R^c$, all of its child trees are sub-optimal and can be pruned (and never explored). }

\begin{proof}
From the objective lower bound of $t'$ defined in Equation \ref{eq:child_lb}, and Equation \ref{eq:fix_loss}, we have:
\begin{equation}
    R(t',\X,\y) \geq \mathcal{L}(t_{\fix},\X,\y) + \lambda H_{t'}. 
\end{equation}
Because the child tree has at least one more leaf than the parent tree, $H_{t'} \geq H_t + 1$, we have:
\begin{equation}\label{eq:look_ahead} 
    R(t',\X,\y) \geq \mathcal{L}(t_{\fix},\X,\y) + \lambda H_{t} + \lambda.
\end{equation}
Thus, if $\mathcal{L}(t_{\fix},\X,\y) + \lambda H_{t} + \lambda>R^c$, then $R(t',\X,\y)>R^c$ for child trees $t'$, which means all the child trees can be pruned.
\end{proof}
\subsection{Proof of Theorem \ref{thm:equiv_lb}}
\textbf{Theorem \ref{thm:equiv_lb}}
\textit{(Equivalent Points Lower Bound). Let $t = (t_{\fix}, \delta_{\fix}, t_{\splita}, \delta_{\splita}, K, H_t)$ be a tree with $K$ fixed leaves and $H_t - K$ splitting leaves. For any child tree $t'= (t'_{\fix}, \delta'_{\fix}, t'_{\splita}, \delta'_{\splita}, K', H_{t'}) \in \sigma(t)$:
\begin{equation}
    R(t', \X, \y) \geq  \mathcal{L}(t_{\fix},\X,\y) + \lambda H_{t} + \sum_{u=1}^U  \mathcal{E}_u \cdot \mathbf{1}_{\capt(t_{\splita}, u)},
\end{equation}
where $\mathbf{1}_{\capt(t_{\splita}, u)}$ is 1 when $t_{\splita}$ captures set $u$, 0 otherwise.
\noindent\textit{Combining with the idea of Theorem \ref{thm:look_ahead}, we have}:
\begin{equation}
    R(t', \X, \y) \geq  \mathcal{L}(t_{\fix},\X,\y) + \lambda H_{t} + \lambda + \sum_{u=1}^U  \mathcal{E}_u \cdot \mathbf{1}_{\capt(t_{\splita}, u)}.
\end{equation}}

\begin{proof}
\begin{eqnarray}\nonumber
   \lefteqn{R(t', \X, \y) = \mathcal{L}(t',\X,\y) + \lambda H_{t'}}\\\nonumber
   &=& \mathcal{L}(t'_{\fix},\X,\y) +  \mathcal{L}(t'_{\splita},\X,\y) + \lambda H_{t'}\\
   \label{eq:child_tree_obj}
       &=& \mathcal{L}(t_{\fix},\X,\y) + \mathcal{L}(t'_{\fix} \setminus t_{\fix},\X,\y) + \mathcal{L}(t'_{\splita},\X,\y) + \lambda H_{t'}.
   \end{eqnarray}
Since samples captured by $t_{\splita}$ are captured either by $t'_{\fix} \setminus t_{\fix}$ or $t'_{\splita}$, and equivalence loss cannot be eliminated in any tree, we have:
\begin{equation}\label{eq:split_leaf_equiv_lb}
\mathcal{L}(t'_{\fix} \setminus t_{\fix},\X,\y)+\mathcal{L}(t'_{\splita},\X,\y) \geq \sum_{u=1}^U  \mathcal{E}_u \cdot \mathbf{1}_{\capt(t_{\splita}, u)}.
\end{equation}
Equality is achieved when each leaf in $(t'_{\fix} \setminus t_{\fix}) \cup t'_{\splita}$ captures exactly one set of equivalent points and no other points.
Substituting Equation \ref{eq:split_leaf_equiv_lb} into Equation \ref{eq:child_tree_obj}, we have:
\[R(t', \X, \y) \geq \mathcal{L}(t_{\fix},\X,\y) + \sum_{u=1}^U  \mathcal{E}_u \cdot \mathbf{1}_{\capt(t_{\splita}, u)} + \lambda H_{t'}.\]
Because $H_{t'} \geq H_t + 1$:
\[R(t', \X, \y) \geq  \mathcal{L}(t_{\fix},\X,\y) + \lambda H_{t} + \lambda + \sum_{u=1}^U  \mathcal{E}_u \cdot \mathbf{1}_{\capt(t_{\splita}, u)}.\]
\end{proof}
\subsection{Proof of Theorem \ref{thm:kmeans_lb}}
\textbf{Theorem \ref{thm:kmeans_lb}}\textit{
(k-Means Lower Bound).
Consider tree $t = (t_{\fix}, \delta_{\fix}, t_{\splita}, \delta_{\splita}, K, H_t)$
and any child tree $t' =(t'_{\fix}, \delta'_{\fix}, t'_{\splita}, \delta'_{\splita}, K', H_{t'}) \in 
\sigma(t)$. Let $(\X_{t_{\splita}}, \y_{t_{\splita}})$ be samples captured by the splitting leaves $t_{\splita}$. Then,
\begin{align*}\label{eq:kmeans_lb}
    R(t', \X, \y) \geq & \mathcal{L}(t_{\fix},\X,\y) + \lambda K\\
    & + \min_C \left(\frac{1}{N}\text{k-Means}(C,\y_{t_{\splita}}) + \lambda C\right)
\end{align*}
where $\text{k-Means}(C, \y')$ is the optimal objective of the k-Means algorithm clustering 1D points $\y'$ of size $N'$ into $C$ clusters ($C\geq 1$).}

\begin{proof}
From Equation \ref{eq:child_tree_obj} we know that for any child tree $t'$:
\[R(t', \X, \y) = \mathcal{L}(t_{\fix},\X,\y) + \mathcal{L}(t'_{\fix} \setminus t_{\fix},\X,\y) + \mathcal{L}(t'_{\splita},\X,\y) + \lambda H_{t'}.\]
Rearranging:
\begin{eqnarray}
R(t', \X, \y) &\geq&  \mathcal{L}(t_{\fix},\X,\y) + \lambda K  
+  \mathcal{L}(t'_{\fix} \setminus t_{\fix},\X,\y)  + \mathcal{L}(t'_{\splita},\X,\y)
+ \lambda (H_{t'} - K). \label{eq:kmeans_proof_initial}
\end{eqnarray}
Here,
$\mathcal{L}(t'_{\fix} \setminus t_{\fix},\X,\y) + \mathcal{L}(t'_{\splita},\X,\y)$ can be viewed as the squared loss of one way to assign samples $(\X_{t_{\splita}}, \y_{t_{\splita}})$ to $(H_{t'} - K)$ clusters, and the objective of this assignment is the sum of squared Euclidean distances between every point in $(\X_{t_{\splita}}, \y_{t_{\splita}})$ and its cluster mean, because in regression trees, we predict using the mean of the targets in each leaf. 

We now follow the logic above the statement of the theorem, where the optimal k-Means assignment (which considers labels only) can achieve a better objective than any other assignment of points to leaves (like that of our current tree).
By definition, we have:
\begin{equation}
\mathcal{L}(t'_{\fix} \setminus t_{\fix},\X,\y) + \mathcal{L}(t'_{\splita},\X,\y) 
\label{eq:tree_to_kmeans}
\geq \frac{1}{N}\text{k-Means}(H_{t'} - K,\y_{t_{\splita}}).
\end{equation}
Adding $\lambda (H_{t'} - K)$ to each side, we have
\[\mathcal{L}(t'_{\fix} \setminus t_{\fix},\X,\y) + \mathcal{L}(t'_{\splita},\X,\y) + \lambda (H_{t'} - K)\]
\[\geq \frac{1}{N}\text{k-Means}(H_{t'} - K,\y_{t_{\splita}}) + \lambda (H_{t'} - K).\]
Because
\[\frac{1}{N}\text{k-Means}(H_{t'} - K,\y_{t_{\splita}}) + \lambda (H_{t'} - K)\]
\begin{equation*}
    \geq \min_C \left(\frac{1}{N}\text{k-Means}(C,\y_{t_{\splita}}) + \lambda C\right),
\end{equation*}
we have:
\[\mathcal{L}(t'_{\fix} \setminus t_{\fix},\X,\y) + \mathcal{L}(t'_{\splita},\X,\y) + \lambda (H_{t'} - K)\]
\begin{equation}\label{eq:clustering}
    \geq \min_C \left(\frac{1}{N}\text{k-Means}(C,\y_{t_{\splita}}) + \lambda C\right).
\end{equation}
Substituting Equation \ref{eq:clustering} into Equation \ref{eq:kmeans_proof_initial}, we have:
\begin{equation*}
R(t', \X, \y) \geq  \mathcal{L}(t_{\fix},\X,\y) + \lambda K  + \min_C \left(\frac{1}{N}\text{k-Means}(C,\y_{t_{\splita}}) + \lambda C\right).
\end{equation*}
\end{proof}

\subsection{Proof of Theorem \ref{thm:kmeans_equiv_lb}}
\textbf{Theorem \ref{thm:kmeans_equiv_lb}}\textit{
(k-Means Equivalent Points Lower Bound).
Consider tree $t = (t_{\fix}, \delta_{\fix}, t_{\splita}, \delta_{\splita}, K, H_t)$.
and any child tree $t' =(t'_{\fix}, \delta'_{\fix}, t'_{\splita}, \delta'_{\splita}, K', H_{t'}) \in 
\sigma(t)$. Let $(\X_{t_{\splita}}, \y_{t_{\splita}})$ be samples captured by the splitting leaves $t_{\splita}$. Then, 
\begin{eqnarray}
\nonumber
R(t', \X, y) \geq \mathcal{L}(t_{\fix},\X,\y) + \lambda K +
\min_C \left(\frac{1}{N}\text{Constrained\_k-Means}(C,\X_{\splita},\y_{t_{\splita}}) + \lambda C\right)
\end{eqnarray}
where Constrained\_k-Means is defined in Equation \ref{eq:cons_kmeans}.}

\begin{proof}
The proof is very similar to the k-Means lower bound (Theorem \ref{thm:kmeans_lb}).
From Equation \ref{eq:kmeans_proof_initial}, we have:
\begin{eqnarray}\nonumber
R(t', \X, \y) &\geq&  \mathcal{L}(t_{\fix},\X,\y) + \lambda K  
\\\nonumber
&&  +  \mathcal{L}(t'_{\fix} \setminus t_{\fix},\X,\y)  + \mathcal{L}(t'_{\splita},\X,\y)\\
&&+ \lambda (H_{t'} - K). \nonumber
\end{eqnarray}
View $\mathcal{L}(t'_{\fix} \setminus t_{\fix},\X,\y) + \mathcal{L}(t'_{\splita},\X,\y)$ as the objective of one way to assign samples $(\X_{t_{\splita}}, \y_{t_{\splita}})$ into $(H_{t'} - K)$ clusters, under the constraint that \textit{equivalent points must be assigned to the same cluster}. This gives:
\begin{eqnarray}\nonumber
\lefteqn{\mathcal{L}(t'_{\fix} \setminus t_{\fix},\X,\y) + \mathcal{L}(t'_{\splita},\X,\y)}\\
\label{eq:tree_to_cons_kmeans}
&\geq \frac{1}{N}\text{Constrained\_k-Means}(H_{t'} - K,\X_{t_{\splita}},\y_{t_{\splita}}).&
\end{eqnarray}
The rest would be the same as the proof for Theorem \ref{thm:kmeans_lb}, with $\text{k-Means}(H_{t'} - K,\y_{t_{\splita}})$ replaced by $\text{Constrained\_k-Means}(H_{t'} - K,\X_{t_{\splita}},\y_{t_{\splita}})$.
\end{proof}
\subsection{Proof of Theorem \ref{thm:cons_weighted_kmeans}}
\textbf{Theorem \ref{thm:cons_weighted_kmeans}} \textit{
(Constrained k-Means with Equivalent Points is equivalent to weighted k-Means)
Recall in Definition \ref{def:cons_kmeans},
we have $N'$ 1D target points $\y'$ with feature $\X'$and number of clusters $C$. We also have a constraint that for all points in any equivalent set $u$, they must be assigned to the same leaf.  
Define a modified dataset $(\X_{\modify}, \y_{\modify}, \w_{\modify})$, where all points of equivalent set $u$ in the original dataset $(\X',\y')$ are represented as a single point $(\x_u, y_u, w_u)$, where $\x_u$ is the same as $\x'_u$ the feature vector of equivalent set $u$,
\begin{equation}
    y_u = \frac{1}{|u|}\sum_{(\x'_i, y'_i) \in u} y'_i,
\end{equation} and the weight is the size of the equivalent set $u$\begin{equation}
    w_u = |u|,
\end{equation}
then an optimal clustering of the modified dataset will provide an optimal clustering of the original dataset with the equivalent points constraint from Equation \ref{eq:cons_kmeans}. (All points from the original dataset contributing to a weighted point in the modified dataset will be assigned to the same cluster.)
}
That is, solving the Weighted k-Means problem would result in the same solution(s) as solving the Constrained k-Means problem.
\begin{proof}
Let $U$ be the set of equivalent sets in the original dataset $(\X',\y')$, namely our modified dataset has $|U|$ points,  $(\x_u,y_u, w_u)$ where $ u=1,2, \dots, |U|$.
By the definition of the optimal clustering of the modified dataset, we have
\[
    \argmin_{A} \sum_{u \in U} w_{u} (y_{u} -  z_{A(y_{u})})^2\] 
    \textrm{rewritten as}:
\begin{eqnarray}\nonumber
    \lefteqn{
    \argmin_{A} \sum_{k=1}^C \sum_{ A(y_{u}) = c_k} w_{u}(y_{u} -  z_{c_k})^2
    }
    \\\nonumber
    \hspace{-5pt}&=&\hspace*{-10pt}\argmin_{A}\sum_{k=1}^C \sum_{ A(y_{u}) = c_k} w_{u}y^2_{u} - 2w_{u}y_{u}z_{c_k} +  w_{u}z^2_{c_k}\\\nonumber
   \hspace{-5pt} &=&\hspace*{-10pt}\argmin_{A}
    \sum_{u \in U} w_{u}y^2_{u} + 
    \sum_{k=1}^C \sum_{ A(y_{u}) = c_k}-2w_{u}y_{u}z_{c_k} +  w_{u}z^2_{c_k}.\\ \label{eqn:bigequivbound}
    \end{eqnarray}
     $\sum_{u \in U} w_{u}y^2_{u}$ can be removed because it is a constant, and according to Equation \ref{eq:weighted_cluster_mean}, we have \[ \sum_{A(y_{u}) = c_k} w_u y_u = z_{c_k} \sum_{A(y_{u}) = c_k} w_u. \]
     For cluster $c_k$, we also have \[\sum_{A(y_{u}) = c_k} w_{u}z^2_{c_k} =z^2_{c_k} \sum_{A(y_{u}) = c_k} w_u. \]
    Substituting them into the equation above, we continue:
    \[\textrm{Equation } \ref{eqn:bigequivbound}=\argmin_{A} \sum_{k=1}^C - z^2_{c_k} \sum_{A(y_{u}) = c_k} w_u.\]
    Substituting Equation \ref{eq:weighted_cluster_mean} into it:
    \begin{equation}=\argmin_{A} \sum_{k=1}^C - \left(\frac{\sum_{A(y_u) = c_k} w_u y_u}{\sum_{A(y_u) = c_k} w_u}\right)^2 \sum_{A(y_u) = c_k} w_u.\label{eqn:partofproof}\end{equation}
Recall that $w_u$ is the size of the equivalent set $u$ in the original dataset $(\X', \y')$. We can switch the index of summation back to point $i$ instead of equivalent points sets $u$,
   \[z_{c_k} = \frac{\sum_{A(y_u) = c_k} w_u y_u}{\sum_{A(y_u) = c_k} w_u}= \frac{\sum_{A(y'_i) = c_k} y'_i}{\sum_{A(y'_i) = c_k} \mathbf{1}}.\]
   Then, we continue:
   \begin{eqnarray*}
   \textrm{Equation }\ref{eqn:partofproof} &=&\argmin_{A} \sum_{k=1}^C -\left( \frac{\sum_{A(y'_i) = c_k}y'_i}{\sum_{A(y'_i) = c_k} \mathbf{1}}\right)^2 \sum_{A(y'_i) = c_k} \mathbf{1}\\
    &=&\argmin_{A}\sum_{k=1}^C - z^2_{c_k}\sum_{A(y'_i) = c_k} \mathbf{1} \\
    &=&\argmin_{A}\sum_{k=1}^C \left(-2 z^2_{c_k}\sum_{A(y'_i) = c_k} \mathbf{1} + z^2_{c_k}\sum_{A(y'_i) = c_k} \mathbf{1}\right)\\
    &=&\argmin_{A} \sum_{k=1}^C \sum_{A(y'_i) = c_k} - 2z_{c_k}y'_i + z_{c_k}^2.
    \end{eqnarray*}
    Adding a constant $\sum_{i=1}^{N'}{y'_i}^2$ that does not affect the argmin, we have: 
    \begin{eqnarray*}
    &=&\argmin_{A}\sum_{i=1}^{N'}{y'_i}^2 + \sum_{k=1}^C \sum_{A(y'_i) = c_k} - 2z_{c_k}y'_i + z_{c_k}^2\\
    &=&\argmin_{A} \sum_{k=1}^C \sum_{A(y'_i) = c_k} {y'_i}^2 - 2z_{c_k}y'_i + z_{c_k}^2\\
    &=&\argmin_{A}(y'_i - z_{A(c_k)})^2.
    \end{eqnarray*}
    which is equivalent to the k-Means problem on the original dataset under the constraint that the equivalent points must be in the same cluster. 
\end{proof}
Thus, solving the weighted k-Means problem on the modified dataset provides the same results as solving the constrained k-Means on the original dataset.
\subsection{Proof of Theorem \ref{thm:kmeans_k_convex}}
\textbf{Theorem \ref{thm:kmeans_k_convex}}\textit{
(Convexity of Weighted k-Means Objective in Number of Clusters)
Recall $\text{Weighted\_k-Means}(C,\y',\w )$ from Definition \ref{def:weighted_kmeans} for number of clusters $C$, 1D points $\y'$, and weights $\w$. Then, we have 
\begin{eqnarray}
\text{Weighted\_k-Means}(C-1,\y',\w) + \text{Weighted\_k-Means}(C+1,\y',\w)
 \geq2\times \text{Weighted\_k-Means}(C,\y',\w). \label{eq:kmeans_k_convex2}
\end{eqnarray}
}
\begin{proof}
\begin{eqnarray}\label{eq:kmeans_k_concave}
\nonumber
\text{Weighted\_k-Means}(C-1,\y',\w) 
\nonumber
-\text{Weighted\_k-Means}(C,\y',\w)  \\
\geq \text{Weighted\_k-Means}(C,\y',\w) 
-\text{Weighted\_k-Means}(C+1,\y',\w).
\end{eqnarray}
Equation \ref{eq:kmeans_k_concave} is proved in \citet{aggarwal1994finding}. They show in Application V that a $\text{Weighted\_k-Means}(C,\y',\w)$ problem can get reduced to a ``minimum $C$-link path in a concave Monge DAG with $|\y'| + 1$ nodes''. Equation \ref{eq:kmeans_k_concave} is adapted directly from Corollary 7, rearranging it we have Equation \ref{eq:kmeans_k_convex2}.
\end{proof}

\section{Theorems directly adapted from GOSDT}\label{sec:old_bounds}
\begin{theorem}\label{thm:subtree_lb}
(Hierarchical Objective Lower Bound for Sub-trees). Let $R^c$ be the current best objective so far. Let $t$ be a tree such that the root node is split by a feature, where two sub-trees $t_{\lf}, t_{\ri}$ are generated with $H_{\lf}$ leaves for $t_{\lf}$ and $H_{\ri}$ leaves for $t_{\ri}$. The data captured by the left tree is $(\X_{\lf}, \y_{\lf})$ and the data captured by the right tree is $(\X_{\ri}, \y_{\ri})$. Then, the objective
lower bounds of the left sub-tree and right sub-tree are $b(t_{\lf}, \X_{\lf}, \y_{\lf}) $ and $b(t_{\ri}, \X_{\ri}, \y_{\ri})$, which obey $R(t_{\lf}, \X_{\lf}, \y_{\lf})\geq b(t_{\lf}, \X_{\lf}, \y_{\lf})$, and $R(t_{\ri}, \X_{\ri}, \y_{\ri}) \geq b(t_{\ri}, \X_{\ri}, \y_{\ri})$. If $b(t_{\lf}, \X_{\lf}, \y_{\lf}) > R^c$ or $b(t_{\ri}, \X_{\ri}, \y_{\ri}) > R^c$ or $b(t_{\lf}, \X_{\lf}, \y_{\lf}) + b(t_{\ri}, \X_{\ri}, \y_{\ri}) > R^c$, then $t$ is not an optimal tree, and none of  its child trees are optimal.
\end{theorem}
\begin{proof}
This bound adapts directly from GOSDT \citep{lin2020generalized}, where the proof can be found.
\end{proof}
This bound can be applied to any tree, even if the tree is partially constructed. In a partially constructed tree $t$, if one of its subtrees has objective worse than current best objective $R^c$, we can prune tree $t$ and all of its child trees without constructing the other subtree.

\subsection*{Leaf Bounds}\label{sec:leaf_bound}
The following upper bounds on the number of leaves permit us to prune trees whose leaves exceed these upper bounds.
\begin{theorem}
(Upper Bound on the Number of Leaves). Let $H_t$ be the number of leaves of tree $t$ and let $R^c$ be the current best objective. For any optimal tree $t^{\ast}$ with $H_{t^{\ast}}$ leaves, it is true that:
\begin{equation}\label{eq:leaf_ub}
    H_{t^{\ast}} \leq \min \{ \lfloor R_c/\lambda \rfloor , 2^M\},
\end{equation}
 where $M$ is the number of features.
\end{theorem}
\begin{proof}
This bound adapts directly from OSDT \citep{HuRuSe2019}, where the proof can be found.
\end{proof}

\begin{theorem}\label{thm:parent_num_leaves_ub}
(Parent-specific upper bound on the number of leaves). Let $t = (t_{\textrm{fix}}, \delta_{\textrm{fix}}, t_{\textrm{split}}, \delta_{\textrm{split}}, K, H_t)$ be a tree with child tree $t' = (t'_{\textrm{fix}}, \delta'_{\textrm{fix}}, t'_{\textrm{split}}, \delta'_{\textrm{split}}, K', H_{t'}) \in \sigma(t)$ with $H_{t'}$ leaves is a possibly optimal tree. Then:
\begin{equation}
     H_{t'} \leq \min \left\{ H_t + \left\lfloor \frac{R_c- \mathcal{L}(t_{\fix},\X, \y) - \lambda H_t}{\lambda} \right \rfloor, 2^M\right\}.
\end{equation}
\end{theorem}
\begin{proof}
This bound adapts directly from OSDT \citep{HuRuSe2019}, where the proof can be found.
\end{proof}

\subsection*{Splitting Bounds}\label{sec:split_bound}
When constructing a new child tree $t' =(t'_{\textrm{fix}}, \delta'_{\textrm{fix}}, t'_{\textrm{split}}, \delta'_{\textrm{split}}, K', H_{t'})$, $t'_{\textrm{split}}$ needs to be determined. Splitting bounds help determine which leaves in $t'$ \textit{cannot be further split} and which leaves \textit{must be further split}.

\begin{theorem}\label{thm:incre_pro_split}
(Incremental Progress Bound to Determine Splitting). For any optimal tree $t^{\ast}$, any parent node of its leaves must have loss at least $\geq \lambda$ when considered as a leaf.
\end{theorem}
\begin{proof}
Let $t^{\ast} = \{l_1, l_2, \dots, l_i, l_{i+1}, \dots, l_{H_{t^{\ast}}} \}$ be an optimal tree with $H_{t^{\ast}}$ leaves. $t' = \{l_1, l_2, \dots, l_{i-1}, l_{i+2}, \dots, l_{H_{t^{\ast}}, l_j}\}$ is a tree created by deleting a pair of leaves $l_i$ and  $l_{i+1}$ in $t^{\ast}$ and adding their parent node $l_j$. 
\begin{eqnarray*}
   R(t', \X, \y) -  R(t^{\ast}, \X, \y) &=&  \mathcal{L}(l_j,\X, \y) + \lambda (H_{t^{\ast}} - 1) -  \mathcal{L}(l_i,\X, \y)
    - \mathcal{L}(l_{i+1},\X, \y) - \lambda H_{t^{\ast}}\\
   &\leq&  \mathcal{L}(l_j,\X, \y) - \lambda
\end{eqnarray*}
where  $\mathcal{L}(l,\X, \y) \geq 0$ is loss of leaf $l$. And because $t^{\ast}$ is an optimal tree, $R(t', \X, \y) -  R(t^{\ast}, \X, \y) \geq 0$, and we have:
\[\mathcal{L}(l_j,\X, \y) - \lambda \geq 0.\]
\end{proof}

\begin{theorem}\label{thm:lb_incre_pro}
(Lower Bound on Incremental Progress). Consider any optimal tree $t^{\ast} = \{l_1, l_2, \dots, l_i, l_{i+1}, \dots, l_{H_{t^{\ast}}} \}$ with $H_{t^{\ast}}$ leaves. Let $t' = \{l_1, l_2, \dots, l_{i-1}, l_{i+2}, \dots, l_{H_{t^{\ast}}, l_j}\}$ be a tree created by deleting a pair of leaves $l_i$ and  $l_{i+1}$ in $t^{\ast}$ and adding their parent node $l_j$. The reduction in loss obeys:
\[\mathcal{L}(l_j,\X, \y) -  \mathcal{L}(l_i,\X, \y) - \mathcal{L}(l_{i+1},\X, \y) \geq \lambda.\]
\end{theorem}
\begin{proof}
 \begin{eqnarray*}
 R(t', \X, \y) -  R(t^{\ast}, \X, \y)&=&\mathcal{L}(l_j,\X, \y) + \lambda (H_{t^{\ast}} - 1) -  \mathcal{L}(l_i,\X, \y) - \mathcal{L}(l_{i+1},\X, \y) - \lambda H_{t^{\ast}}\\
 &=& \mathcal{L}(l_j,\X, \y)-  \mathcal{L}(l_i,\X, \y) - \mathcal{L}(l_{i+1},\X, \y) - \lambda.
 \end{eqnarray*}
Since $t^{\ast}$ is an optimal tree, we have $R(t', \X, \y) -  R(t^{\ast}, \X, \y) \geq 0$, and thus:
\[\mathcal{L}(l_j,\X, \y)-  \mathcal{L}(l_i,\X, \y) - \mathcal{L}(l_{i+1},\X, \y) \geq \lambda.\]
\end{proof}

When constructing new trees, if a leaf has loss less than $\lambda$ (it fails to meet Theorem \ref{thm:incre_pro_split}), then it cannot be further split. If a pair of leaves in that tree reduce loss from their parent node by less than $\lambda$ (they fail to meet Theorem \ref{thm:lb_incre_pro}), then at least one of this pair of leaves must be further split to search for optimal trees.

\subsection*{Permutation Bound}\label{sec:permutation}
\begin{theorem}\label{thm:permutation}
(Leaf Permutation Bound). Let $\pi$ be any permutation of $\{1 \dots H_t\}$. Let $t = \{l_1, l_2, \dots,l_{H_{t}}\}$, $T =\{l_{\pi(1)}, l_{\pi(2)}, \dots,l_{\pi(H_t)}\}$, that is, the leaves in $T$ are a permutation of the leaves in $t$. The objective lower bounds of $t$ and $T$ are the same and their child trees correspond to permutations of each other.
\end{theorem}
\begin{proof}
This bound adapts directly from OSDT \citep{HuRuSe2019}, where the proof can be found.
\end{proof}
This bound avoids duplicate computation of trees with leaf permutation.
\subsection*{Subset Bound}\label{sec:subset}
\begin{theorem}\label{thm:subset}
Let $t$ and $T$ to be two trees
with the same root node, where $t$ uses feature $f_1$ to split the root node and $T$ uses feature $f_2$ to split the root node. Let $t_1, t_2$ be subtrees of $t$ under its root node, and $(\X_{t_1}, \y_{t_1}), (\X_{t_2}, \y_{t_2})$ be samples captured by $t_1$ and $t_2$. Similarly, let $T_1, T_2$ be subtrees of $T$ under its root node, and $(\X_{T_1}, \y_{T_1}), (\X_{T_2}, \y_{T_2})$ be samples captured by $T_1$ and $T_2$. Suppose $t_1, t_2 $ are optimal trees for $(\X_{t_1}, \y_{t_1}), (\X_{t_2}, \y_{t_2})$ respectively, and $T_1, T_2$ are optimal trees for $(\X_{T_1}, \y_{T_1}), (\X_{T_2}, \y_{T_2})$ respectively. If $R(t_1,\X_{t_1}, \y_{t_1}) \leq R(T_1,\X_{T_1}, \y_{T_1}) $ and $(\X_{t_2}, \y_{t_2}) \subset (\X_{T_2}, \y_{T_2})$, then $R(t, \X,\y ) \leq R(T, \X,\y )$.
\end{theorem}
\begin{proof}
This bound adapts directly from GOSDT \citep{lin2020generalized}, where the proof can be found.
\end{proof}
Similar to Theorem  \ref{thm:subtree_lb}, this bound ensures that we can safely prune a partially constructed tree without harming optimality. It checks whether subtree $t_1$ has a better objective than $T_1$, despite handling more data.

\input{experiments}

\clearpage

%% file: experiments.tex




\section{Experiment Details}
In the following subsections, we provide details on the data sets, pre-processing, and experimental setup used in \S \ref{sec:exp}.

\subsection{Datasets}\label{sec:dataset}
We use twelve regression datasets. Ten of them are from the UCI Machine Learning Repository \citep{Dua:2019}, including \textbf{Airfoil Self-Noise, Auction Verification, Optical Interconnection Network, Real Estate Valuation, Seoul Bike Sharing Demand, Servo, Synchronous Machine,  Yacht Hydrodynamics, Energy efficiency}, and \textbf{Individual Household Electric Power Consumption}. \textbf{Air quality} comes from \citet{chambers1983graphical} and \textbf{Medical Cost Personal} is from \citet{kaggle:insurance}. We predict the scaled sound pressure level for Airfoil Self-Noise, the runtime of verification procedure for the Auction Verification dataset, the channel utilization for the Optical Interconnection Network dataset, the house price of unit area for the Real Estate Valuation dataset, the rented bike count for the Seoul Bike Sharing Demand dataset, the rise time of a servomechanism for Servo, the excitation current of a synchronous machine for the Synchronous Machine dataset, the residuary resistance per unit weight of displacement for Yacht Hydrodynamics, and the mean ozone in parts per billion from 1300 to 1500 hours at Roosevelt Island for Air Quality, the individual medical costs billed by health insurance for Medical Cost Personal, the heating load and cooling load for Energy Efficiency, and the global active power for Individual Household Electric Power Consumption.

\subsection{Preprocessing}
First, we removed all observations with missing values. 
Second, since we performed hundreds of experiments, each of which required substantial computation time and each needed to be solved to provable optimality, in all cases below where we transformed a continuous feature into binary features, we discretized the feature into equal-width partitions and used one-hot encoding. We preprocessed datasets as follows:\\\\
\textbf{Airfoil Self-Noise (airfoil):} We discretized each of the features \textit{frequency, angle of attack, suction side displacement thickness} into 4 categories.\\  
\textbf{Air Quality (airquality):} We discretized each of features \textit{solar R, temp, wind, day} into 4 categories.\\
\textbf{Auction Verification (auction):} We discretized each continuous variable by using a binary feature to encode a threshold
between each pair of adjacent values; the threshold is set equal to the average of the two surrounding values. The classification label is treated as a categorical feature.\\
\textbf{Optical Interconnection Network (optical):} We discretized each of \textit{processor utilization, channel waiting time, input waiting time, network response time} into 4 categories.\\
\textbf{Real Estate Valuation (real-estate):} We discretized each continuous feature into 4 categories.\\
\textbf{Seoul Bike Sharing Demand (seoul-bike):} We discretized each continuous feature into 4 categories.\\
\textbf{Servo (servo):} We directly use this dataset that only contains categorical features.\\
\textbf{Synchronous Machine (sync):} We discretized each feature into 4 categories.\\
\textbf{Yacht Hydrodynamics (yacht):} We discretized each of \textit{Beam-draught ratio, Froude number} into 4 categories.\\
\textbf{Medical Cost Personal (insurance):} We discretized each of \textit{age, bmi} into 4 categories.\\
\textbf{Energy efficiency (enb-heat, enb-cool)}: We discretized each of \textit{X1 Relative Compactness, X2 Surface Area} into 4 categories. \textbf{enb-heat} predicts heating load and \textbf{enb-cool} predicts cooling load.\\
\textbf{Individual Household Electric Power Consumption (household):} We transformed the \textit{Date} feature into \textit{Month}, \textit{Time} into \textit{Hour}. Then we discretized each of \textit{Month, Hour, Global\_reactive\_power, Voltage, Global\_intensity} into 4 categories.\\
\textbf{Table \ref{tab:data} summarizes all datasets after preprocessing.}
\begin{table}[ht]
    \small
    \centering
    \begin{tabular}{|c|c|c|c|c|}\hline 
    Dataset & Samples & Orig. Features & Encoded Binary Features & Prediction Target\\\hline
    airfoil & 1503 & 5 & 17 & scaled sound pressure level\\\hline
    airquality & 111 & 6 & 17 & ozone \\ \hline
    auction & 2043 & 8 & 48 &verification time \\\hline
    optical & 640 & 9 & 29 & channel utilization \\\hline 
    real-estate & 414 & 6 & 18&house price of unit area \\\hline
    seoul-bike & 8760 & 12 & 32 & rented bike count\\\hline
    servo & 167 & 4 & 15 & class\\\hline
    sync & 557& 4 & 12 & ``If'' (current excitation of synchronous machine)\\\hline
    yacht &308 & 6 & 35 & residuary resistance per unit weight of displacement\\\hline
    insurance & 1338 & 6 & 16 & charges\\\hline
    enb-heat & 768 & 8 & 27 & Y1 heating load\\\hline 
    enb-cool & 768 & 8 & 27 & Y2 cooling load\\\hline
    household & 2,049,280 & 5 & 15 & global active power \\\hline
    \end{tabular}
    \caption{Datasets Summary}
    \label{tab:data}
\end{table}
\subsection{Experimental Platform}
We ran all experiments on a TensorEX TS2-673917-DPN Intel Xeon Gold 6226 Processor, 2.7Ghz (768GB RAM 48 cores). We set a 5-minute time limit for objective optimality (Section \ref{exp:loss_vs_sparsity}) and cross-validation experiments (Section \ref{exp:cv}) and 30-minute time limit for running time and scalability experiments(Section \ref{exp:time_vs_sparsity}, \ref{exp:scalability}, \ref{exp:ablation_lb} and \ref{exp:profile}). The memory limit is 200GB. All algorithms ran single-threaded. 
\subsection{Software Packages}
\textbf{InterpretableAI (IAI):}\\ 
We used OptimalTreeRegressor in version 3.0.1 of IAI (https://docs.interpretable.ai/stable/). We were given a free license for this software. Source code was not available.\\ 
\textbf{EvolutionaryTree (Evtree):} \\
We used the CRAN package of evtree, version 1.0-8 (https://cran.r-project.org/web/packages/evtree/index.html).\\
\textbf{Classification and Regression Tree (CART):}\\ The Python implementation from Sci-Kit Learn 1.1.1.\\
\textbf{Generalized,Unbiased,Interaction Detection and Estimation (GUIDE):} \\We used the executable file(version 40.2) from  (https://pages.stat.wisc.edu/$\sim$loh/guide.html). Source code was not available.\\
\section{Experiment: Loss vs$.$ Sparsity}\label{exp:loss_vs_sparsity}
\noindent\textbf{Collection and Setup:} We ran this experiment on  8 datasets: \textit{airfoil, airquality, real-estate, seoul-bike, servo, sync, yacht, insurance}. We trained models on the entire dataset to measure time to convergence/optimality. All runs that exceeded the time limit of 5 minutes were discarded (for depths of 7-9, time-outs occurred for OSRT and evtree, whereas for depth 6 or less, most runs are within the 5-minute time limits, please see Section \ref{sec:timeout} for run time statistics). For each dataset, we ran algorithms with different configurations:
\begin{itemize}
    \item CART: We ran this algorithm with 8 different configurations: depth limit, $d$, ranging from 2 to 9, and a corresponding maximum leaf limit $2^d$. All other parameters were set to the default.
    \item GUIDE: We ran this algorithm with 8 different configurations: depth limit, $d$, ranging from 2 to 9, and a corresponding maximum leaf limit $2^d$. The minimum leaf node size was set to 2. All other parameters were set to the default.
    \item IAI: We ran this algorithm with $8 \times 20$ different configurations: depth limits ranging from 2 to 9, and 20 different regularization coefficients (0.0001, 0.0002, 0.0005, 0.001, 0.002, 0.003, $\dots$ 0.009, 0.01, 0.035, 0.055, 0.08, 0.1, 0.105, 0.2, 0.5). The random seed was 1. All other parameters were set to the default.
    \item Evtree: We ran this algorithm with $8 \times 20$ different configurations: depth limits ranging from 2 to 9, and 20 different regularization coefficients ($0.1, 0.2, 0.3, \dots , 0.9, 1, 1.1, 1.2 \dots 2 $). The minimum leaf node size was set to 1, and the minimum internal node size was set to 2. The random seed was set to 666. All other parameters were set to the default.
    \item OSRT (our method): We ran this algorithm with $8 \times 20$ different configurations: depth limits ranging from 2 to 9, and 20 different regularization coefficients (0.0001, 0.0002, 0.0005, 0.001, 0.002, 0.003, $\dots$ 0.009, 0.01, 0.035, 0.055, 0.08, 0.1, 0.105, 0.2, 0.5).
\end{itemize}

\textit{Note: OSRT and IAI have the exact same objective function while Evtree has a slightly different regularization term: the regularization coefficient depends on the number of samples, $N$. Therefore we used a different scale of coefficients for Evtree.}\\

\noindent\textbf{Calculations:} We drew one plot (training loss vs$.$ number of leaves) for each combination of dataset and depth limit. Under the same depth limit, runs of one algorithm with different regularization coefficients may generate trees with the same number of leaves. In this case, we plotted the median loss of those trees and showed the best and worst loss among these trees as lower and upper error values respectively. These plots do not display trees where the number of leaves exceeds 30, as these tend to be uninterpretable and overfitted.   

\noindent\textbf{Results:} Figures \ref{fig:lvs:airfoil}-\ref{fig:lvs:seoul-bike}
show that \ourmethod{} consistently produces optimal trees that minimize the objective defined in Equation \ref{eq:obj2}, while IAI and Evtree lose optimality when the depth limit increases and regularization coefficient decreases. \ourmethod{}, which provides provably optimal trees, defines a frontier between training loss and the number of leaves. These plots also show how far away the CART and GUIDE objectives are from the optimal solution. \textit{It is not possible to determine whether the solutions of any method are optimal or close to optimal, without comparing to OSRT, because OSRT is the only method with an optimality guarantee.} IAI and Evtree are likely to find optimal trees under shallow depth constraints (2 and 3), but often fail once the depth constraint is greater than 3. Note that, often, all methods do achieve an optimal solution, which means there is no room for improvement on that problem.

We also observed high variance in the loss of evtrees (and sometimes IAI trees) under the same depth limit, even with a fixed random seed (see Figure \ref{fig:lvs:airfoil} at depths 5 and 6, Figure \ref{fig:lvs:airquality} at depths 5, 6, 7, 8 and 9, Figure \ref{fig:lvs:real-estate} at depths 5, 6 and 8, and Figure \ref{fig:lvs:servo} at depths 7 and 8). Given a dataset and depth limit, if two regularization coefficients result in trees with the same number of leaves, those should have the same loss, otherwise at least one of them must be sub-optimal. We observe that sometimes they can be very far from optimal. We discuss the uncontrollability of IAI and Evtree in Section \ref{exp:control}. 

\input{exp_fig_code/loss_vs_sparsity}

\newpage 
\section{Timeout Statistics}\label{sec:timeout}
We observed timeouts in some of the experiments in Section \ref{exp:loss_vs_sparsity}. Evtree timed out on dataset \textit{insurance} and \textit{seoul-bike}, and OSRT timed out on dataset \textit{seoul-bike}. \textit{Evtree tends to time out more often than OSRT.}
Tables \ref{tab:timeout_insurance} and \ref{tab:timeout_bike} show the number of timed out runs.
\begin{table}[ht]
\setlength{\tabcolsep}{6pt.}
\begin{minipage}{.49\linewidth}
\centering

\caption{Number of Timeout on Dataset \textit{insurance}}
\label{tab:timeout_insurance}

\medskip

\begin{tabularx}{\linewidth}{c X X}
    \toprule
    \textbf{Depth} & \textbf{\#Timeout/total Evtree runs} & \textbf{\#Timeout/total OSRT runs} \\
    \midrule
    4 & 0 & 0\\
    5 & 0 & 0 \\
    6 & 10\% & 0\\
    7 & 60\% & 0 \\
    8 & 70\% & 0\\
    9 & 80\% & 0 \\
    \midrule
    \textbf{Total} & 27.5\% & 0 \\
    \bottomrule
\end{tabularx}
\end{minipage}\hfill
\begin{minipage}{.49\linewidth}
\centering

\caption{Number of Timeout on Dataset \textit{seoul-bike} }
\label{tab:timeout_bike}

\medskip

\begin{tabularx}{\linewidth}{c X X}
    \toprule
    \textbf{Depth} & \textbf{\#Timeout/total Evtree runs} & \textbf{\#Timeout/total OSRT runs } \\
    \midrule
    4 & 10\% & 0 \\
    5 & 40\% & 30\% \\
    6 & 90\% & 45\%\\
    7 & 100\% & 45\% \\
    8 & 100\% & 45\%\\
    9 & 100\% & 45\% \\
    \midrule
    \textbf{Total} & 55\% & 26.25\% \\
    \bottomrule
\end{tabularx}
\end{minipage}\hfill
\end{table}

\section{Experiment: Controllability}\label{exp:control}
\noindent\textbf{Collection and Setup:} We ran this experiment on 2 datasets: \textit{real-estate, servo}. For each dataset, we gave IAI, Evtree and OSRT the same depth limit and regularization coefficient. (Recall that Evtree has different scale of regularization coefficients, we gave it coefficient that produce trees with a similar number of leaves.) For each dataset, we used the configurations described below:
\begin{itemize}
    \item \textit{real-estate}: Set depth limit 4 -- 9 for all three algorithms, regularization coefficients 0.0005, 0.001, 0.005, 0.01, 0.05 for IAI and OSRT, 0.05, 0.08, 0.1, 0.5, 1  for Evtree.
    \item \textit{servo}: Set depth limit 4 -- 9 for all three algorithms, regularization coefficients 0.0005, 0.001, 0.005, 0.01, 0.05 for IAI and OSRT, 0.05, 0.08, 0.1, 0.5, 1  for Evtree.
\end{itemize}

\noindent For each combination of dataset, depth limit, regularization coefficient and algorithm, we ran 10 times with 10 different random seeds. \\\\

\noindent\textbf{Calculations:} For each combination of dataset, regularization coefficient and algorithm, we produced a set of up to 10 trees, depending on if the runs exceeded the time limit. We summarized the measurements of training loss and number of leaves across the set of up to 10 trees by plotting the median of training loss and number of leaves; we showed the minimum and maximum number of leaves, and the best and worst training error in the set as the lower and upper error values respectively.

\noindent\textbf{Results:} Figures \ref{fig:variance:real-estate} and \ref{fig:variance:servo} show that for a given dataset, depth constraint and regularization coefficient, \textit{OSRT consistently found the same optimal tree, while IAI and Evtree tended to depend on the random seeds.} Trees generated by IAI and Evtree could vary a lot in terms of prediction accuracy and sparsity. Poor sparsity can produce models that are uninterpretable and overfitted; IAI and Evtree are unable to avoid these phenomena.  
\input{exp_fig_code/variance}
\newpage
\section{Experiment: Time vs$.$ Sparsity}\label{exp:time_vs_sparsity}
\noindent\textbf{Collection and Setup:} We compared the running time of different methods that produce trees of similar size, by running this experiment on 6 datasets: \textit{airfoil, real-estate, servo, sync, yacht, insurance}. We used the same configurations in Section \ref{exp:loss_vs_sparsity}, except we set the time limit to 30 minutes in this experiment and set regularization coefficients $(0.10, 0.11, 0.12 \dots, 0.99, 1.00, 1.1, 1.2, \dots, 1.9, 2.0)$ for Evtree.

\noindent\textbf{Calculations:} We created one  plot (time vs number of leaves) for each combination of dataset and depth limit. We again excluded trees with more than 30 leaves. Given a dataset and depth limit, one algorithm might produce multiple trees with the same number of leaves, so we plotted the median of these training times and show the best and worst time as lower and upper error values respectively. 

\noindent\textbf{Results:} Figures \ref{fig:tvs:sync}-\ref{fig:tvs:yacht}
show that \textit{\ourmethod{} is generally faster than Evtree.} Recall that trees of IAI and Evtree become sub-optimal once the depth limit is greater than 3, and OSRT is only slightly slower than IAI and much faster than Evtree, which means our method is \textit{both accurate and fast.} Evtree tends to produce large and uninterpretable trees when the depth limit is greater than 7, therefore there are only few evtree points in some plots. OSRT slows down when the optimal trees for given depth and regularization are overfitted (due to correlated variables), which is shown in Section \ref{exp:cv}. OSRT typically performs well for interpretable and sparse trees. 

Evtree tends to converge to sub-optimal trees early if the depth is large and regularization is small, when the difference of the objective between trees is also relatively small, since Evtree will terminate when its best 5\% population of trees stabilizes. 

We observe ``abort'' issues when we run IAI on datasets that need more training time (e.g., \textit{auction, seoul-bike and optical}). Error messages (see Listing \ref{iai_error}) in calls of the local search function indicate IAI is using a local search approach, which may explain why its running time stays relatively unchanged with increasing tree depth. Since IAI is proprietary, we cannot confirm its algorithmic approach.
\begin{lstlisting}[caption={IAI Error Code},label={iai_error},language=Python]
RuntimeError: <PyCall.jlwrap (in a Julia function called from Python)
JULIA: obj mismatch: before -2.384185791015625e-7 after 0.0
Stacktrace:
  [1] error(s::String)
  [2] greedy_search!(tree::IAITrees.Tree{IAIBase.RegressionTask, IAITrees.Node{IAIBase.RegressionTask, IAIBase.RegressionFit}}, gs::OptimalTrees.LocalSearcher{IAIBase.RegressionTask, OptimalTrees.RegressionEvaluatorMSEConstant, IAIBase.RegressionTarget})
  [3] run_worker_iteration!(rep::Int64, show_progress::Bool, ls::OptimalTrees.LocalSearcher{IAIBase.RegressionTask, OptimalTrees.RegressionEvaluatorMSEConstant, IAIBase.RegressionTarget}, gs::OptimalTrees.LocalSearcher{IAIBase.RegressionTask, OptimalTrees.RegressionEvaluatorMSEConstant, IAIBase.RegressionTarget}, rng_gen::IAIBase.RandomStreams.MRG32k3aGen)
  [4] run_task(tid::Int64, f_run::Function, f_setup::Function, job_channel::Channel{Int64}, results_channel::Channel{Any}, progress::IAIBase.Progress, uses_subprogress::Bool, obj::Tuple{OptimalTrees.LocalSearcher{IAIBase.RegressionTask, OptimalTrees.RegressionEvaluatorMSEConstant, IAIBase.RegressionTarget}, OptimalTrees.LocalSearcher{IAIBase.RegressionTask, OptimalTrees.RegressionEvaluatorMSEConstant, IAIBase.RegressionTarget}})
  [5] spawn_tasks(::Int64, ::Int64, ::Function, ::Vararg{Any})
  [6] run_distributed!(f_consume::Function, obj::Tuple{OptimalTrees.LocalSearcher{IAIBase.RegressionTask, OptimalTrees.RegressionEvaluatorMSEConstant, IAIBase.RegressionTarget}, OptimalTrees.LocalSearcher{IAIBase.RegressionTask, OptimalTrees.RegressionEvaluatorMSEConstant, IAIBase.RegressionTarget}}, f_run::Function, f_setup::Function, n_jobs::Int64, procs::Vector{Int64}, n_threads::Int64, progress::IAIBase.Progress, uses_subprogress::Bool)
  [7] run_distributed(show_progress::Bool, procs::Vector{Int64}, n_threads::Int64, n_jobs::Int64, message::String; iter_func::Function, iter_input::Tuple{OptimalTrees.LocalSearcher{IAIBase.RegressionTask, OptimalTrees.RegressionEvaluatorMSEConstant, IAIBase.RegressionTarget}, OptimalTrees.LocalSearcher{IAIBase.RegressionTask, OptimalTrees.RegressionEvaluatorMSEConstant, IAIBase.RegressionTarget}}, iter_setup::Function, iter_uses_subprogress::Bool, consume_func::Function)
  [8] (::IAIBase.var"#run_distributed##kw")(::NamedTuple{(:iter_input, :iter_setup, :iter_func, :iter_uses_subprogress, :consume_func), Tuple{Tuple{OptimalTrees.LocalSearcher{IAIBase.RegressionTask, OptimalTrees.RegressionEvaluatorMSEConstant, IAIBase.RegressionTarget}, OptimalTrees.LocalSearcher{IAIBase.RegressionTask, OptimalTrees.RegressionEvaluatorMSEConstant, IAIBase.RegressionTarget}}, typeof(OptimalTrees.task_local_copy), typeof(OptimalTrees.run_worker_iteration!), Bool, OptimalTrees.var"#95#97"{OptimalTrees.OptimalTreeRegressor}}})
\end{lstlisting}
\input{exp_fig_code/time_vs_sparsity}

\newpage 
\section{Experiment: Cross Validation}\label{exp:cv}
\noindent\textbf{Collection and Setup:}
We ran 5-fold cross-validation on the 5 datasets: \textit{airfoil, sync, servo, seoul-bike, insurance}. The time limit was set to 5 minutes. For each dataset, we ran algorithms with different configurations:
\begin{itemize}
    \item CART: We ran this algorithm with 4 different configurations: depth limit, $d$, ranging from 2 to 5, and a corresponding maximum leaf limit $2^d$. All other parameters were set to their default.
    \item GUIDE: We ran this algorithm with 4 different configurations: depth limit, $d$, ranging from 2 to 5, and a corresponding maximum leaf limit $2^d$. The minimum leaf node size was set to 2. All other parameters were set to the default.
    \item IAI: We ran this algorithm with $4 \times 10$ different configurations: depth limits ranging from 2 to 5, and 10 different regularization coefficients (0.1, 0.05, 0.025, 0.01, 0.0075, 0.005, 0.0025, 0.001, 0.0005, 0.0001). 
    \item Evtree: We ran this algorithm with $4 \times 20$ different configurations: depth limits ranging from 2 to 5, and 20 different regularization coefficients (0.1, 0.2, 0.3, 0.4, 0.5, 0.6, 0.7, 0.8, 0.9, 1.0, 1.1, 1.2, 1.3, 1.4, 1.5, 1.6, 1.7, 1.8, 1.9, 2). The minimum leaf node size was 1, minimum internal node size was 2. All other parameters were set to the default.
    \item OSRT: We ran this algorithm with $4 \times 10$ different configurations: depth limits ranging from 2 to 5, and 10 different regularization coefficients (0.1, 0.05, 0.025, 0.01, 0.0075, 0.005, 0.0025, 0.001, 0.0005, 0.0001).
\end{itemize}

\noindent\textbf{Calculations:} We drew one plot per dataset and depth. For each combination of regularization coefficient and algorithm in the same plot, we produced a set of up to 5 trees, depending on if the runs exceeded the time limit. We summarized the measurements of training loss, testing loss and number of leaves across
the set of up to 5 trees by plotting the median, showing the minimum and maximum number of leaves, the best and worst
training/testing error in the set as the lower and upper error values respectively.

\noindent\textbf{Results:} Figure \ref{fig:cv:airfoil} to \ref{fig:cv:insurance} show that \ourmethod{} trees produce the lowest testing loss among all the regression trees. We noticed that the generalization performance of GUIDE is much worse than that of the other four methods. (GUIDE trees are not shown, because their testing loss is over four times higher than that of other methods). If an optimal tree significantly outperforms other sub-optimal trees in terms of training performance, it is also outperformed in testing (e.g., \textit{airfoil} depth 5), otherwise the difference in testing loss becomes insignificant due to generalization error. Note that large trees start overfitting when depth is greater than 4 or 5, and sparse trees tend to have better generalization.

\input{exp_fig_code/cv}

\newpage 
\section{Experiment: Scalability}\label{exp:scalability}
\noindent\textbf{Collection and Setup:} We ran this experiment on the dataset \textit{household} for CART, IAI, Evtree and OSRT. We subsampled 100, 500, 1000, 5000, 10,000, 50,000, 100,000, 500,000, 1000,000 samples from it to form 10 datasets (including the original dataset). The depth limit was set to 5 (for all methods) and regularization coefficient to 0.035 for IAI and OSRT and 0.2 for Evtree. We set the time limit of each run to 30 minutes in this experiment.\\
\noindent\textbf{Results:} Figure \ref{fig:appendix:scalability_full} shows that the scalability of our method is significantly better than Evtree, and it generally performs better than IAI. Evtree timed out when the sample size was larger than 50,000. Figure \ref{fig:appendix:scalability} is zoomed in to the regime of low training times, showing that our method is faster than all but CART when sample size is less than 10,000. 
\begin{figure}
    \centering
    \includegraphics[width=0.5\textwidth]{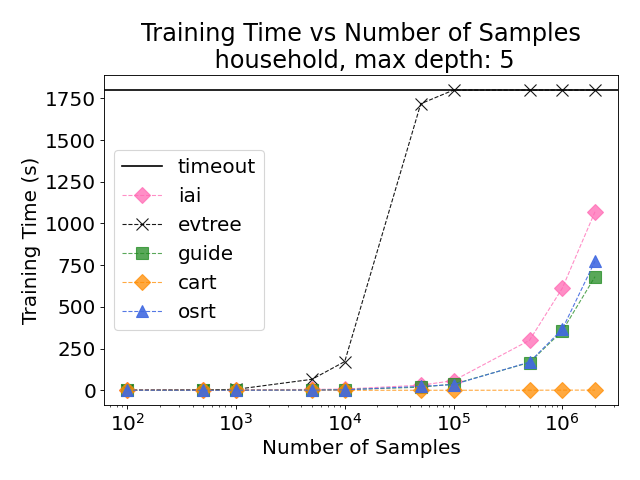}
    \caption{Training time of CART, GUIDE, IAI, Evtree
and OSRT as a function of sample size on dataset: house-hold, $d = 5, \lambda = 0.035$. (30-minute time limit).}
    \label{fig:appendix:scalability_full}
\end{figure}
\begin{figure*}[ht]
    \centering
    \includegraphics[width=0.5\textwidth]{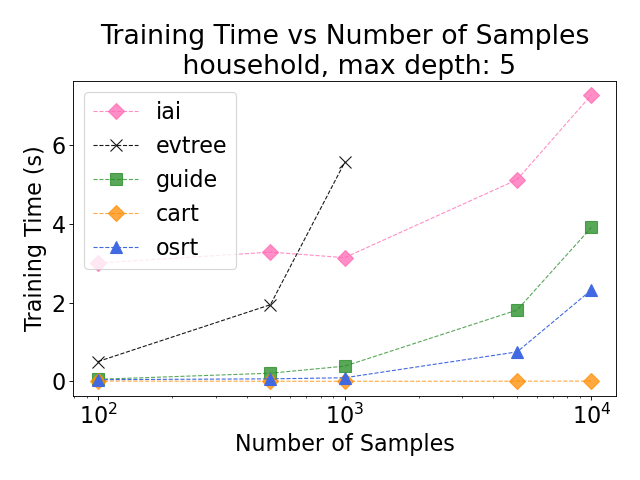}
    \caption{(Zoomed in) Training time of CART, GUIDE, IAI, Evtree
and OSRT as a function of sample size on dataset: house-
hold, $d = 5, \lambda = 0.035$. 30-minute time limit; the plot omits evtree results after 5000 samples (65.69 seconds for 5000 samples, 172.143 seconds for 10,000 samples, 1718.138 seconds for 50,000 samples, time out for the rest).}
    \label{fig:appendix:scalability}
\end{figure*}

\section{Experiment: Ablation}\label{exp:ablation}
\subsection{Value of k-Means Lower Bound}\label{exp:ablation_lb}
We explored how much our new \textit{k-Means} lower bound contributes to speeding up the optimization. Recall that our method reaches optimality when the current best objective score (upper bound) converges with the objective lower bound.\\
\noindent\textbf{Collection and Setup:} We ran this experiment on 5 datasets \textit{(airquality, enb-cool, enb-heat, sync, yacht)} for variations of \ourmethod. We set depth limit to 6 and regularization coefficient to 0.005 for all datasets. For each dataset, we ran \ourmethod{} twice, once using the \textit{k-Means lower bound} and once using the \textit{equivalent point lower bound}. We set the time limit to 30 minutes in this experiment.

\noindent\textbf{Calculations:} We recorded the elapsed time, iterations, and size of dependency graph when each run of \ourmethod{} converged or timed out.

\noindent\textbf{Results:} Figure \ref{fig:ablation:lb} shows that our novel \textit{k-Means lower bound} is significantly faster than the \textit{equivalent points lower bound}. We noticed substantially better running time when the k-Means lower bound is much tighter than the equivalent points lower bound. It typically reduces runtime by a factor of two and sometimes more. The iteration number and graph size in our optimization framework were also reduced when using the k-Means lower bound, which means less memory was used.
\input{exp_fig_code/ablation_lb}
\subsection{Depth Constraint}\label{exp:ablation_depth}
Recall that our method can optimize the objective score \textit{without depth constraints} while other methods cannot. Adding a depth constraint significantly reduces the search space, so being able to work without a depth constraint, OSRT can solve \textit{much harder} problems than other methods. For most datasets we used in previous sections, a depth constraint substantially reduced OSRT searching time; but we also observed  cases where, given a relatively small regularization coefficient and a large depth limit on some datasets (e.g., \textit{sync, airquality}), \ourmethod{} took longer to reach optimality compared to optimizing with the regularization penalty only. For the \textit{airquality} dataset, we used $\lambda = 0.001$ and ran OSRT with depth limit 7 and with no depth limit: OSRT without the depth limit took 115 seconds to reach optimality while the other run took 203 seconds. This is because the hard depth constraint prevents the algorithm from re-using the bounds it has calculated.  
\section{Experiment: Execution Trace}\label{exp:profile}
Recall that OSRT keeps tracking the lower and upper bound of root problem objective, and optimal trees are found when two bounds converge.
\noindent\textbf{Collection and Setup:} We ran this experiment on 4 datasets \textit{(airfoil, airquality, optical, insurance)} only for \ourmethod. We used two configurations: $(d=4, \lambda = 0.05)$ and $(d= 7, \lambda = 0.001)$. For each combination of dataset and configuration, we ran \ourmethod{} and recorded the lower and upper bounds of the objective score during the algorithm execution. We set the time limit as 30 minutes in this experiment.

\noindent\textbf{Results:} Figure \ref{fig:trace} shows the optimality gap of the objective score during execution of our method. If our method does not reach optimality (e.g., the \textit{optical} dataset) before the time limit, we can still gain insight by examining the difference between the current best objective and the current lower bound while IAI and Evtree do not provide such information.
\input{exp_fig_code/execution_trace}
\section{Optimal Regression Trees}\label{exp:opt_tree}
We now visually compare the optimal trees found by OSRT and trees. Figure \ref{fig:trees_6leaves} contains two 6-leaf trees generated by OSRT and Evtree respectively. The OSRT tree has training mean squared error (MSE) of 247.50, $R^2: 77.44\%$ while Evtree has training MSE of 324.61, $R^2: 70.41\%$. OSRT explains 7\% more training data variance than Evtree with the same number of leaves. Figures \ref{fig:osrt_tree2} and \ref{fig:iai_tree} show two 13-leaf trees generated by OSRT and IAI respectively. The OSRT tree has training MSE of 127.46, $R^2: 88.38\%$ while Evtree has training MSE of 156.59, $R^2: 85.71\%$. OSRT explains 2.66\% more training data variance than Evtree with the same number of leaves. Figures \ref{fig:osrt_tree3} and \ref{fig:iai_evtree_tree} show two 10-leaf trees generated by OSRT, IAI and Evtree (IAI and Evtree found the same tree). The OSRT tree has training MSE of 0.42838, $R^2: 82.28\%$ while IAI and Evtree have training MSE of 0.45055, $R^2: 81.36\%$. We noticed that these two trees have some identical leaves.
\begin{figure}
    \centering
    \begin{subfigure}[b]{0.45\textwidth}
    \centering
    \begin{forest} [ {$ 77 < Temp \leq 87$}  [ {$6.9 < Wind  \leq  11.5$} ,edge label={node[midway,left,font=\scriptsize]{True}}  [ $\mathbf{43.73}$  ]  [ {$11.5 < Wind \leq 16.1$}   [ $\mathbf{30.86}$  ]  [ {$23.5 < Day\leq 31 $}   [ $\mathbf{141.50}$  ]   [ $\mathbf{79.60}$  ]  ] ] ] [ {$87 < Temp\leq 97$} ,edge label={node[midway,right,font=\scriptsize]{False}}  [ $\mathbf{90.06}$  ]   [ $\mathbf{18.66}$  ]  ] ] \end{forest}\caption{OSRT, Training loss: $247.5, R^2: 77.44\%$}
    \end{subfigure}
    \begin{subfigure}[b]{0.45\textwidth}
    \centering
    \begin{forest} [ {$ 87 < Temp\leq 97$} [$\mathbf{90.06}$, edge label={node[midway,left,font=\scriptsize]{True}}]   [ {$77 < Temp\leq 87$} ,edge label={node[midway,right,font=\scriptsize]{False}}  [ {$Month = 7$} [{$8.5 < Day\leq 16$} [$\mathbf{27.33}$] [$\mathbf{61.98}$] ] [{$23.5 < Day\leq 31$} [$\mathbf{89.2}$] [$\mathbf{34.5}$] ]  ]   [ $\mathbf{18.66}$  ]  ] ] \end{forest}\caption{Evtree, Training loss: $324.609, R^2: 70.41\%$}
    \end{subfigure}
    
    \caption{Regression trees produced by OSRT and Evtree for \textbf{airquality} dataset with 6 leaves.}
    \label{fig:trees_6leaves}
\end{figure}

\begin{figure}
    \centering
   \begin{forest} [ {$77 < Temp\leq 87$}  [ {$ 23.5 < Day \leq 31$} ,edge label={node[midway,left,font=\scriptsize]{True}} [ {$ 170.5 < Solar.R\leq 252.25 $}  [ {$Month = 8$}  [ {$6.9 < Wind\leq 11.5  =  1$}   [ $\mathbf{59.00}$  ]   [ $\mathbf{168.00}$  ]  ] [ {$11.5 < Wind\leq 16.1$}   [ $\mathbf{45.00}$  ]   [ $\mathbf{111.50}$  ]  ] ] [ {$252.25 < Solar.R\leq 334$}   [ $\mathbf{63.25}$  ]   [ $\mathbf{36.00}$  ]  ] ] [ {$6.9 < Wind\leq 11.5$}   [ $\mathbf{36.95}$  ]  [ {$252.25< Solar.R\leq 334$}  [ {$16 < Day\leq 23.5 $}   [ $\mathbf{61.00}$  ]   [ $\mathbf{135.00}$  ]  ] [ {$11.5 < Wind\leq 16.1$}   [ $\mathbf{23.80}$  ]   [ $\mathbf{67.33}$  ]  ] ] ] ] [ {$87 < Temp\leq 97$} ,edge label={node[midway,right,font=\scriptsize]{False}}  [ $\mathbf{90.06}$  ]   [ $\mathbf{18.66}$  ]  ] ] \end{forest}
    \caption{Optimal regression tree produced by OSRT for \textbf{airquality} dataset with 13 leaves. Training loss: $127.46, R^2: 88.38\%$}
    \label{fig:osrt_tree2}
\end{figure}

\begin{figure}
    \centering
   \begin{forest} [ {$77 < Temp\leq 87$}  [ {$11.5 < Wind\leq 16.1$},edge label={node[midway,left,font=\scriptsize]{True}}[{$23.5 < Day\leq31$}[$\mathbf{48.5}$] [{$Month = 9$} [$\mathbf{38}$] [$\mathbf{14.33}$]] ] [{$6.9 < Wind\leq 11.5$} [{$23.5 < Day\leq 31$} [{$170.5 < Solar.R\leq 252.25$} [$\mathbf{75.33}$] [$\mathbf{54.6}$] ] [$\mathbf{36.95}$] ] [{$16 < Day\leq 23.5$} [$\mathbf{66.33}$] [{$Month = 8 $}[$\mathbf{168}$] [$\mathbf{104.7}$]]]] ] [ {$87 < Temp\leq 97$} ,edge label={node[midway,right,font=\scriptsize]{False}}  [ {$Month = 8$} [{$170.5 < Solar.R\leq 252.25$}[$\mathbf{93.67}$] [$\mathbf{122}$] ] [$\mathbf{84.7}$]  ]   [ $\mathbf{18.66}$  ]  ] ] \end{forest}
    \caption{Sub-optimal tree produced by IAI for \textbf{airquality} dataset with 13 leaves. Training loss: $156.59, R^2: 85.72\%$}
    \label{fig:iai_tree}
\end{figure}

\begin{figure}
    \centering
    \begin{forest} [ {$pgain = 4$}   [ $\mathbf{0.67}$ ,edge label={node[midway,left,font=\scriptsize]{True}} ]  [ {$pgain = 5$} ,edge label={node[midway,right,font=\scriptsize]{False}}  [ $\mathbf{0.53}$  ]  [ {$pgain = 6$}   [ $\mathbf{0.52}$  ]  [ {$motor = E$}  [ {$screw = D$}   [ $\mathbf{1.10}$  ]  [ {$screw = E$}   [ $\mathbf{0.90}$  ]   [ $\mathbf{3.33}$  ]  ] ] [ {$motor = C$}  [ {$screw = D$}   [ $\mathbf{1.70}$  ]   [ $\mathbf{3.67}$  ]  ] [ {$motor = D$}   [ $\mathbf{1.40}$  ]   [ $\mathbf{4.51}$  ]  ] ] ] ] ] ] \end{forest}
    \caption{Optimal regression tree produced by OSRT for \textbf{servo} dataset with 10 leaves. Training loss: 0.42838, $R^2: 82.28\%$ }
    \label{fig:osrt_tree3}
\end{figure}

\begin{figure}
    \centering
    \begin{forest} [ {$pgain = 5$}   [ $\mathbf{0.529}$ ,edge label={node[midway,left,font=\scriptsize]{True}} ]  [ {$pgain = 6$} ,edge label={node[midway,right,font=\scriptsize]{False}}  [ $\mathbf{0.515}$  ]  [ {$pgain = 4$}   [ {$vgain = 3$} [$\mathbf{0.933}$] [$\mathbf{0.525}$]  ]  [ {$motor = E$}  [ {$screw = E$}   [ $\mathbf{0.9}$  ]  [ {$screw = D$}   [ $\mathbf{1.1}$  ]   [ $\mathbf{3.333}$  ]  ] ] [ {$motor = D$}  [ $\mathbf{1.40}$ ] [ {$motor = C$}   [ $\mathbf{3.28}$  ]   [ $\mathbf{4.51}$  ]  ] ] ] ] ] ] \end{forest}
    \caption{Sub-optimal tree produced by IAI and Evtree for \textbf{servo} dataset with 10 leaves. Training loss: $0.45055, R^2: 81.36\%$ }
    \label{fig:iai_evtree_tree}
\end{figure}

%% file: exp_fig_code/loss_vs_sparsity.tex
\begin{figure*}[htbp]
    \centering
    \includegraphics[width=0.4\textwidth]{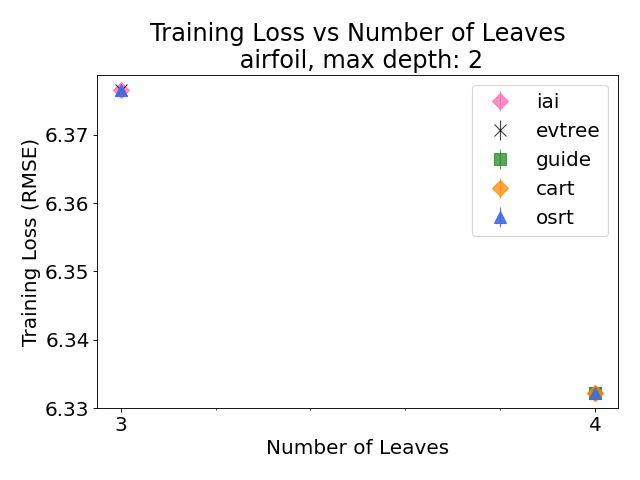}
    \includegraphics[width=0.4\textwidth]{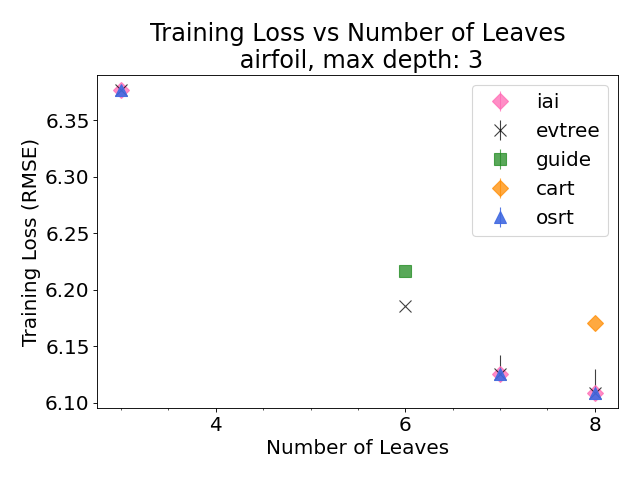}
    \includegraphics[width=0.4\textwidth]{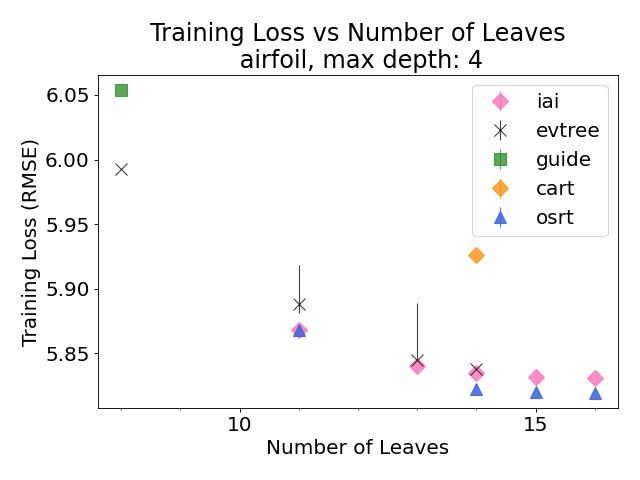}
    \includegraphics[width=0.4\textwidth]{figures/binary_settings/loss_vs_sparsity/airfoil_depth_5.png}
    \includegraphics[width=0.4\textwidth]{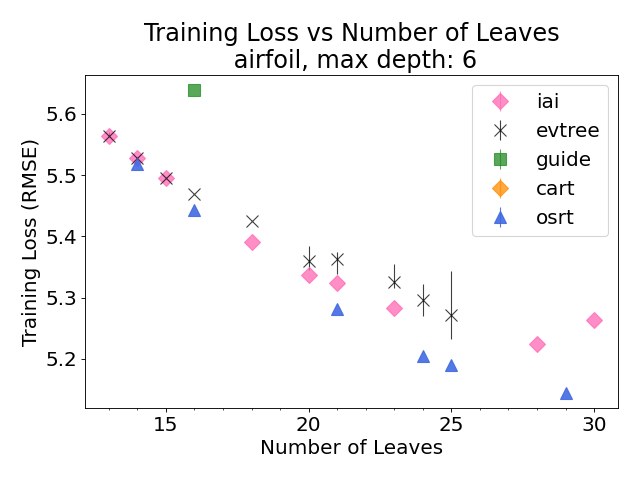}
    \includegraphics[width=0.4\textwidth]{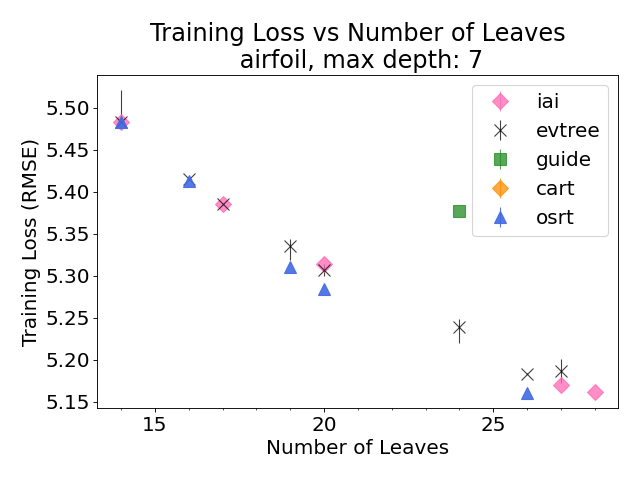}
    \includegraphics[width=0.4\textwidth]{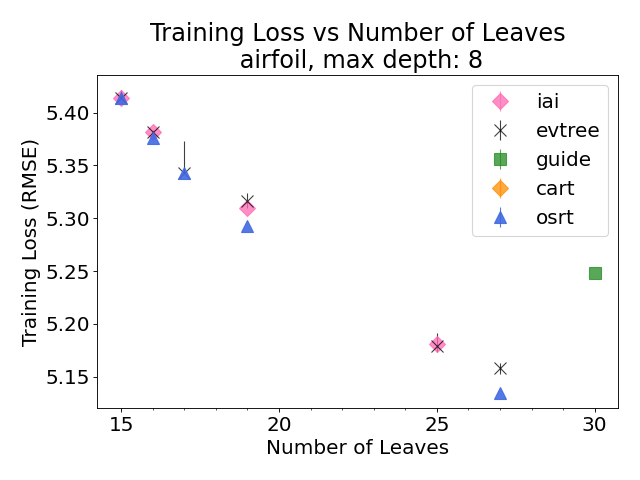}
    \includegraphics[width=0.4\textwidth]{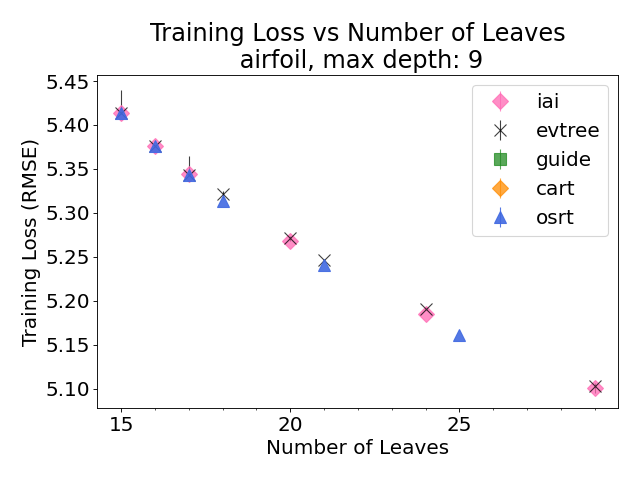}
    \caption{Training loss achieved by IAI, Evtree, GUIDE, CART and OSRT as a function of number of leaves on dataset: airfoil.}
    \label{fig:lvs:airfoil}
\end{figure*}

\begin{figure*}[htbp]
    \centering
    \includegraphics[width=0.4\textwidth]{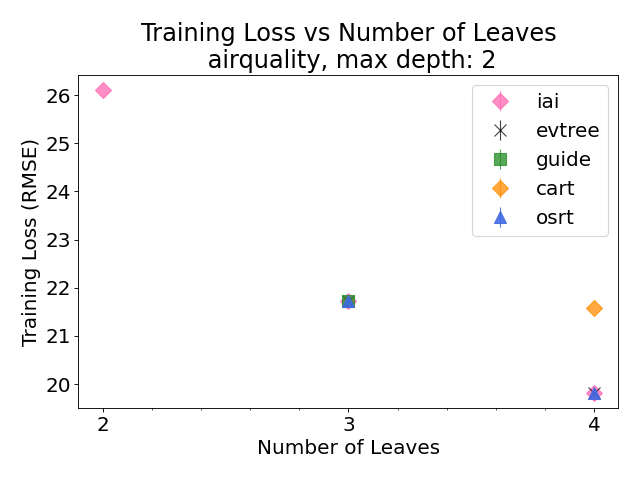}
    \includegraphics[width=0.4\textwidth]{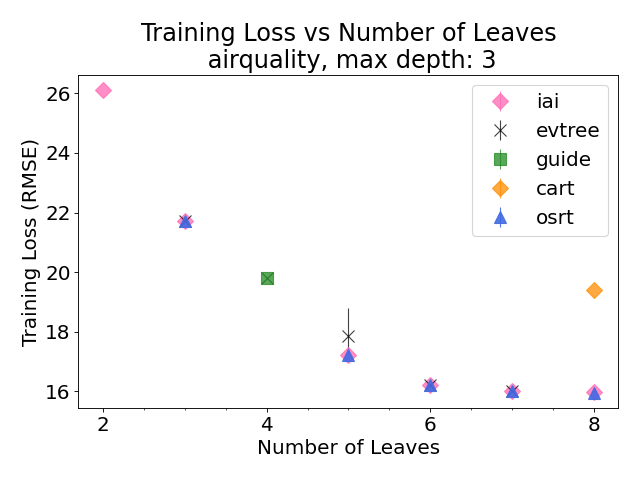}
    \includegraphics[width=0.4\textwidth]{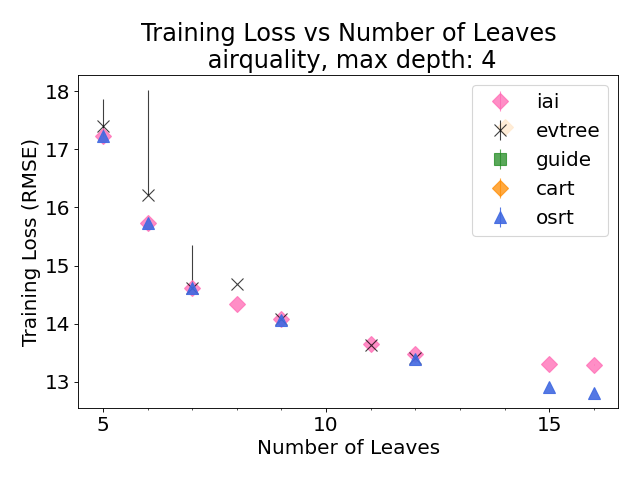}
    \includegraphics[width=0.4\textwidth]{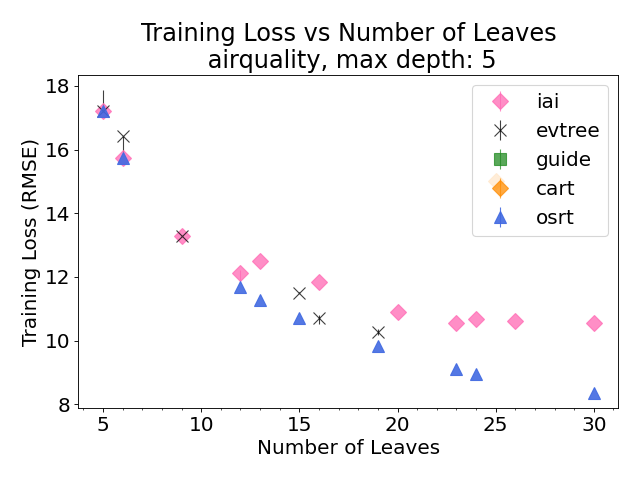}
    \includegraphics[width=0.4\textwidth]{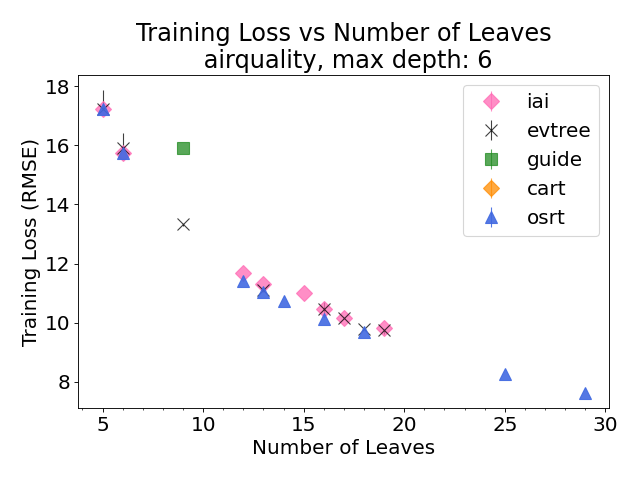}
    \includegraphics[width=0.4\textwidth]{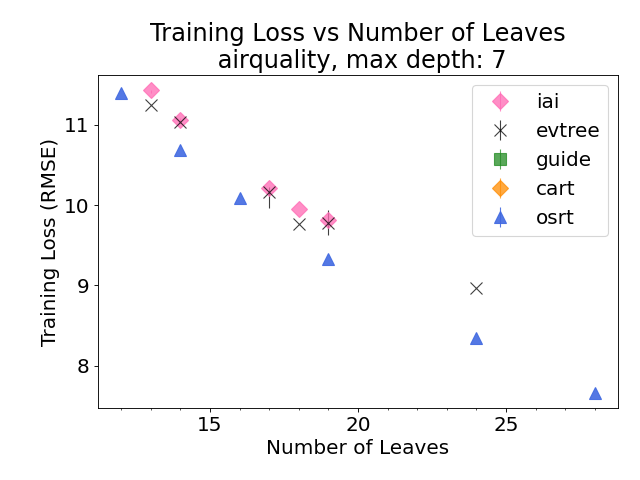}
    \includegraphics[width=0.4\textwidth]{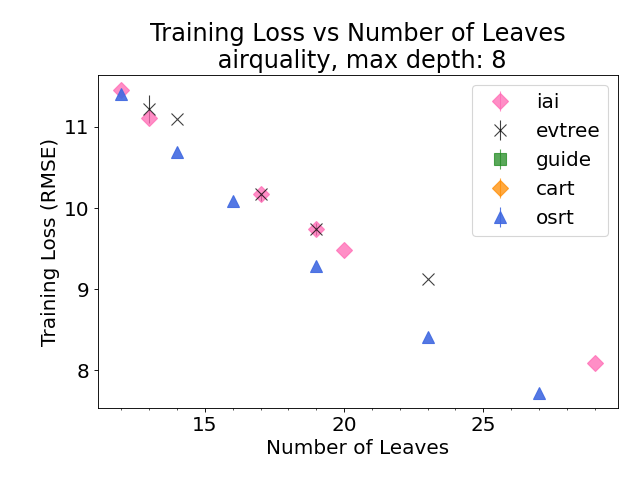}
    \includegraphics[width=0.4\textwidth]{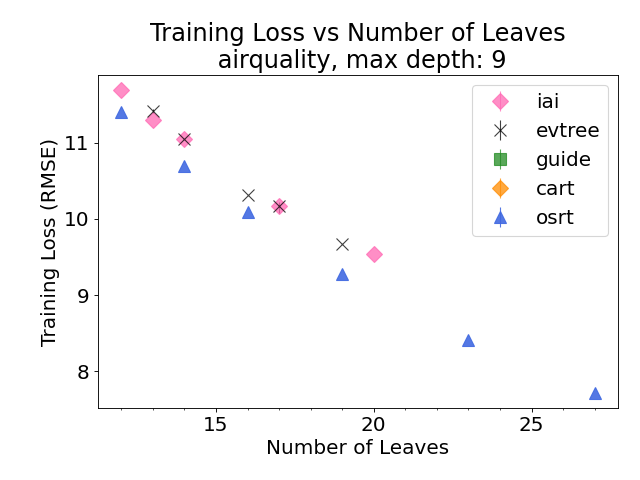}
    
    \caption{Training loss achieved by IAI, Evtree, GUIDE, CART and OSRT as a function of number of leaves on dataset: airquality.}
    \label{fig:lvs:airquality}
    
\end{figure*}

\begin{figure*}[htbp]
    \centering
    \includegraphics[width=0.4\textwidth]{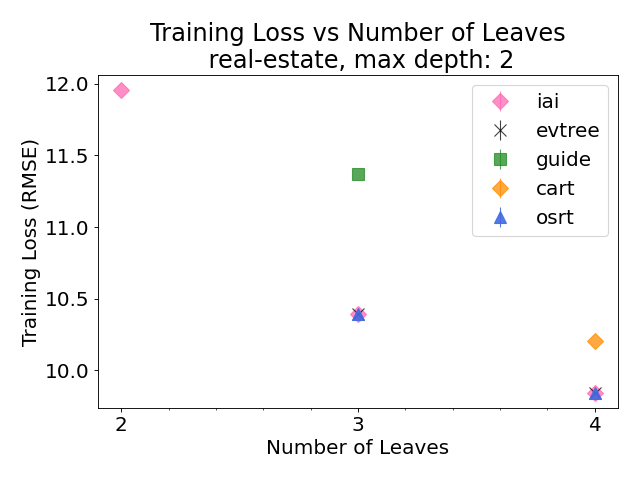}
    \includegraphics[width=0.4\textwidth]{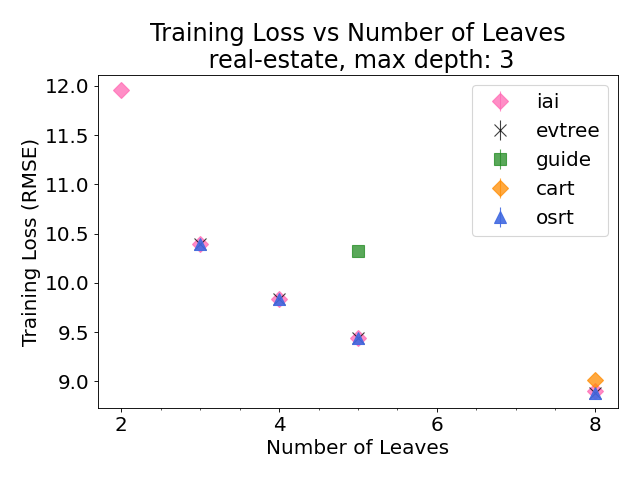}
    \includegraphics[width=0.4\textwidth]{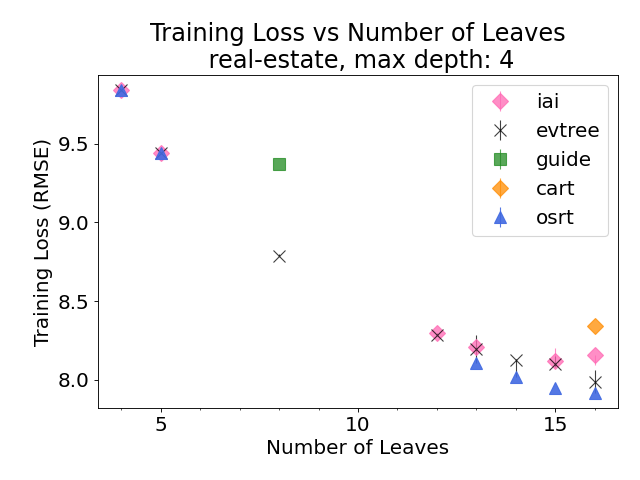}
    \includegraphics[width=0.4\textwidth]{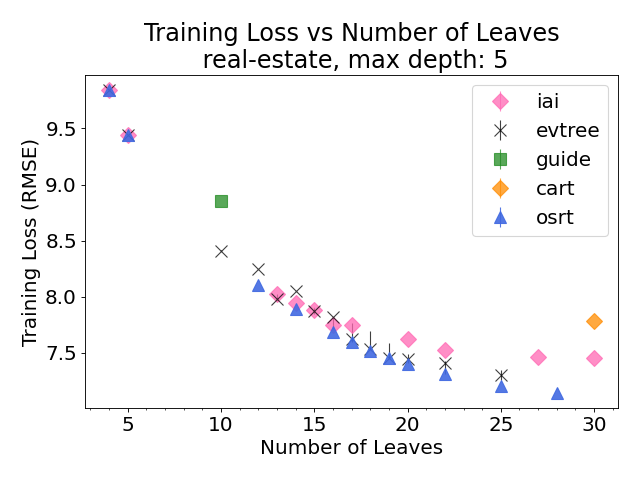}
    \includegraphics[width=0.4\textwidth]{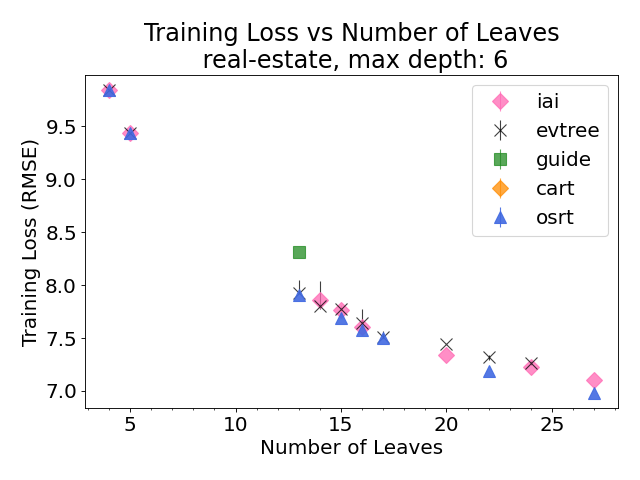}
    \includegraphics[width=0.4\textwidth]{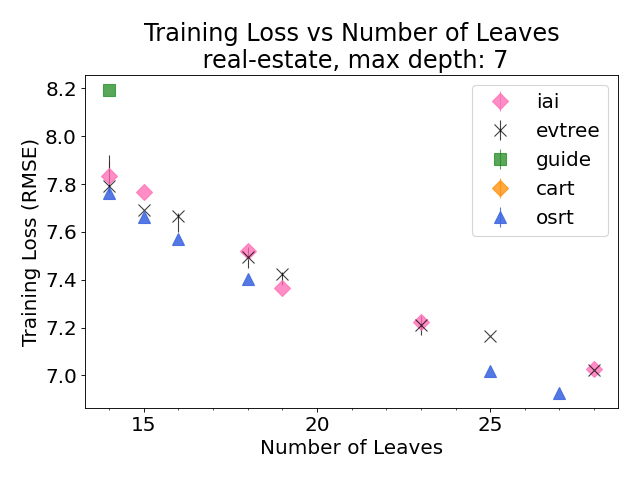}
    \includegraphics[width=0.4\textwidth]{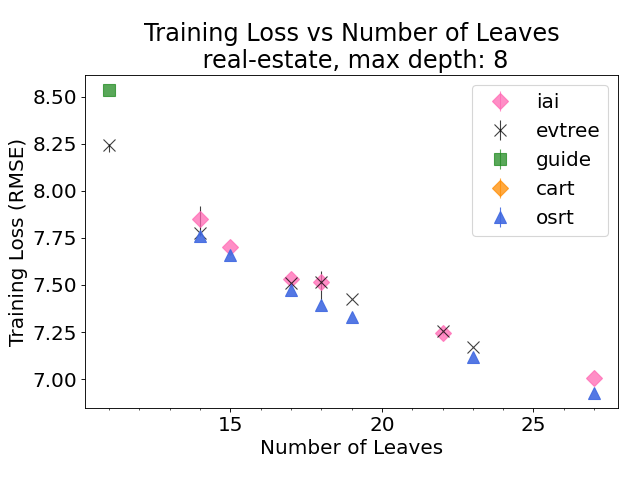}
    \includegraphics[width=0.4\textwidth]{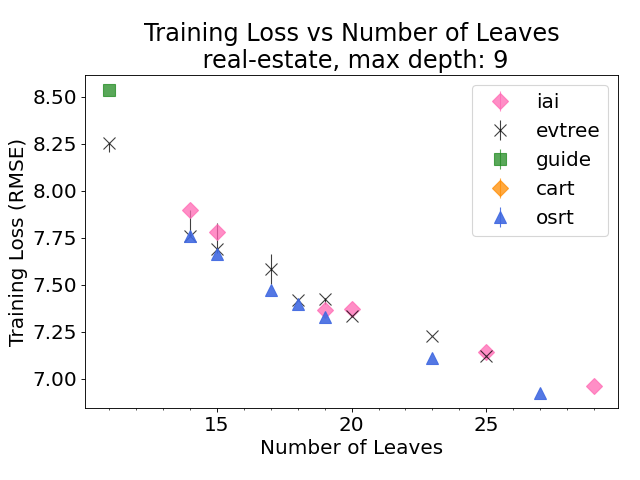}
    \caption{Training loss achieved by IAI, Evtree, GUIDE, CART and OSRT as a function of number of leaves on dataset: real-estate.}
    \label{fig:lvs:real-estate}

\end{figure*}

\begin{figure*}[htbp]
    \centering
    \includegraphics[width=0.4\textwidth]{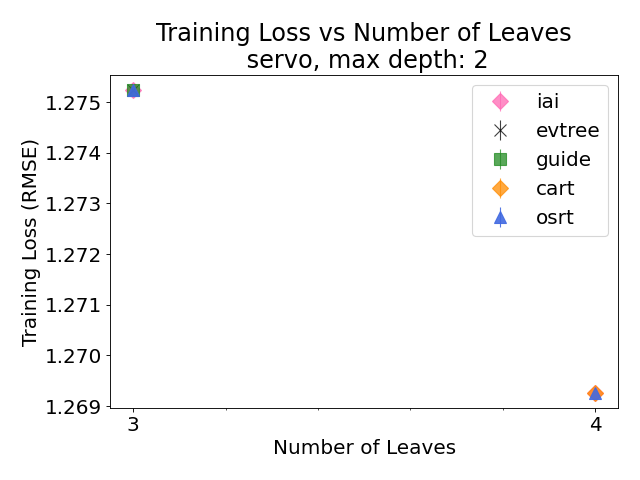}
    \includegraphics[width=0.4\textwidth]{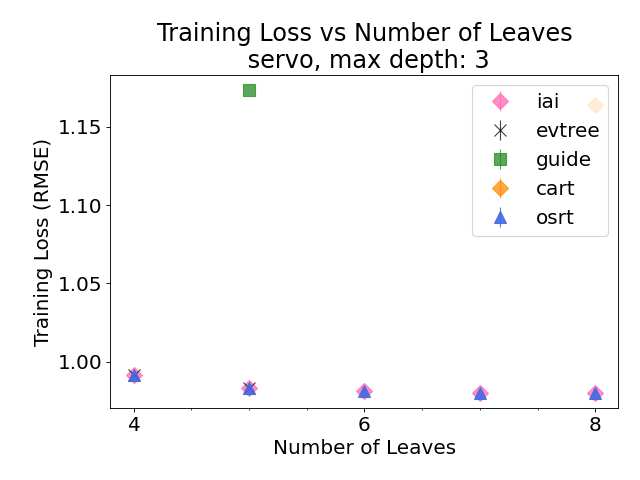}
    \includegraphics[width=0.4\textwidth]{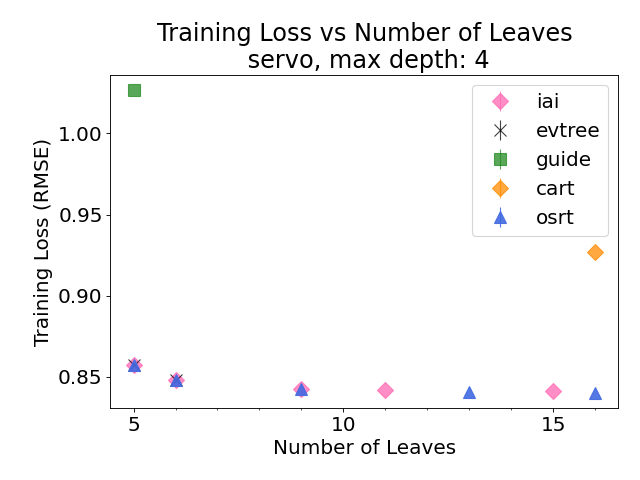}
    \includegraphics[width=0.4\textwidth]{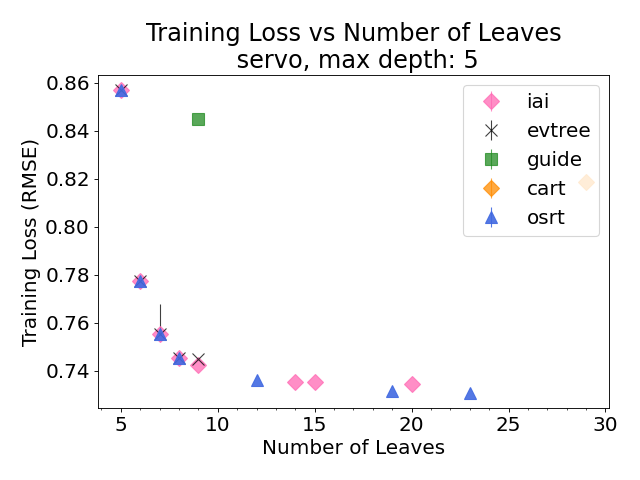}
    \includegraphics[width=0.4\textwidth]{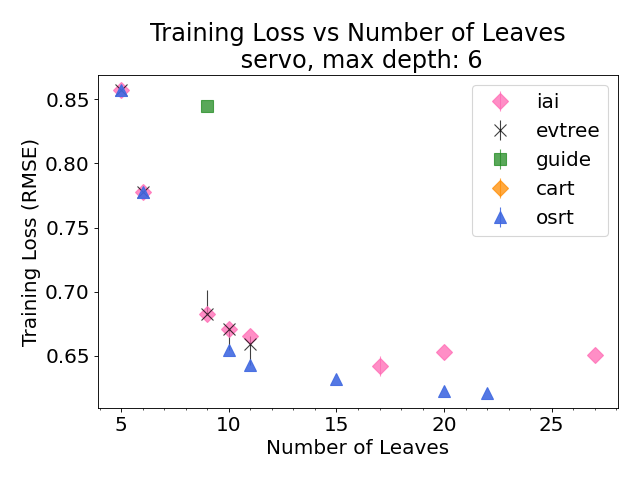}
    \includegraphics[width=0.4\textwidth]{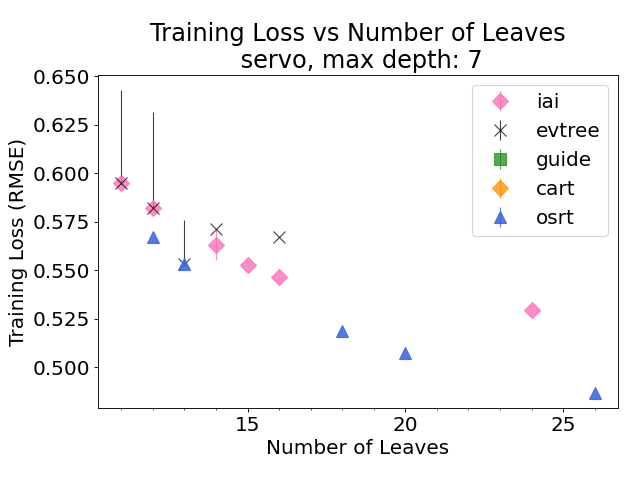}
    \includegraphics[width=0.4\textwidth]{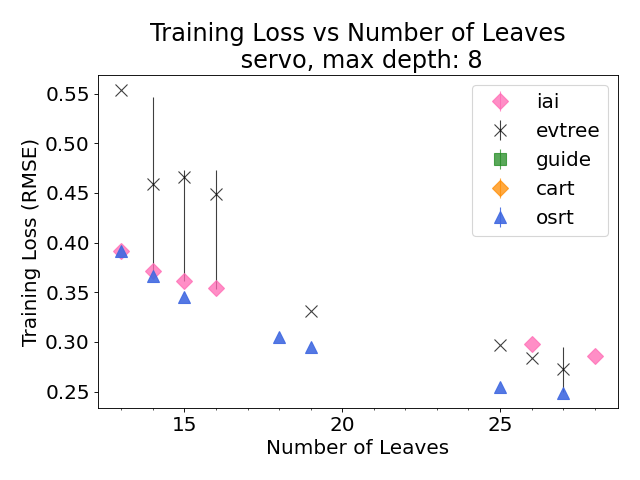}
    \includegraphics[width=0.4\textwidth]{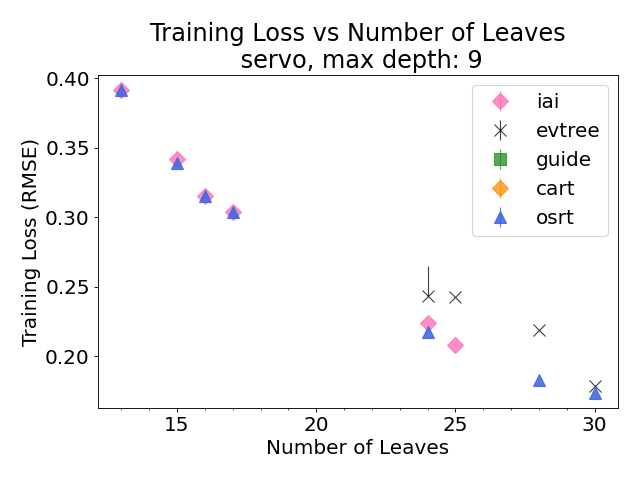}
    \caption{Training loss achieved by IAI, Evtree, GUIDE, CART and OSRT as a function of number of leaves on dataset: servo.}
    \label{fig:lvs:servo}
\end{figure*}

\begin{figure*}[htbp]
    \centering
    \includegraphics[width=0.4\textwidth]{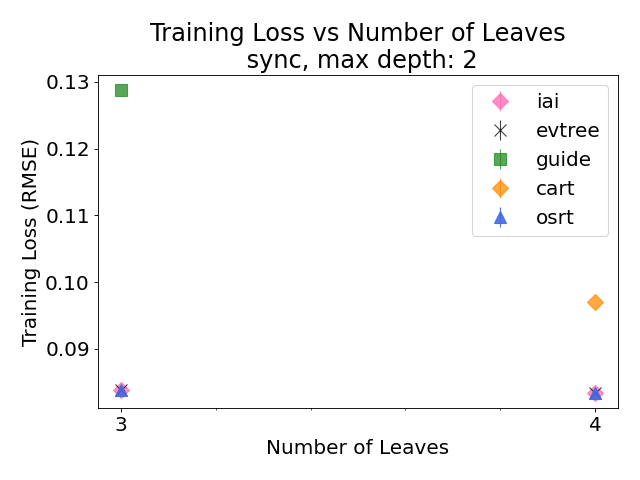}
    \includegraphics[width=0.4\textwidth]{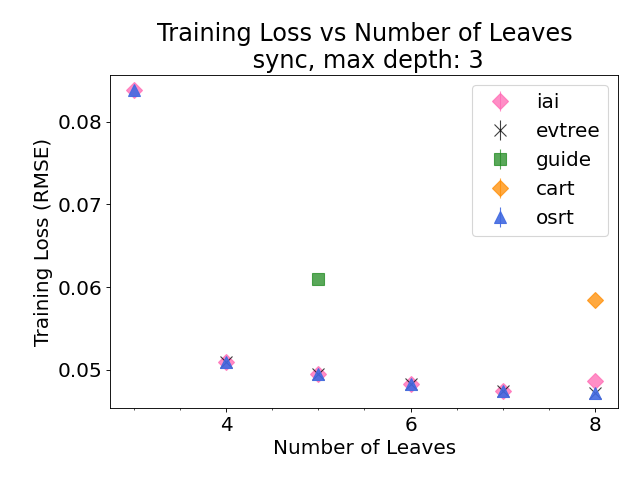}
    \includegraphics[width=0.4\textwidth]{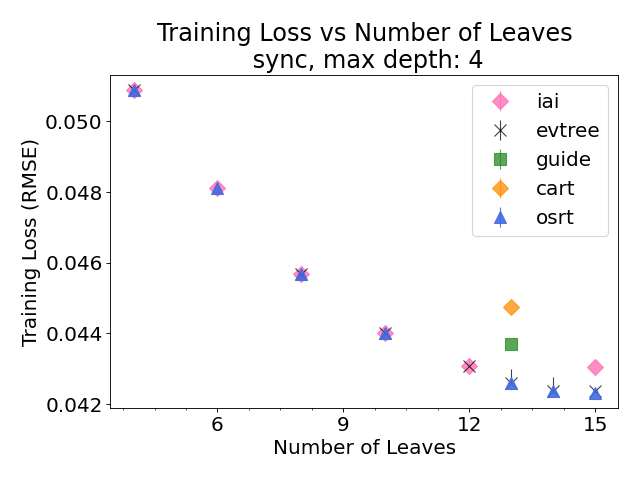}
    \includegraphics[width=0.4\textwidth]{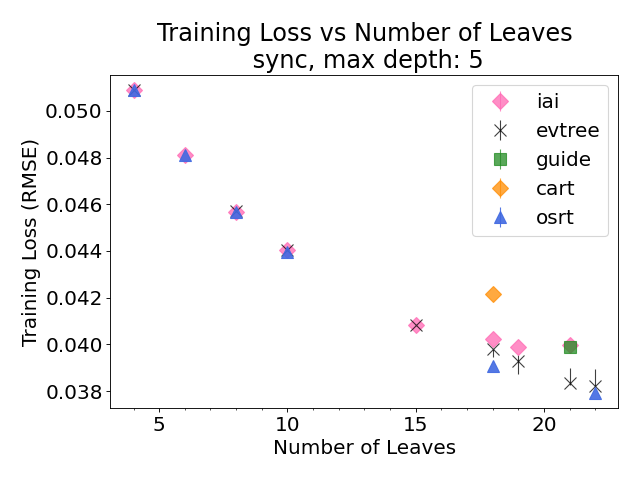}
    \includegraphics[width=0.4\textwidth]{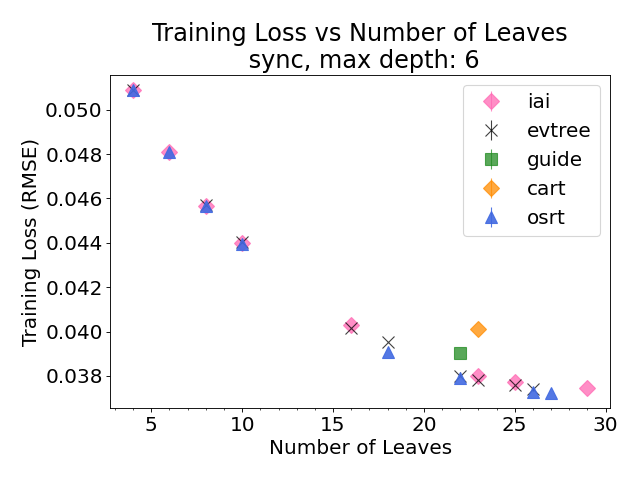}
    \includegraphics[width=0.4\textwidth]{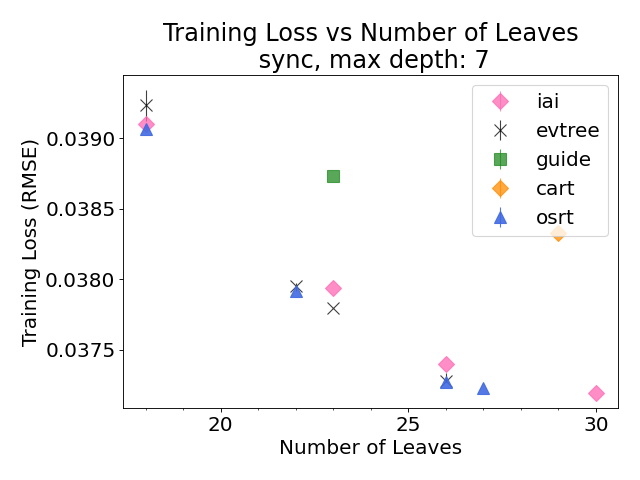}
    \includegraphics[width=0.4\textwidth]{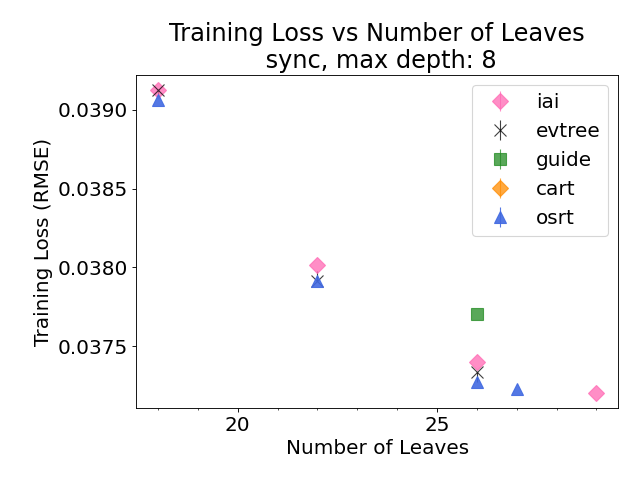}
    \includegraphics[width=0.4\textwidth]{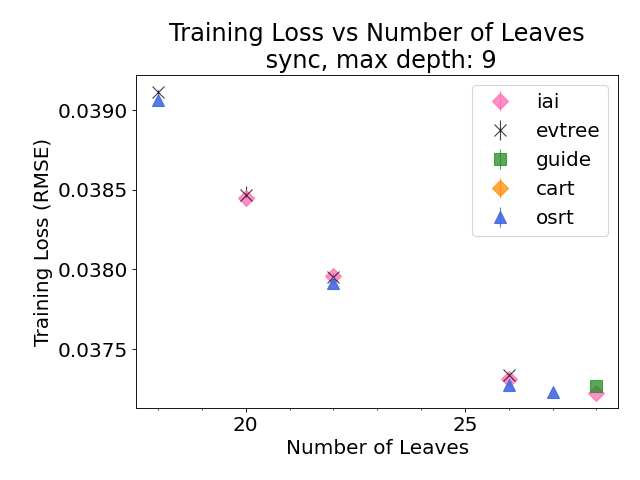}
    \caption{Training loss achieved by IAI, Evtree, GUIDE, CART and OSRT as a function of number of leaves on dataset: sync.}
    \label{fig:lvs:sync}
\end{figure*}

\begin{figure*}[htbp]
    \centering
    \includegraphics[width=0.4\textwidth]{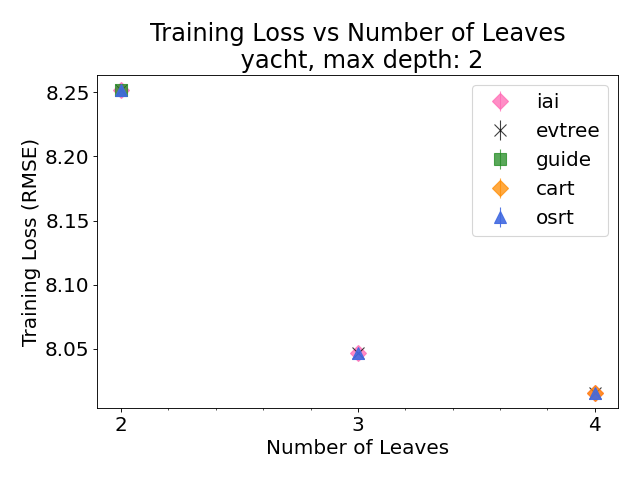}
    \includegraphics[width=0.4\textwidth]{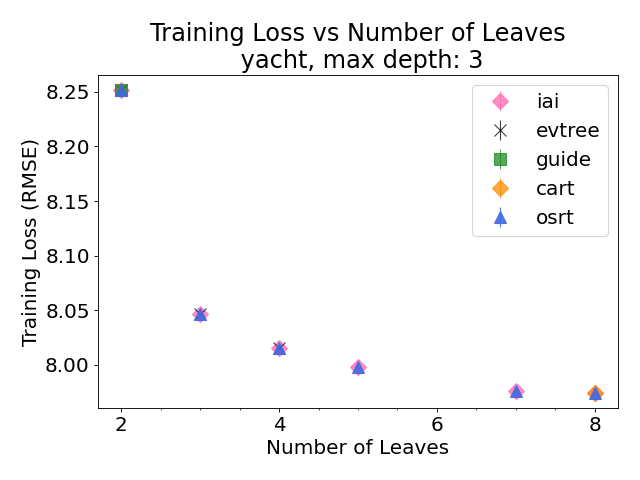}
    \includegraphics[width=0.4\textwidth]{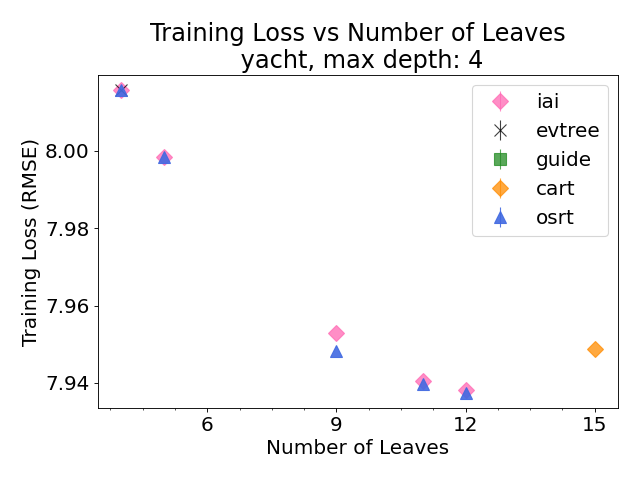}
    \includegraphics[width=0.4\textwidth]{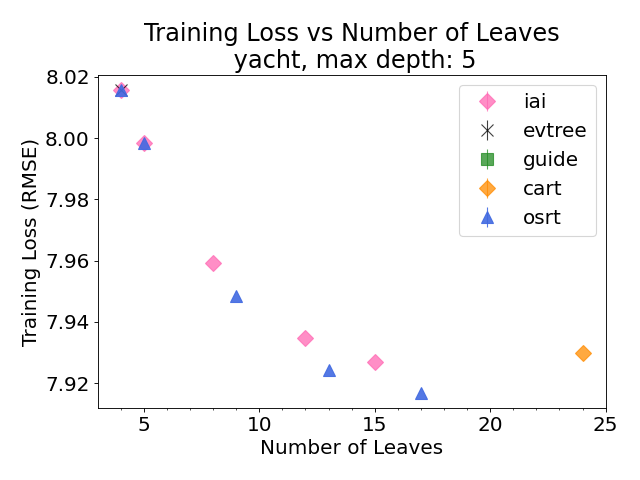}
    \includegraphics[width=0.4\textwidth]{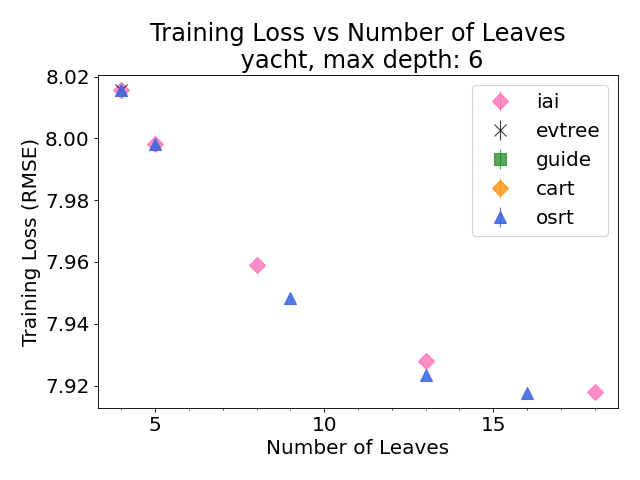}
    \includegraphics[width=0.4\textwidth]{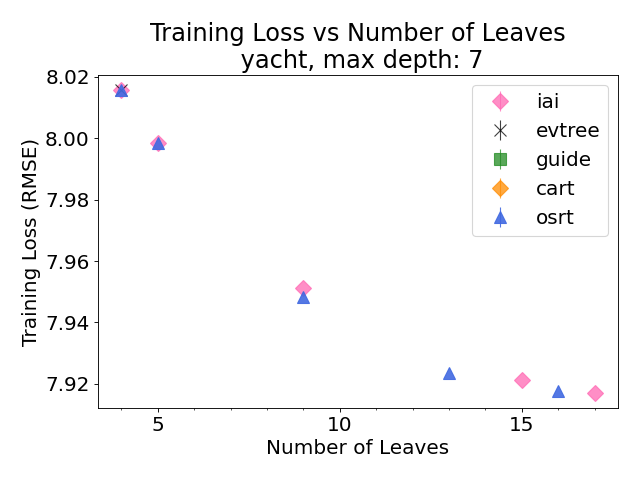}
    \includegraphics[width=0.4\textwidth]{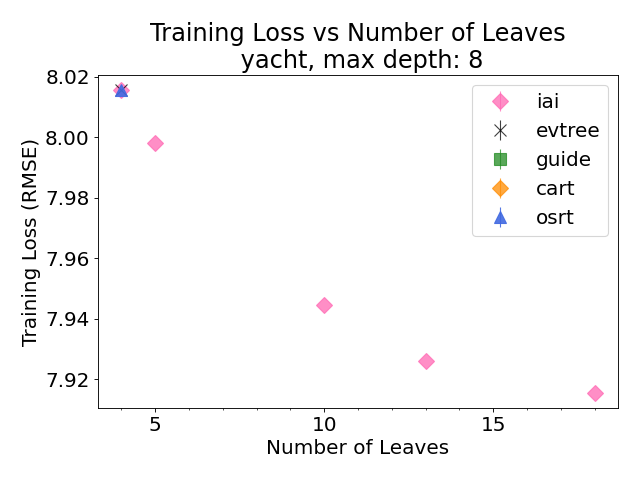}
    \includegraphics[width=0.4\textwidth]{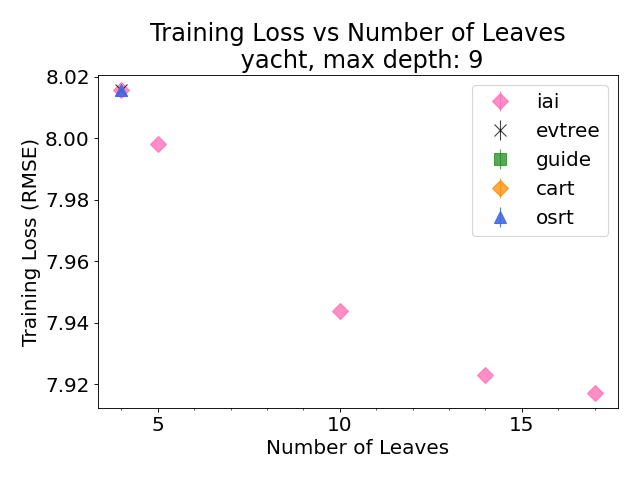}
    \caption{Training loss achieved by IAI, Evtree, GUIDE, CART and OSRT as a function of number of leaves on dataset: yacht.}
    \label{fig:lvs:yacht}
\end{figure*}

\begin{figure*}[htbp]
    \centering
    \includegraphics[width=0.4\textwidth]{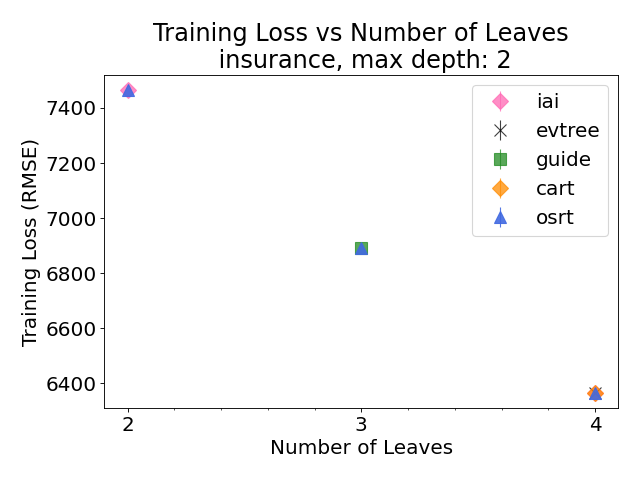}
    \includegraphics[width=0.4\textwidth]{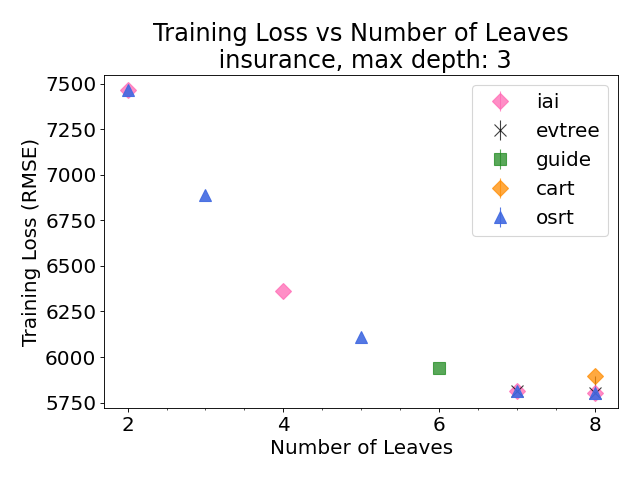}
    \includegraphics[width=0.4\textwidth]{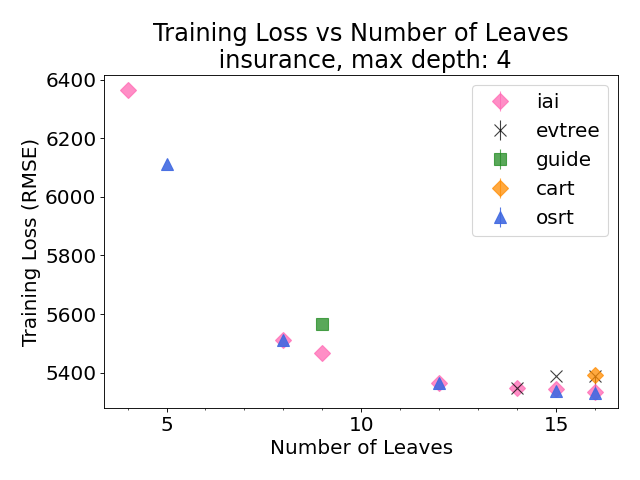}
    \includegraphics[width=0.4\textwidth]{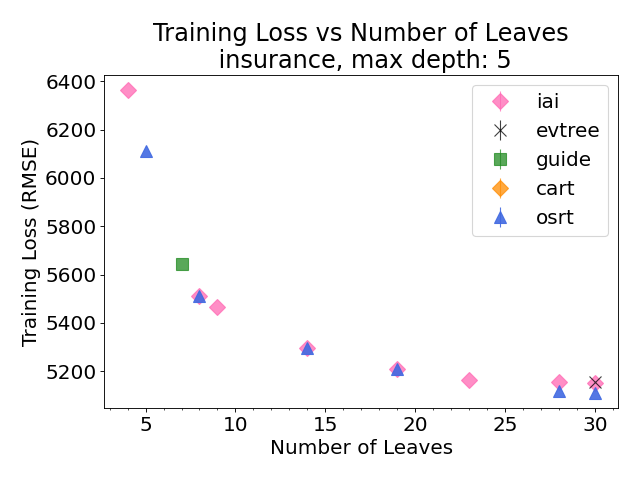}
    \includegraphics[width=0.4\textwidth]{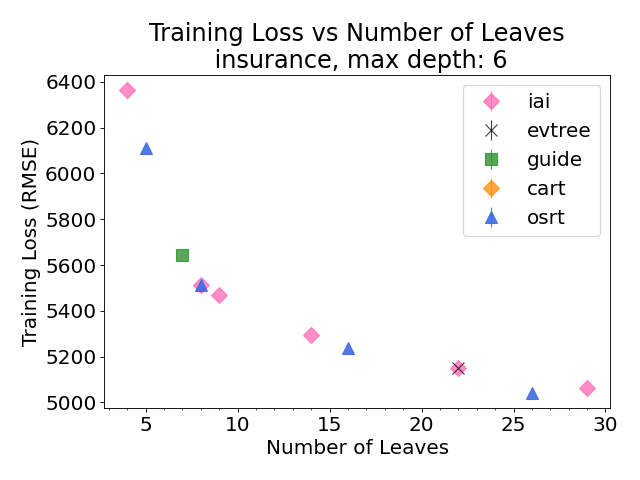}
    \includegraphics[width=0.4\textwidth]{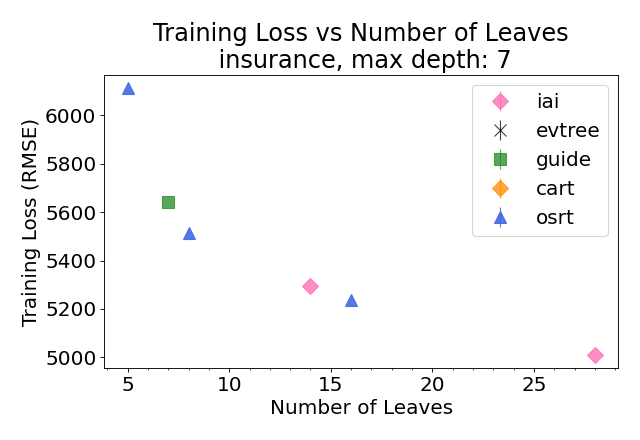}
    \includegraphics[width=0.4\textwidth]{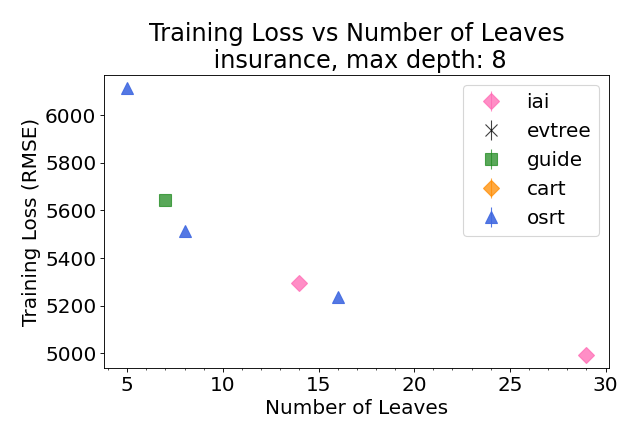}
    \includegraphics[width=0.4\textwidth]{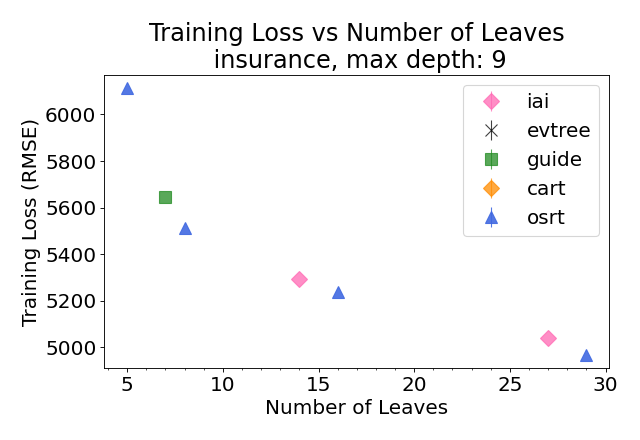}
    \caption{Training loss achieved by IAI, Evtree, GUIDE, CART and OSRT as a function of number of leaves on dataset: insurance.}
    \label{fig:lvs:insurance}
\end{figure*}

\begin{figure*}[htbp]
    \centering
    \includegraphics[width=0.45\textwidth]{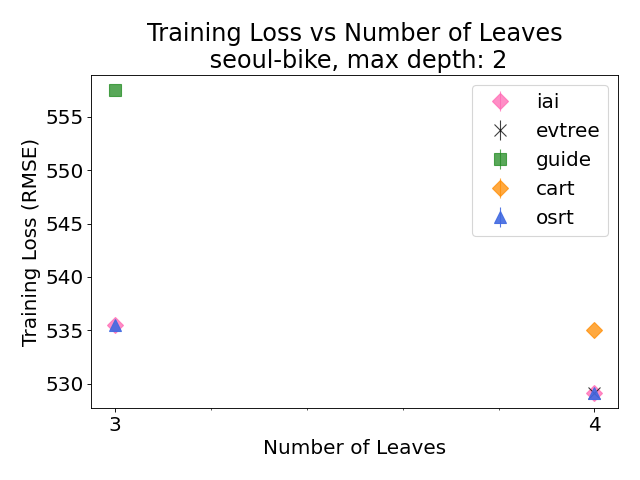}
    \includegraphics[width=0.45\textwidth]{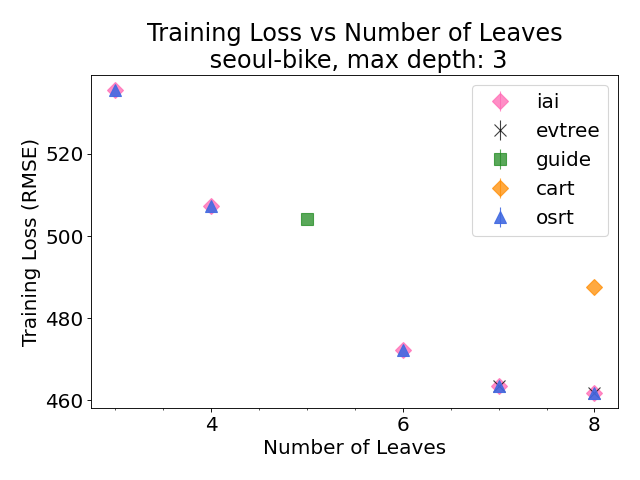}
    \includegraphics[width=0.45\textwidth]{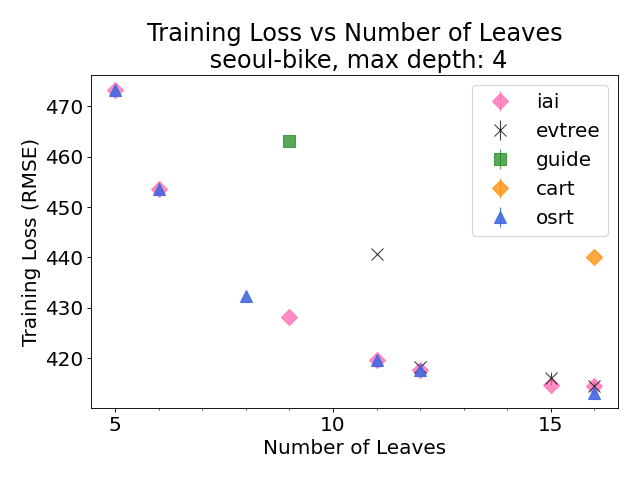}
    \includegraphics[width=0.45\textwidth]{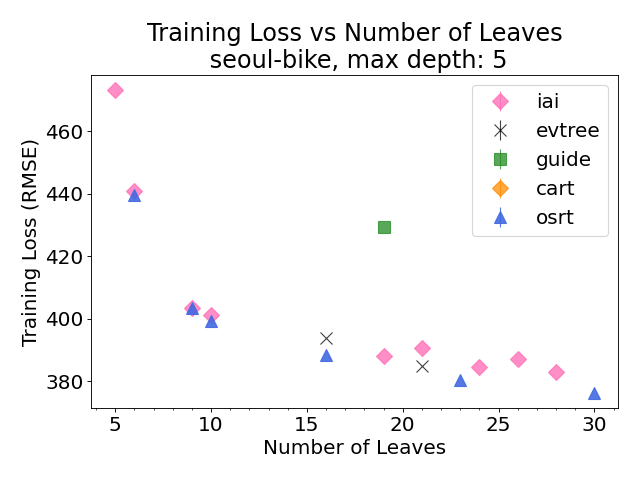}
    \caption{Training loss achieved by IAI, Evtree, GUIDE, CART and OSRT as a function of number of leaves on dataset: seoul-bike, depths 2 -- 5. Depths 6 -- 9 are omitted since Evtree and OSRT timed out.}
    \label{fig:lvs:seoul-bike}
\end{figure*}

%% file: exp_fig_code/variance.tex
\begin{figure*}[htbp]
    \centering
    \includegraphics[width=0.45\textwidth]{figures/binary_settings/variance/real-estate_depth_4.png}
    \includegraphics[width=0.45\textwidth]{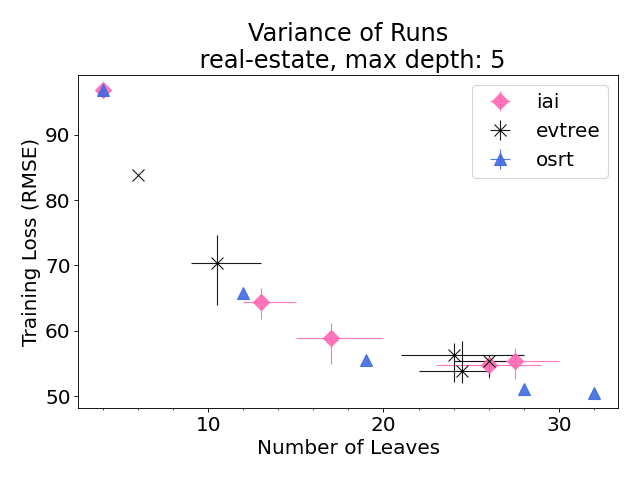}
    \includegraphics[width=0.45\textwidth]{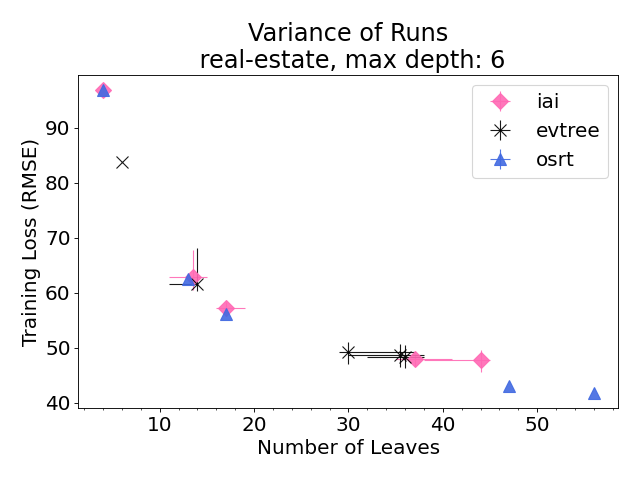}
    \includegraphics[width=0.45\textwidth]{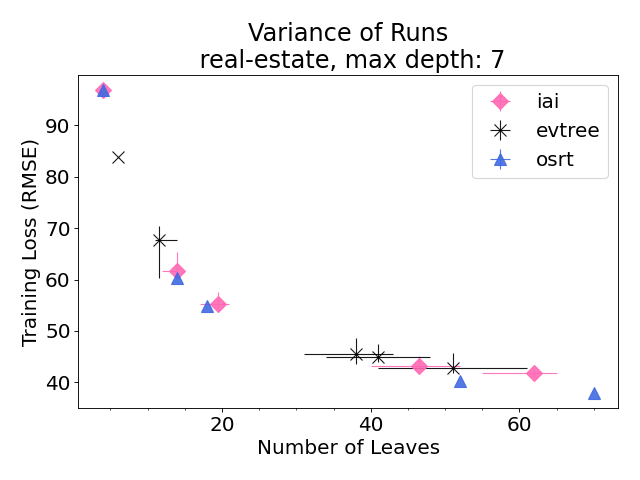}
    \includegraphics[width=0.45\textwidth]{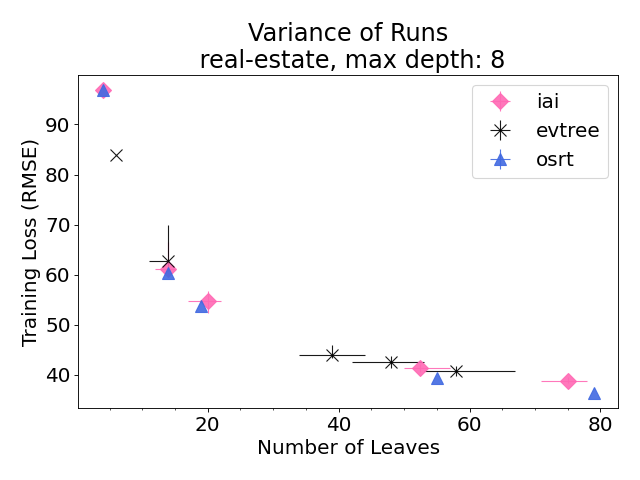}
    \includegraphics[width=0.45\textwidth]{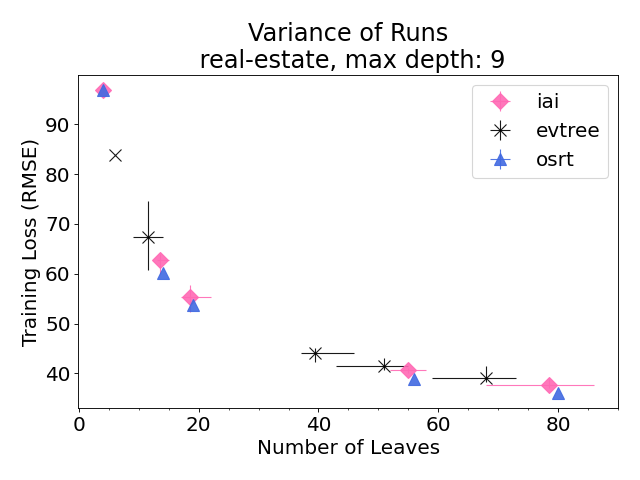}
    \caption{Variance (horizontal and vertical lines) of trees generated by IAI, Evtree and OSRT on dataset: \textit{real-estate}, when different random seeds were used. Note that in cases like the max depth 5 plot for approximately 11 leaves in a tree, while evtree's loss lower bound drops below OSRT's blue triangle (which is optimal), it is not because the blue triangle is suboptimal; it is instead because evtree produced a larger size tree to get that low error, which is why there is  horizontal variance in addition to the vertical variance.}
    \label{fig:variance:real-estate}
\end{figure*}

\begin{figure*}[htbp]
    \centering
    \includegraphics[width=0.45\textwidth]{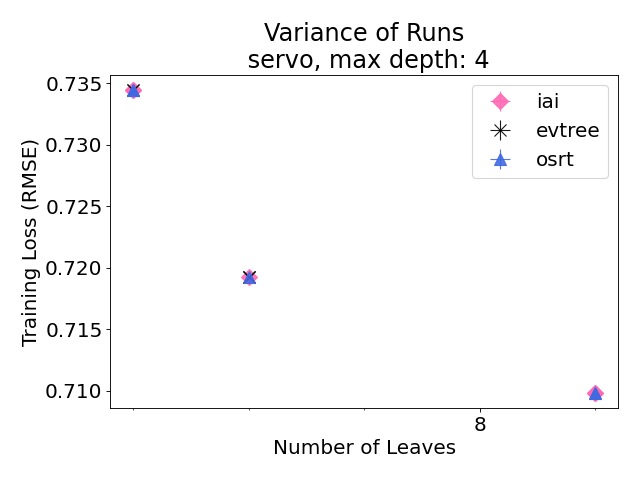}
    \includegraphics[width=0.45\textwidth]{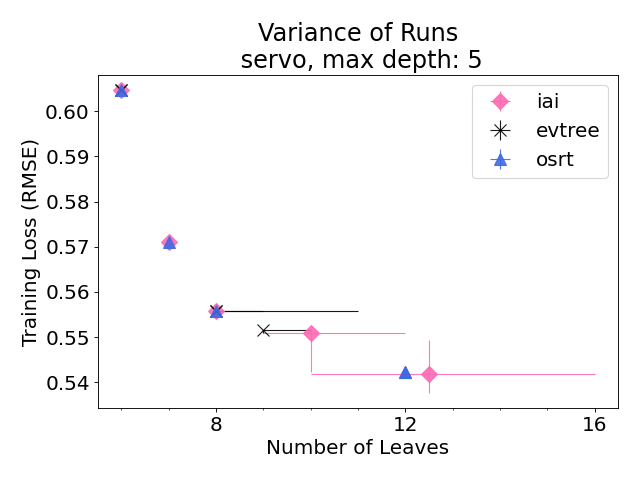}
    \includegraphics[width=0.45\textwidth]{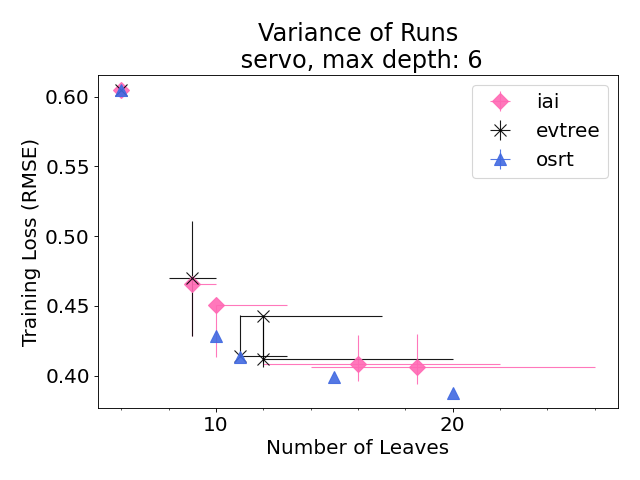}
    \includegraphics[width=0.45\textwidth]{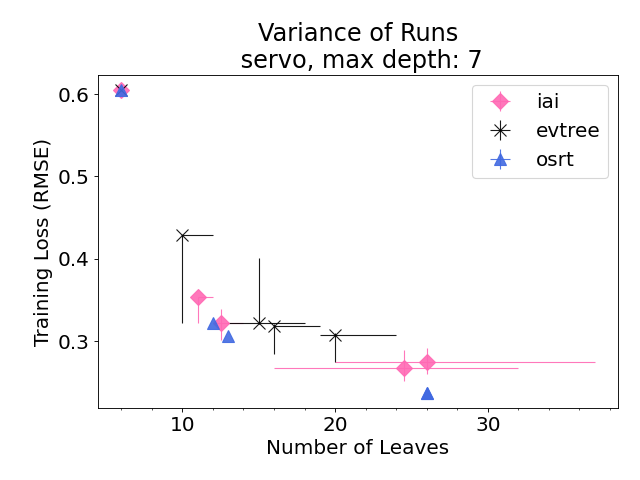}
    \includegraphics[width=0.45\textwidth]{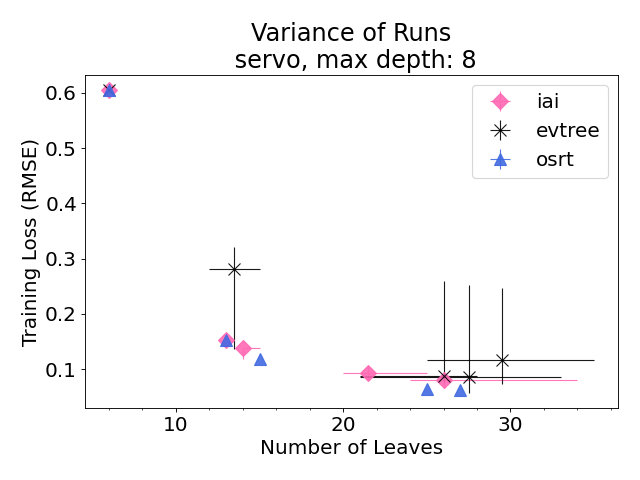}
    \includegraphics[width=0.45\textwidth]{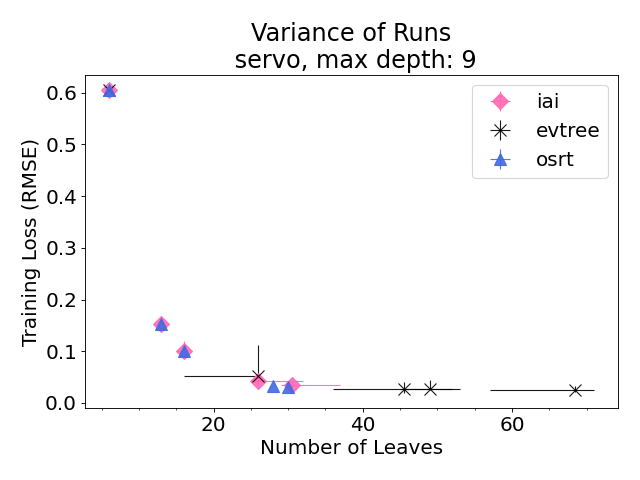}
    \caption{Variance (horizontal and vertical lines) of trees generated by IAI, Evtree and OSRT on dataset: servo, when different random seeds were used. Again, in cases like the max depth 6 figure for 10 leaves, IAI's loss lower bound drops below OSRT's blue triangle, which is optimal. This is because IAI produced a larger model to achieve this low loss, as indicated by the horizontal variance for IAI at 10 leaves.}
    \label{fig:variance:servo}
\end{figure*}

%% file: exp_fig_code/time_vs_sparsity.tex
\begin{figure*}[htbp]
    \centering
    \includegraphics[width=0.41\textwidth]{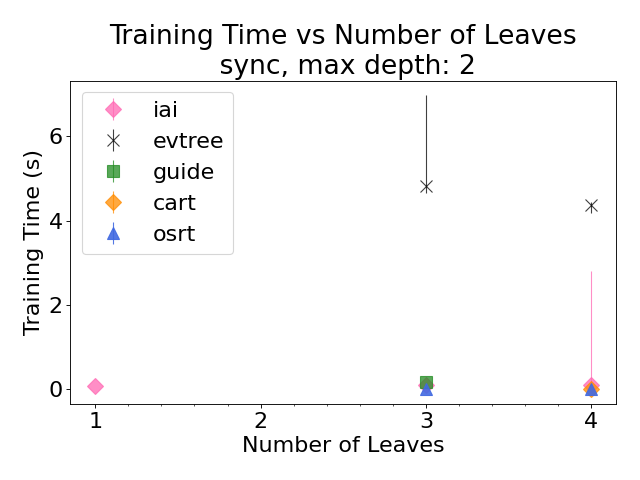}
    \includegraphics[width=0.41\textwidth]{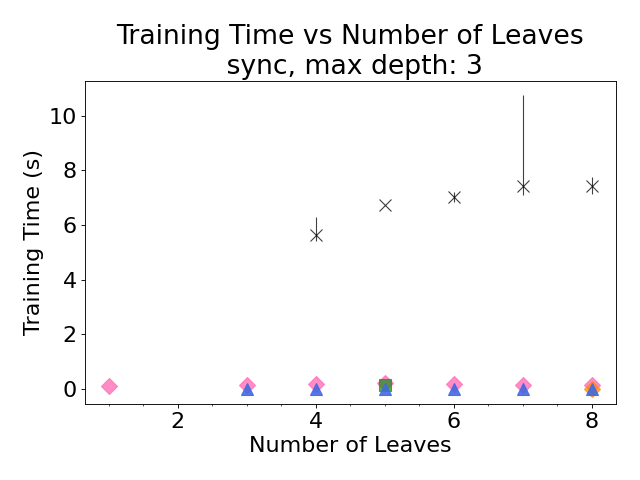}
    \includegraphics[width=0.41\textwidth]{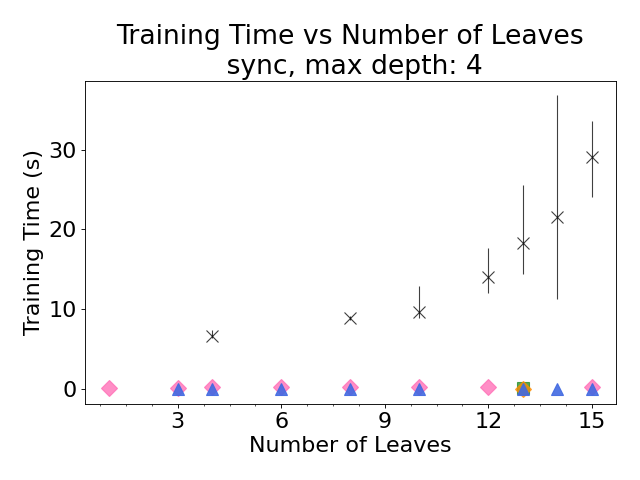}
    \includegraphics[width=0.41\textwidth]{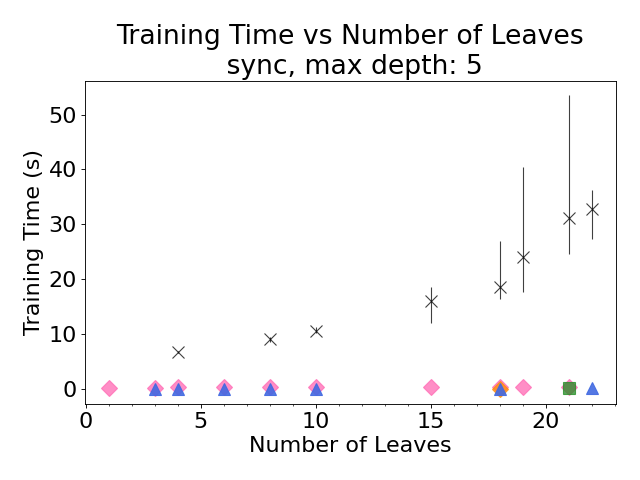}
    \includegraphics[width=0.41\textwidth]{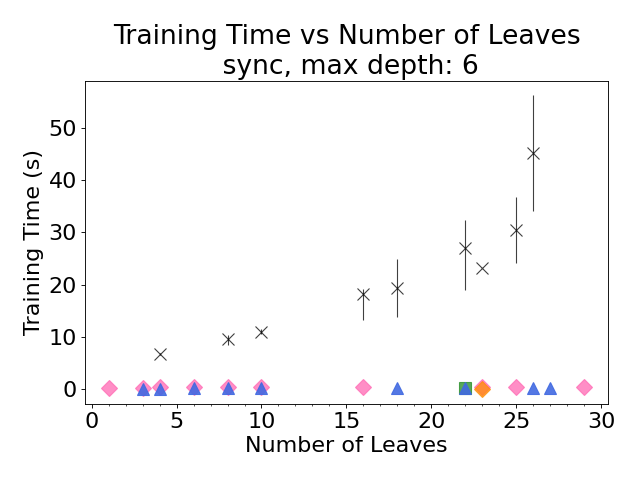}
    \includegraphics[width=0.41\textwidth]{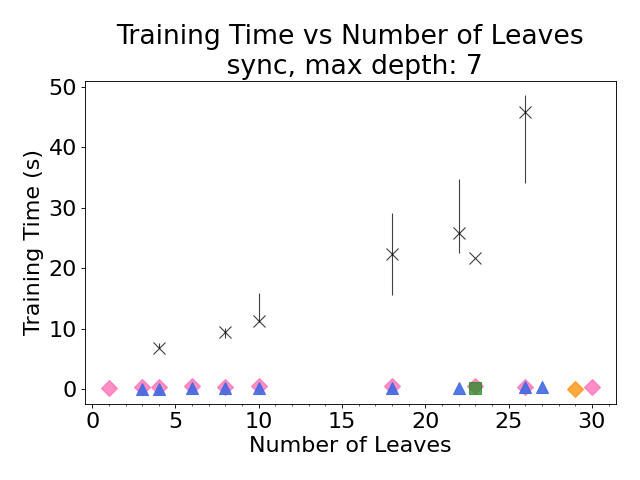}
    \includegraphics[width=0.41\textwidth]{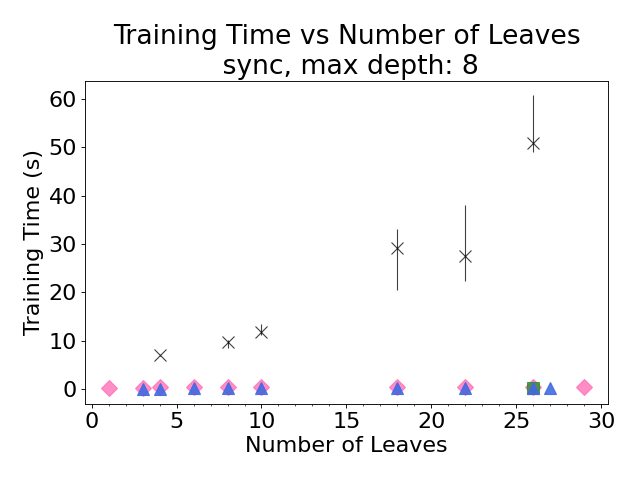}
    \includegraphics[width=0.41\textwidth]{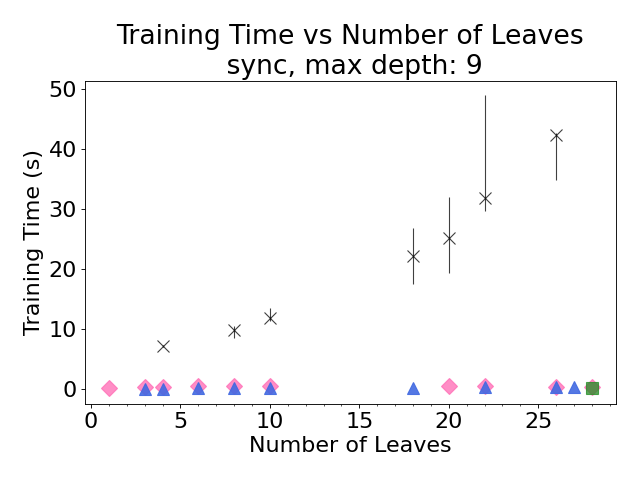}
    
    \caption{Training time of OSRT, IAI, Evtree, CART, GUIDE, as a function of number of leaves, on dataset: sync}
    \label{fig:tvs:sync}
\end{figure*}

\begin{figure*}[htbp]
    \centering
    \includegraphics[width=0.41\textwidth]{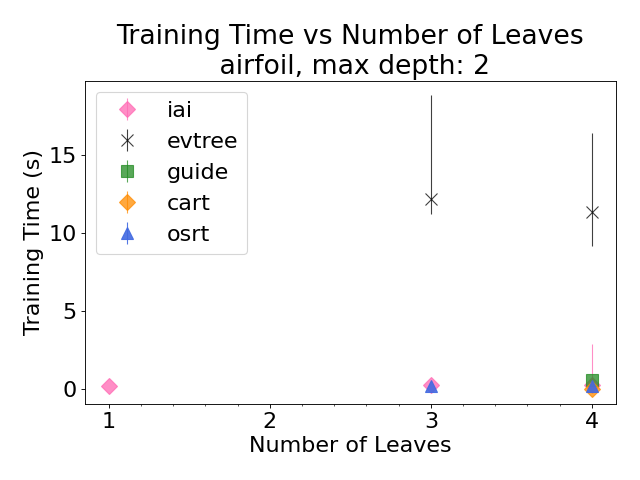}
    \includegraphics[width=0.41\textwidth]{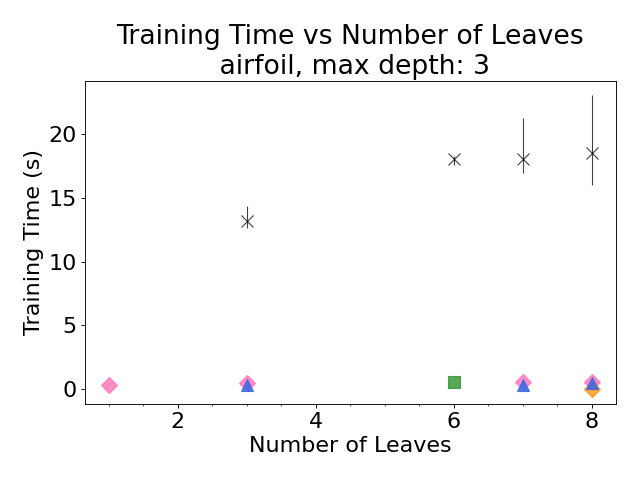}
    \includegraphics[width=0.41\textwidth]{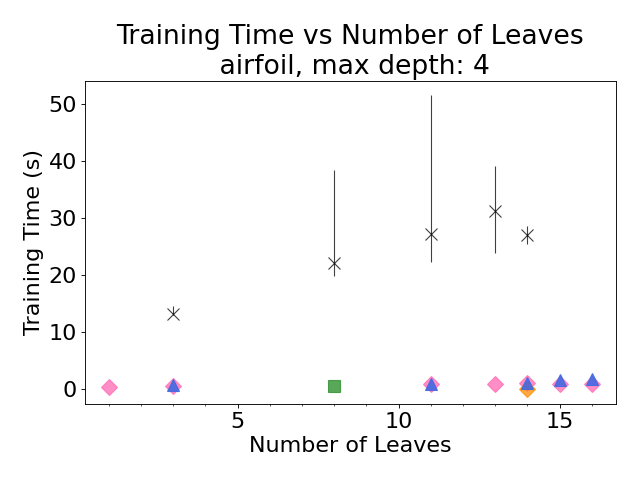}
    \includegraphics[width=0.41\textwidth]{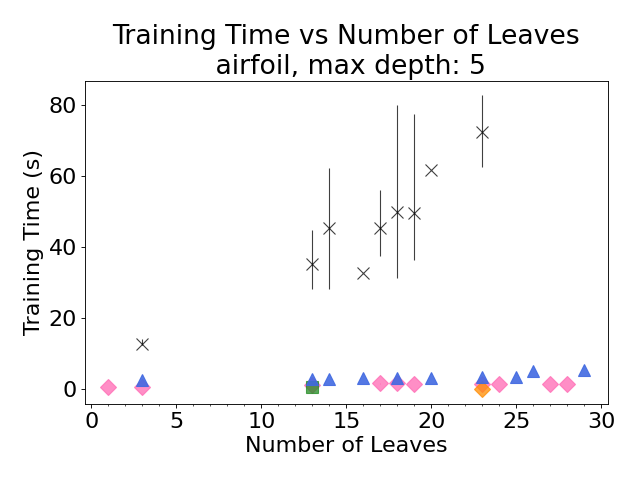}
    \includegraphics[width=0.41\textwidth]{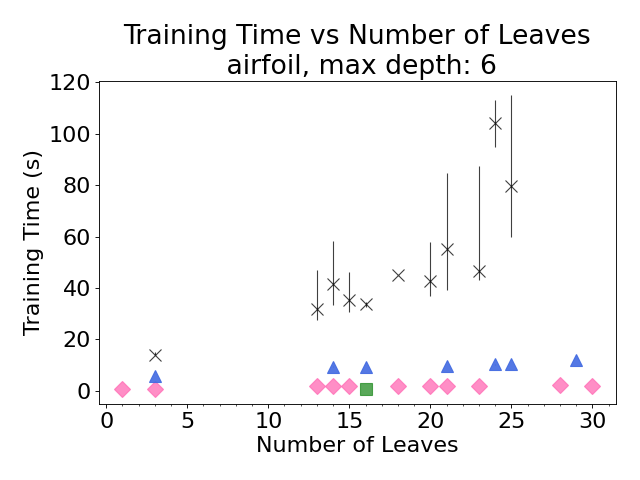}
    \includegraphics[width=0.41\textwidth]{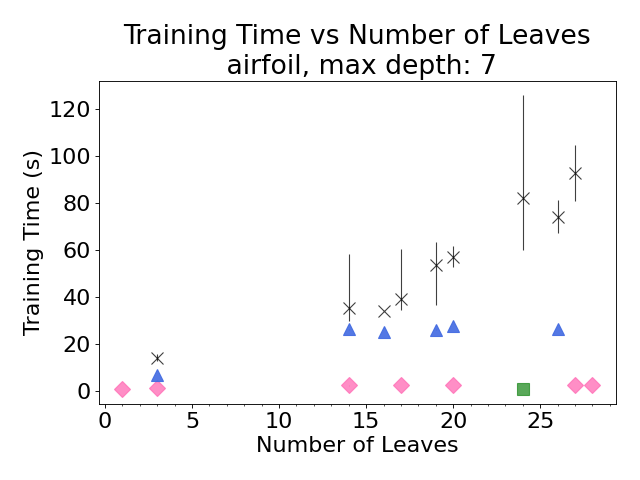}
    \includegraphics[width=0.41\textwidth]{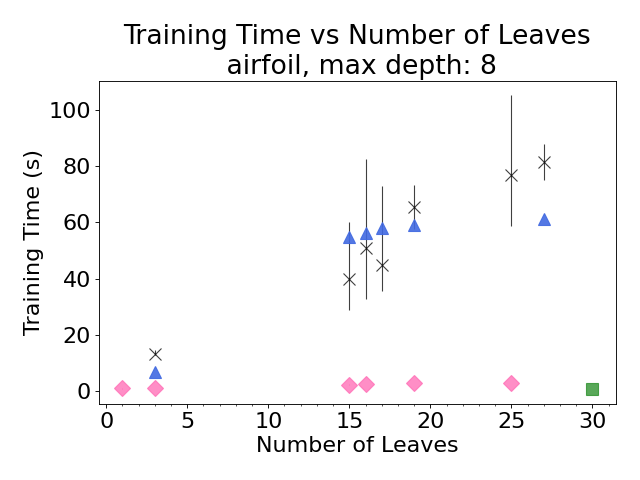}
    \includegraphics[width=0.41\textwidth]{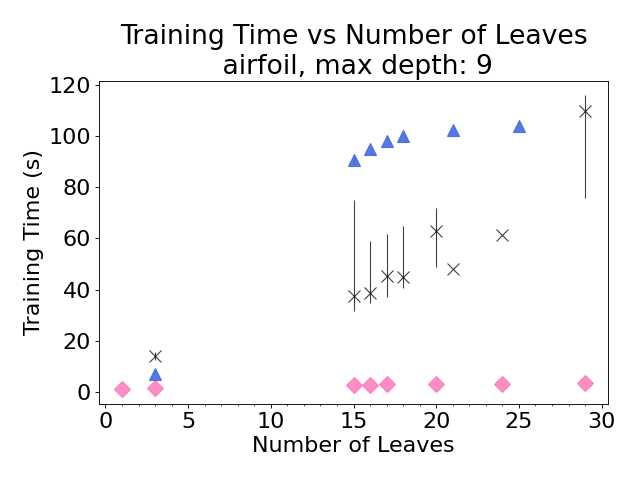}
    \caption{Training time of OSRT, IAI, Evtree, CART, GUIDE as a function of number of leaves on dataset: airfoil}
    \label{fig:tvs:airfoil}
    
\end{figure*}

\begin{figure*}[htbp]
    \centering
    \includegraphics[width=0.41\textwidth]{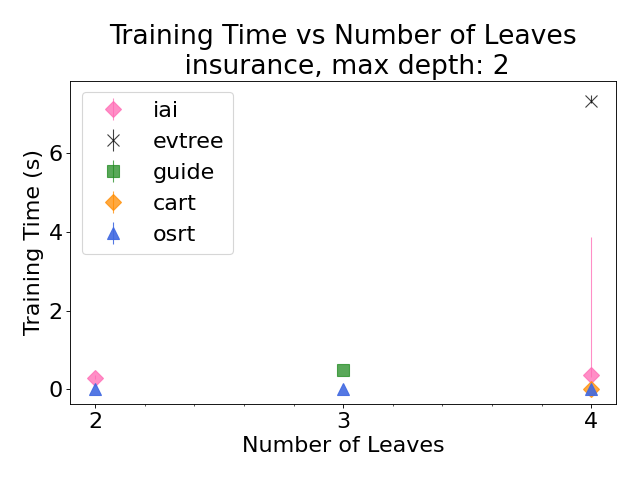}
    \includegraphics[width=0.41\textwidth]{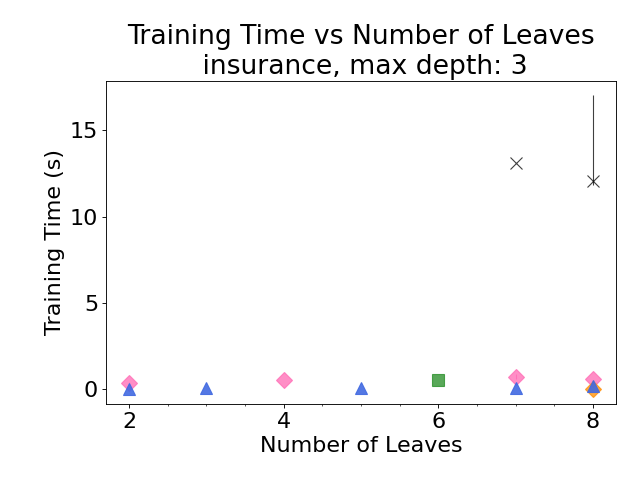}
    \includegraphics[width=0.41\textwidth]{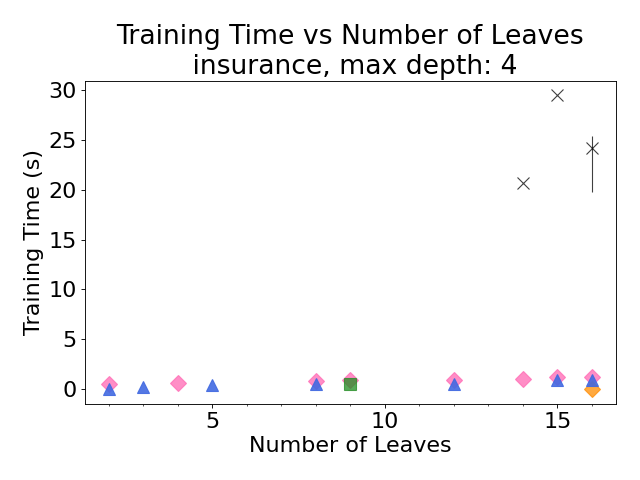}
    \includegraphics[width=0.41\textwidth]{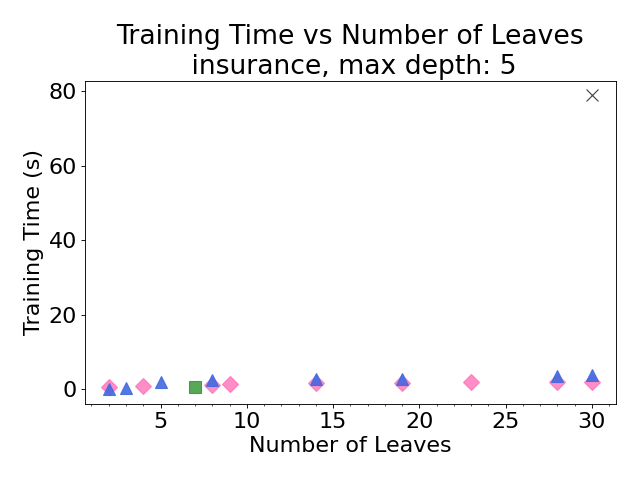}
    \includegraphics[width=0.41\textwidth]{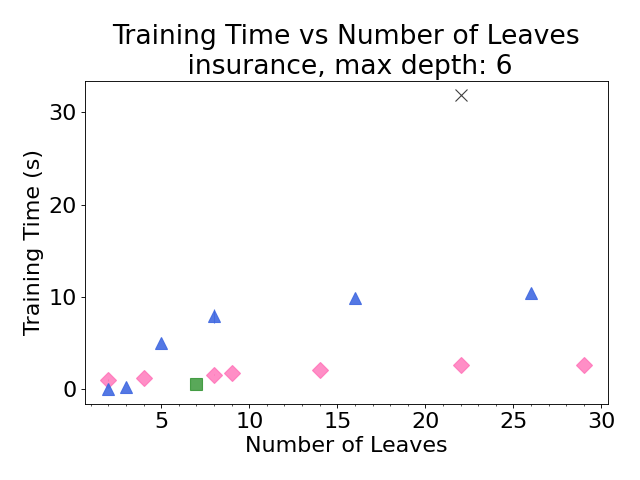}
    \includegraphics[width=0.41\textwidth]{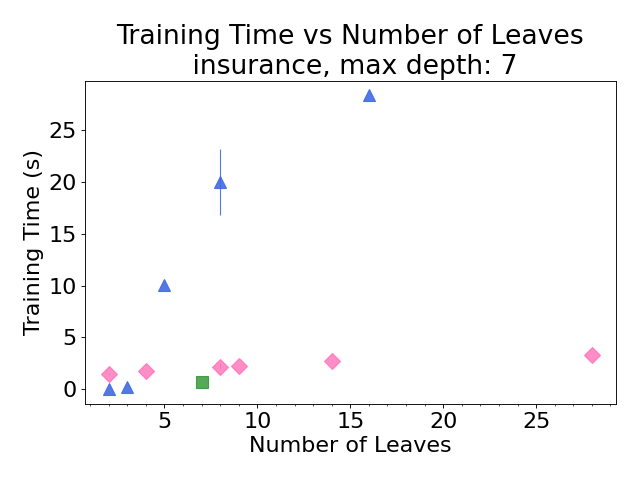}
    \includegraphics[width=0.41\textwidth]{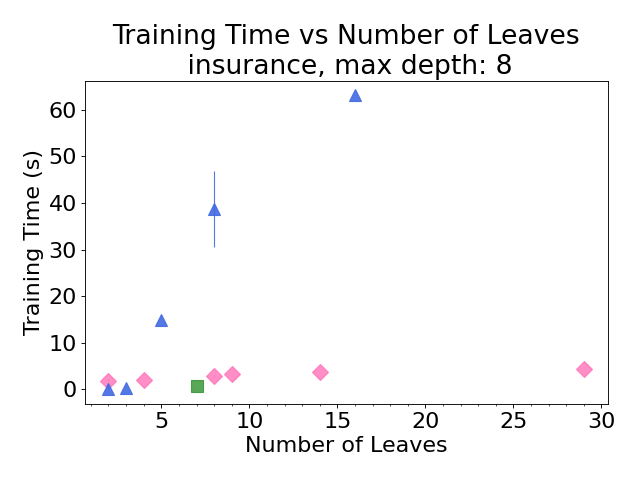}
    \includegraphics[width=0.41\textwidth]{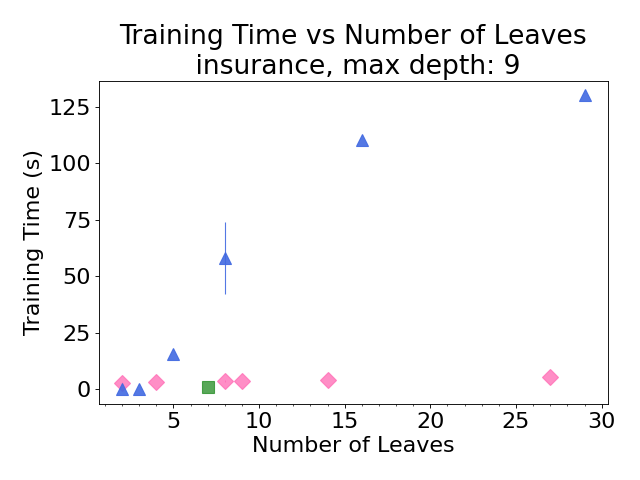}
    
    \caption{Training time of OSRT, IAI, Evtree, CART, GUIDE, as a function of number of leaves, on dataset: insurance}
    \label{fig:tvs:insurance}
\end{figure*}

\begin{figure*}[htbp]
    \centering
    \includegraphics[width=0.41\textwidth]{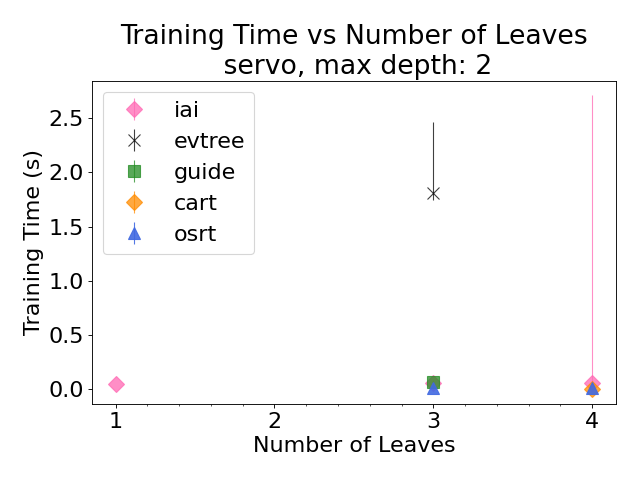}
    \includegraphics[width=0.41\textwidth]{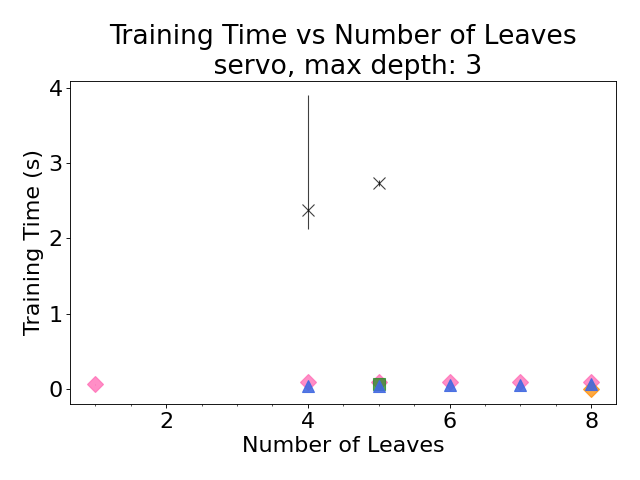}
    \includegraphics[width=0.41\textwidth]{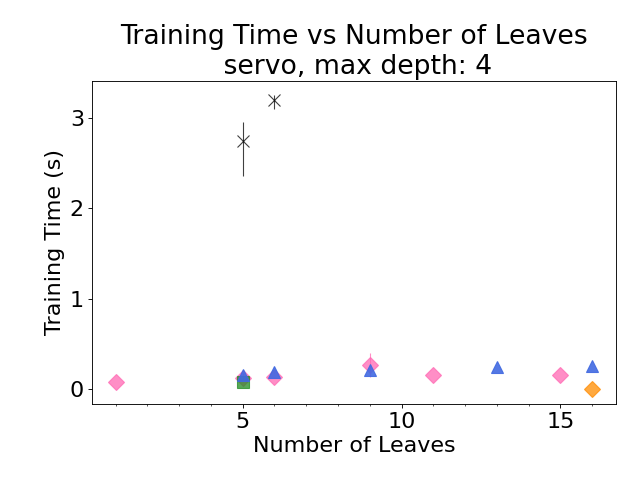}
    \includegraphics[width=0.41\textwidth]{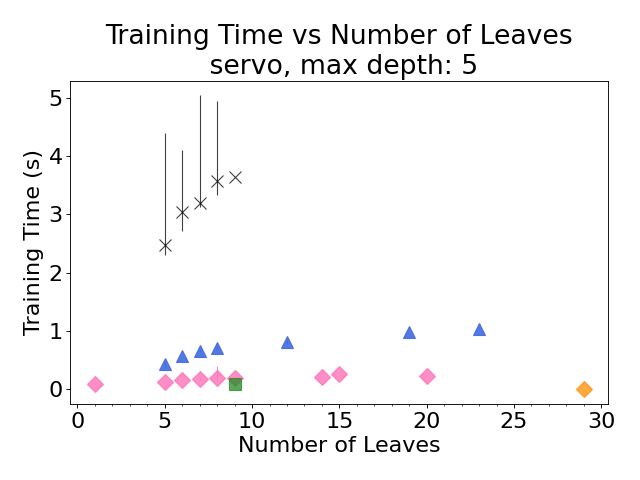}
    \includegraphics[width=0.41\textwidth]{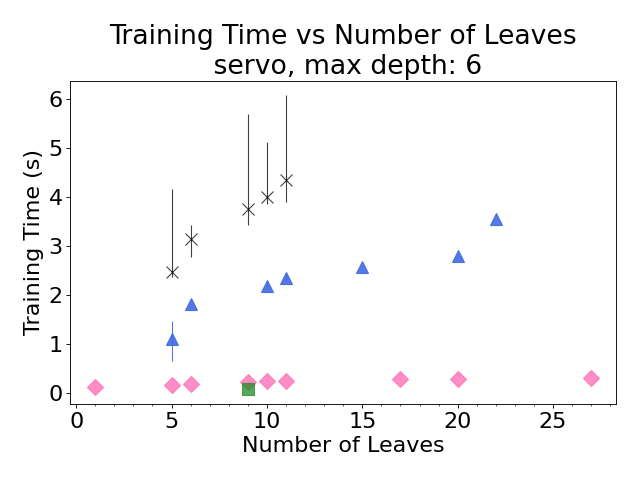}
    \includegraphics[width=0.41\textwidth]{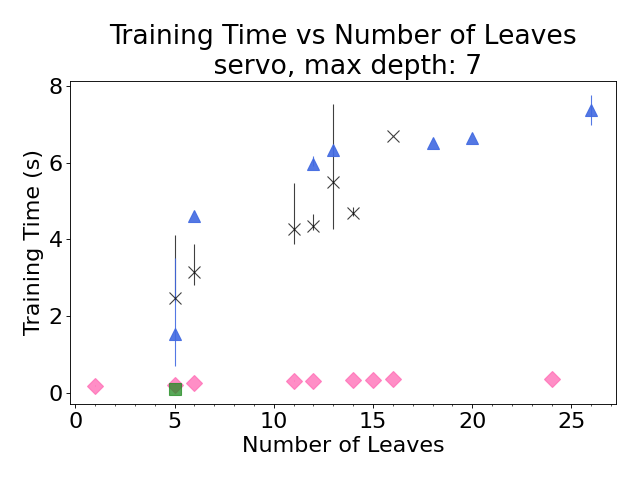}
    \includegraphics[width=0.41\textwidth]{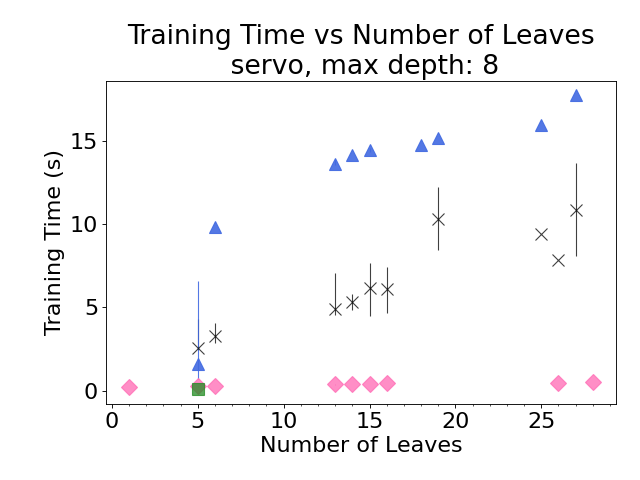}
    \includegraphics[width=0.41\textwidth]{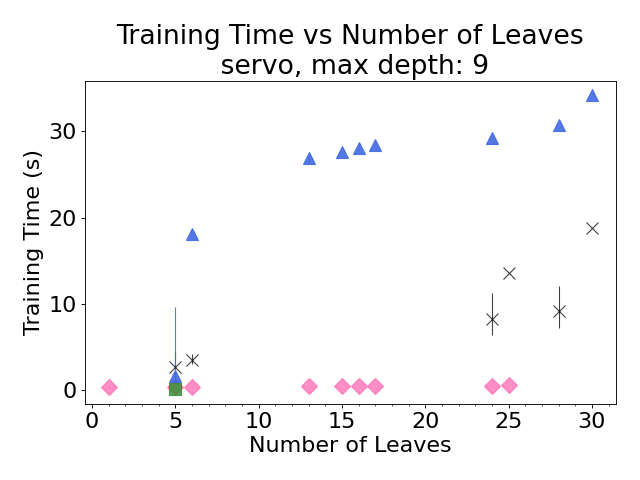}
    
    \caption{Training time of OSRT, IAI, Evtree, as a function of number of leaves, on dataset: servo}
    \label{fig:tvs:servo}
\end{figure*}

\begin{figure*}[htbp]
    \centering
    \includegraphics[width=0.41\textwidth]{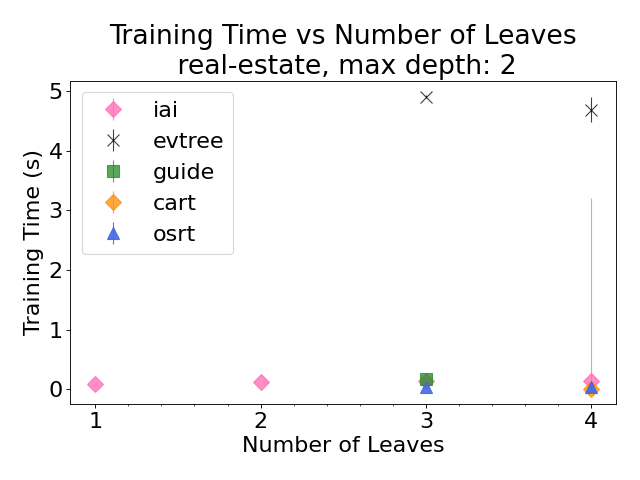}
    \includegraphics[width=0.41\textwidth]{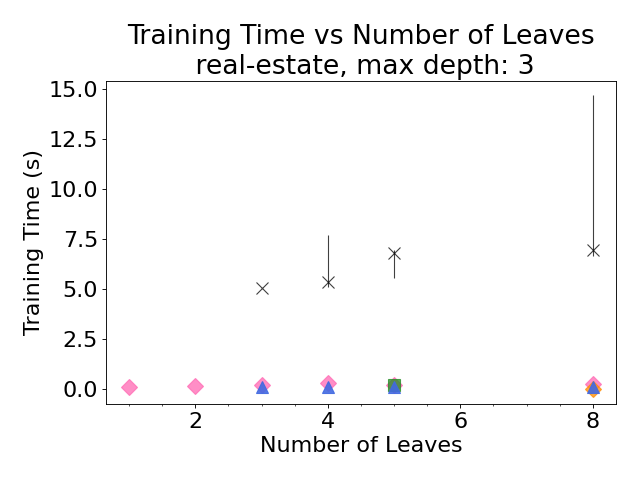}
    \includegraphics[width=0.41\textwidth]{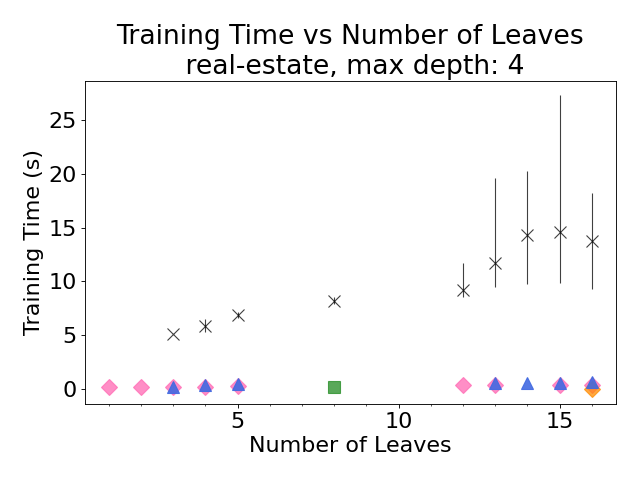}
    \includegraphics[width=0.41\textwidth]{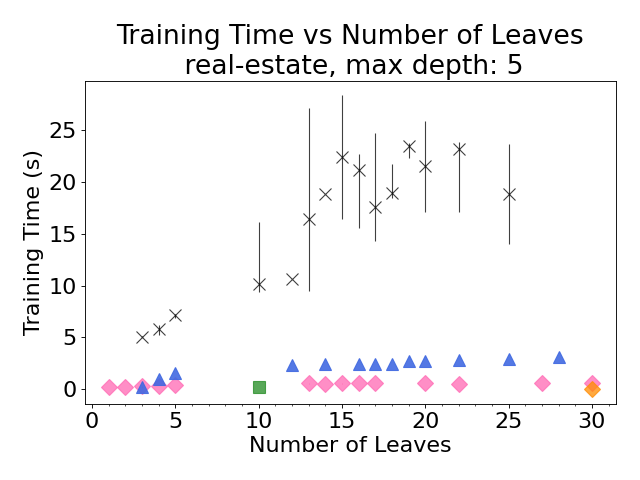}
    \includegraphics[width=0.41\textwidth]{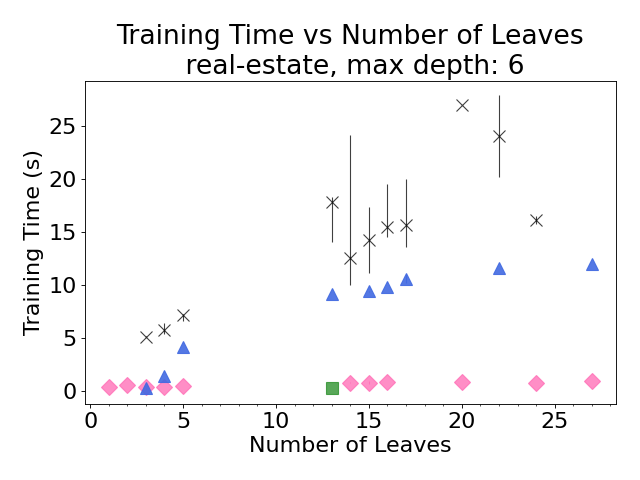}
    \includegraphics[width=0.41\textwidth]{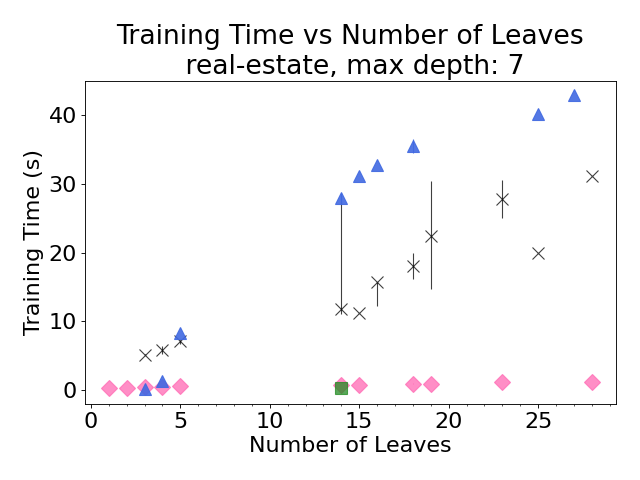}
    \includegraphics[width=0.41\textwidth]{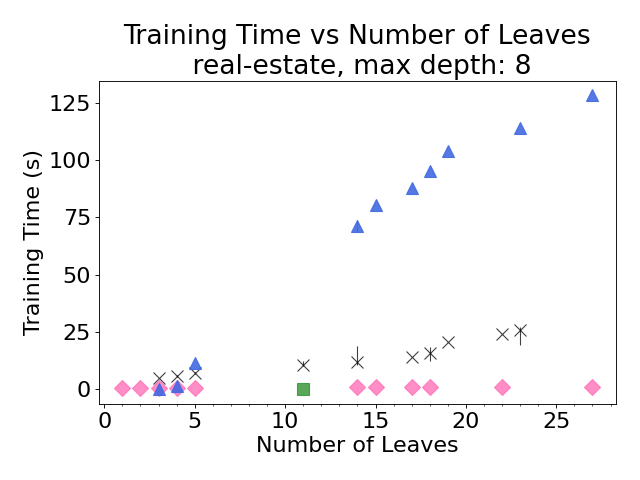}
    \includegraphics[width=0.41\textwidth]{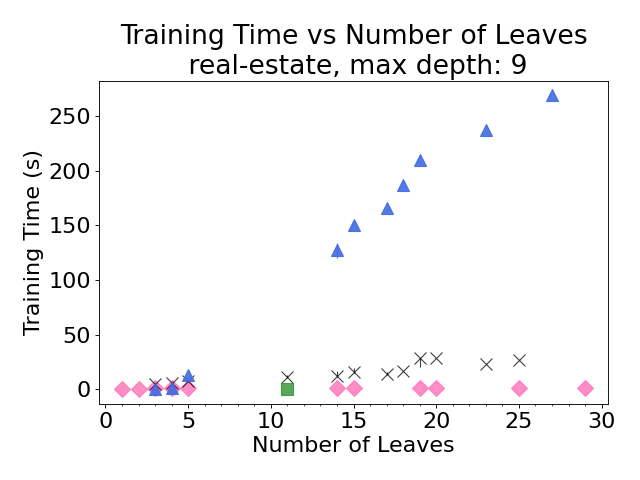}
    \caption{Training time of OSRT, IAI, Evtree, CART, GUIDE, as a function of number of leaves, on dataset: real-estate}
    \label{fig:tvs:real-estate}
\end{figure*}

\begin{figure*}[htbp]
    \centering
    \includegraphics[width=0.41\textwidth]{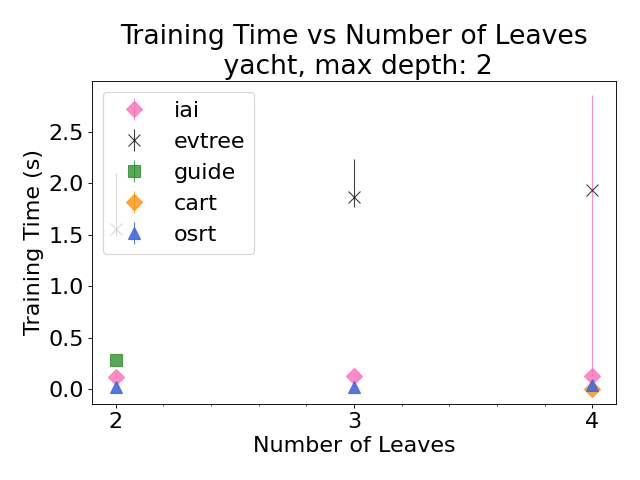}
    \includegraphics[width=0.41\textwidth]{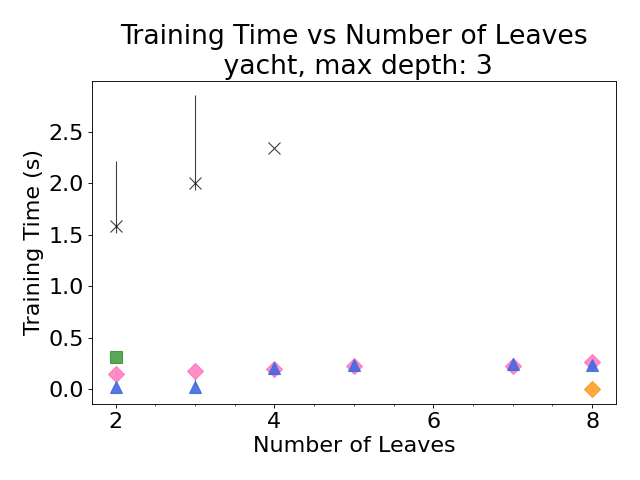}
    \includegraphics[width=0.41\textwidth]{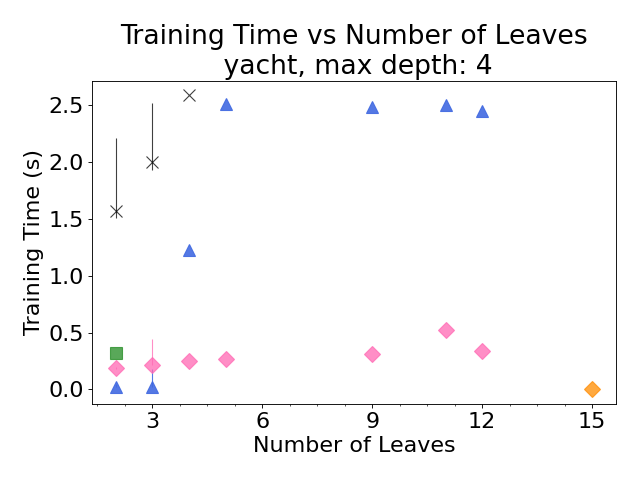}
    \includegraphics[width=0.41\textwidth]{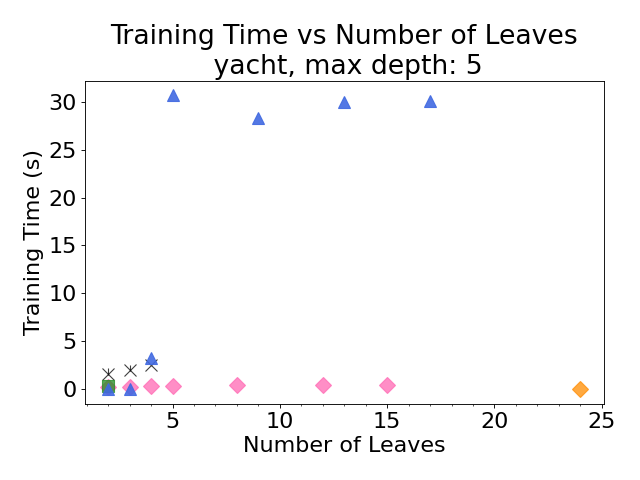}
    \includegraphics[width=0.41\textwidth]{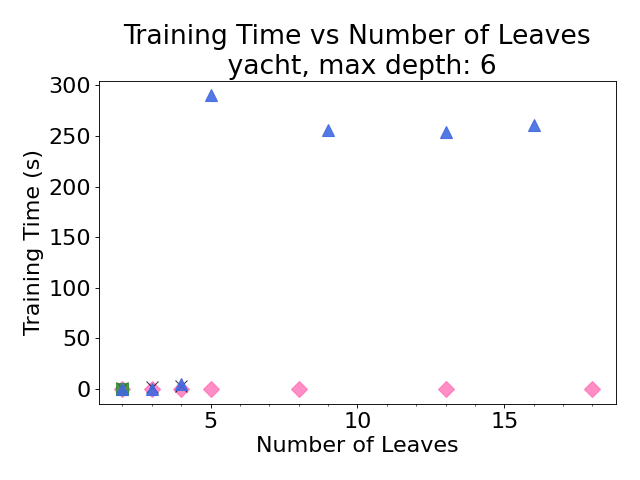}
    \includegraphics[width=0.41\textwidth]{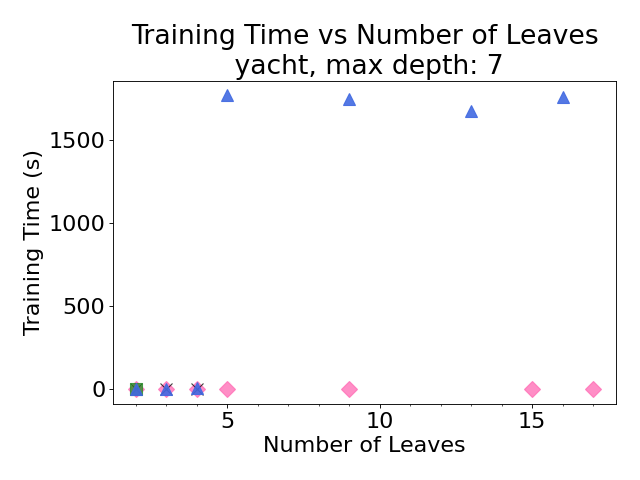}
    \includegraphics[width=0.41\textwidth]{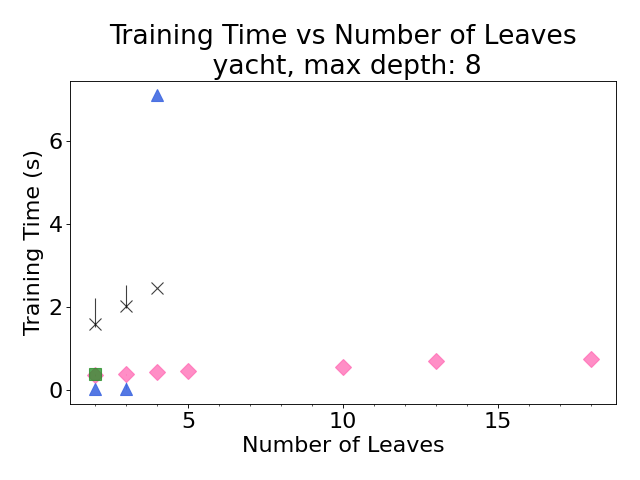}
    \includegraphics[width=0.41\textwidth]{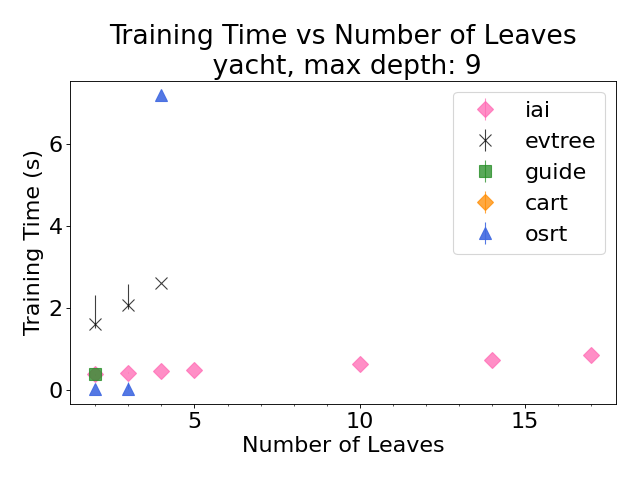}
    \caption{Training time of OSRT, IAI, Evtree, CART, GUIDE, as a function of number of leaves, on dataset: yacht}
    \label{fig:tvs:yacht}
\end{figure*}

%% file: exp_fig_code/cv.tex
\begin{figure*}[htbp]
    \centering
    \includegraphics[width=0.4\textwidth]{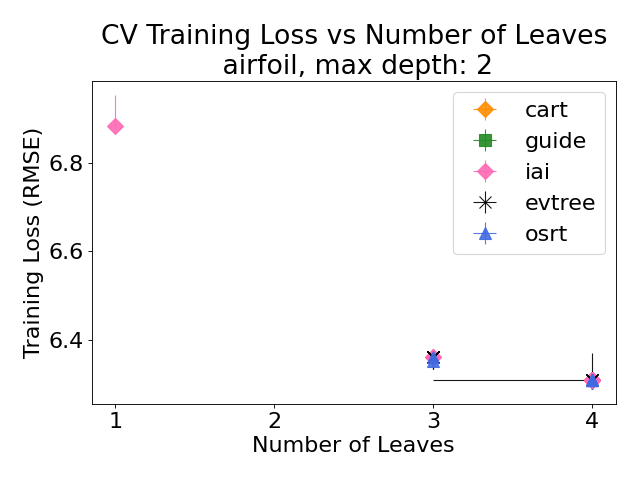}
    \includegraphics[width=0.4\textwidth]{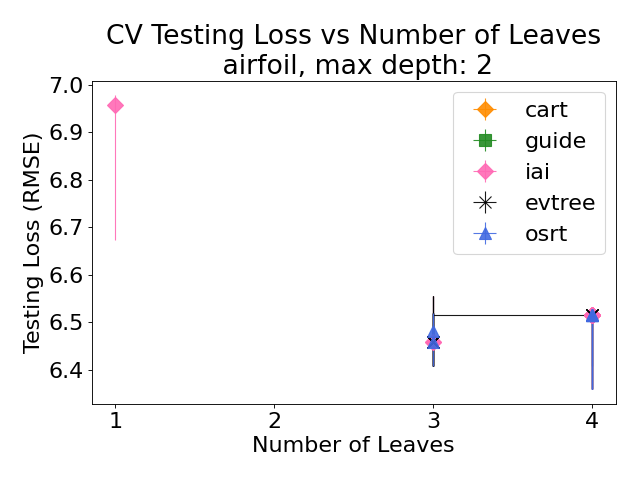}
    \includegraphics[width=0.4\textwidth]{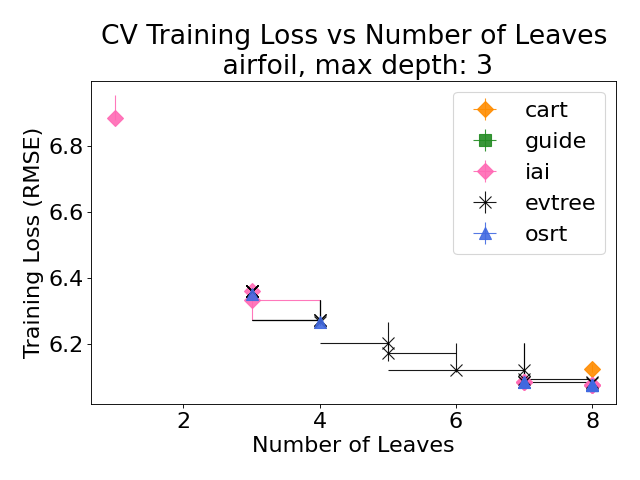}
    \includegraphics[width=0.4\textwidth]{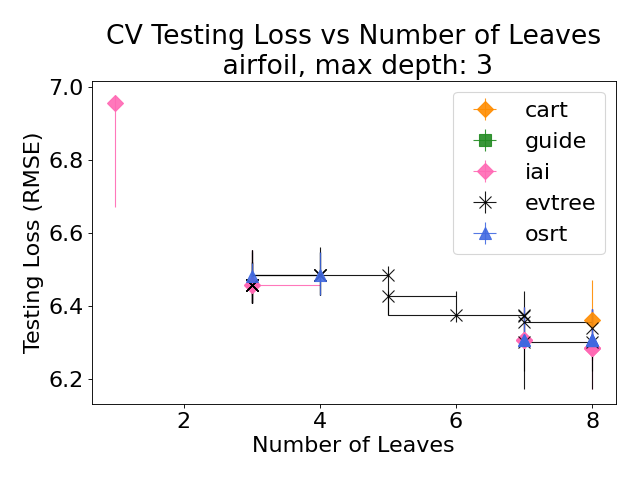}
    \includegraphics[width=0.4\textwidth]{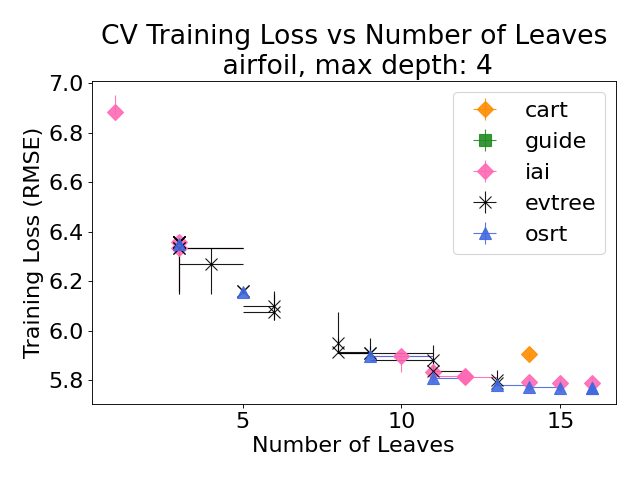}
    \includegraphics[width=0.4\textwidth]{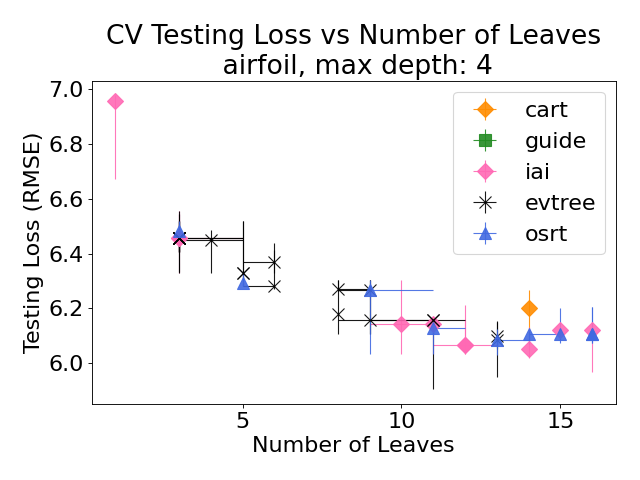}
    \includegraphics[width=0.4\textwidth]{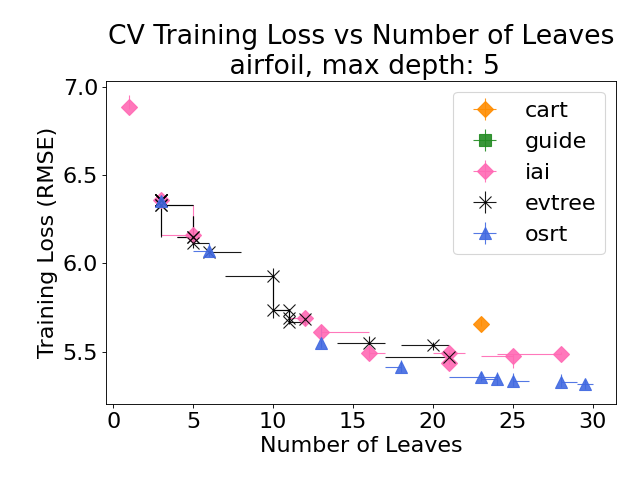}
    \includegraphics[width=0.4\textwidth]{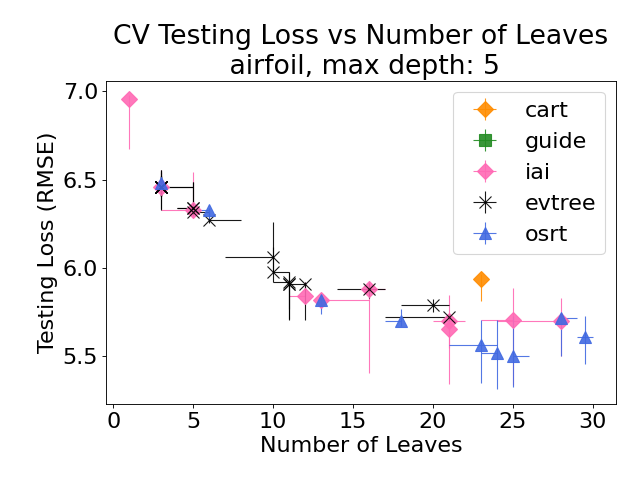}
    
    \caption{5-fold CV of OSRT, IAI, Evtree, CART, GUIDE as a function of number of leaves on dataset: airfoil}
    \label{fig:cv:airfoil}
\end{figure*}

\begin{figure*}[htbp]
    \centering
    \includegraphics[width=0.4\textwidth]{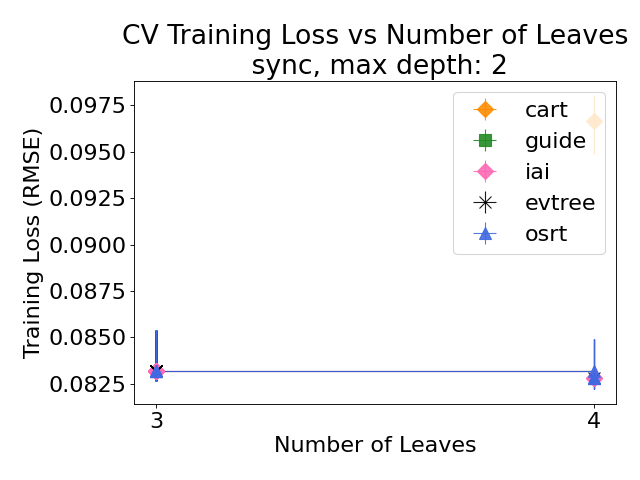}
    \includegraphics[width=0.4\textwidth]{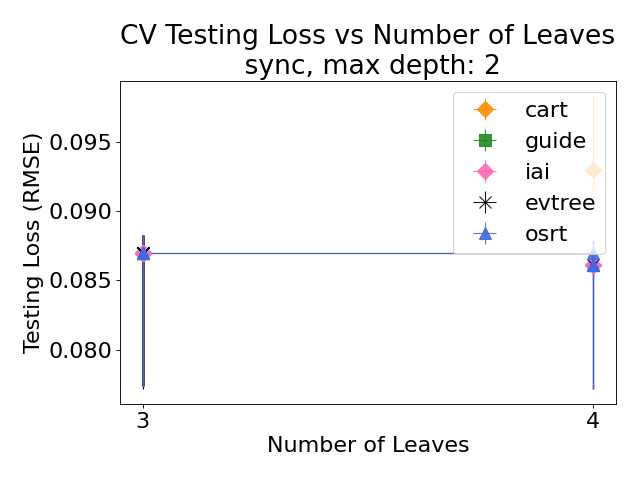}
    \includegraphics[width=0.4\textwidth]{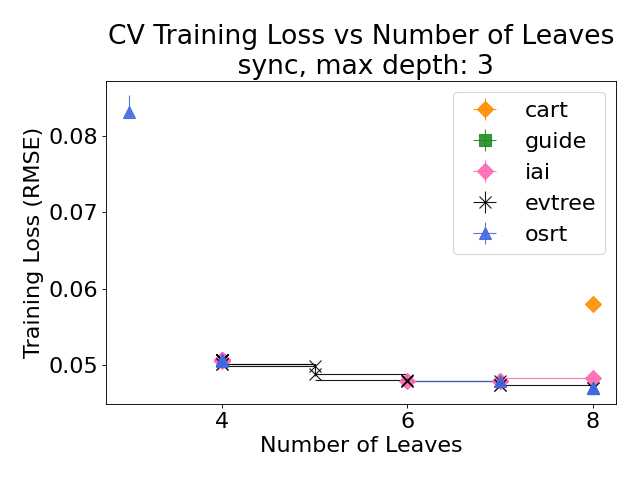}
    \includegraphics[width=0.4\textwidth]{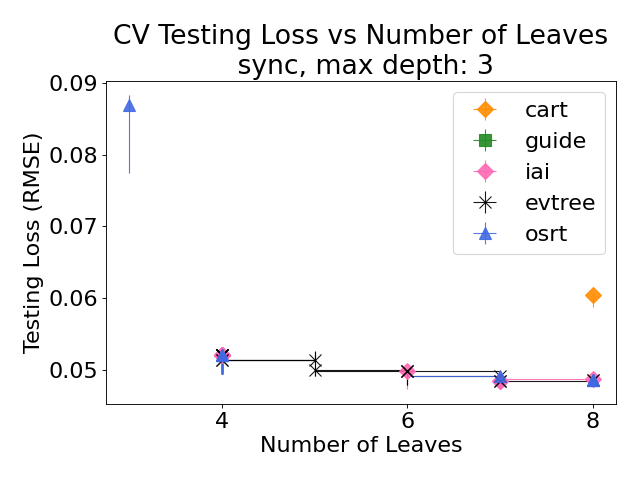}
    \includegraphics[width=0.4\textwidth]{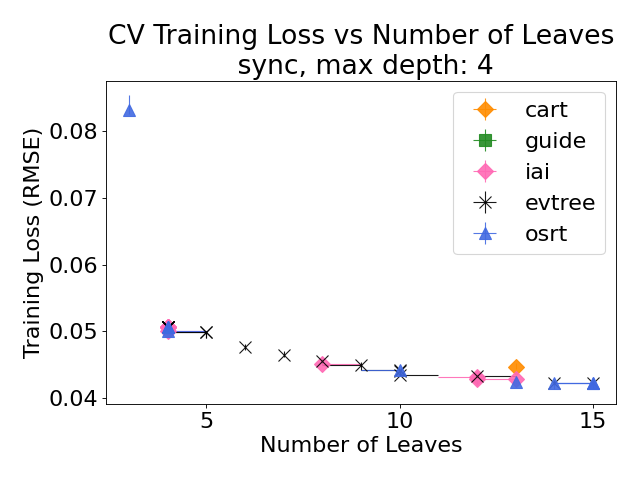}
    \includegraphics[width=0.4\textwidth]{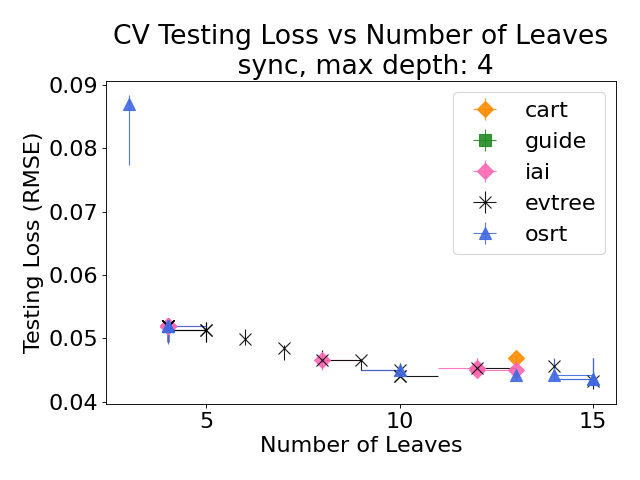}
    \includegraphics[width=0.4\textwidth]{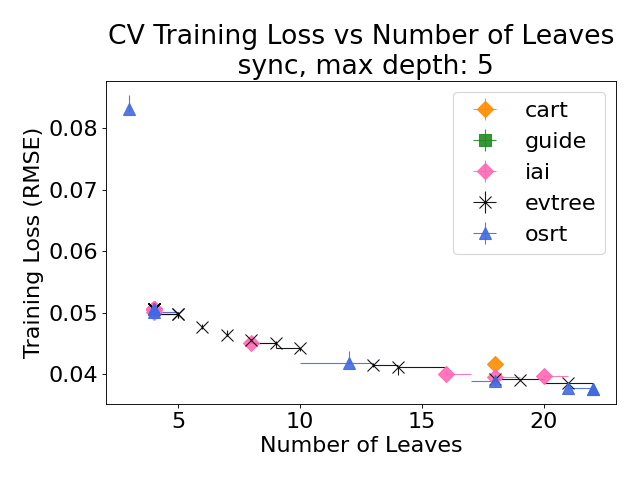}
    \includegraphics[width=0.4\textwidth]{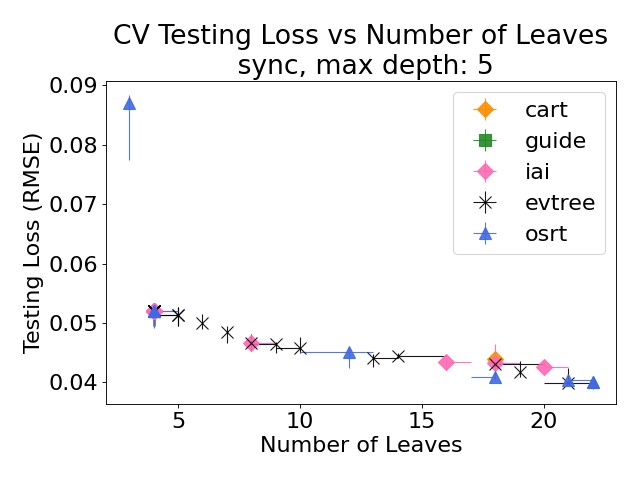}
    
    \caption{5-fold CV of OSRT, IAI, Evtree, CART, GUIDE as a function of number of leaves on dataset: sync}
    \label{fig:cv:sync}
\end{figure*}

\begin{figure*}[htbp]
    \centering
    \includegraphics[width=0.4\textwidth]{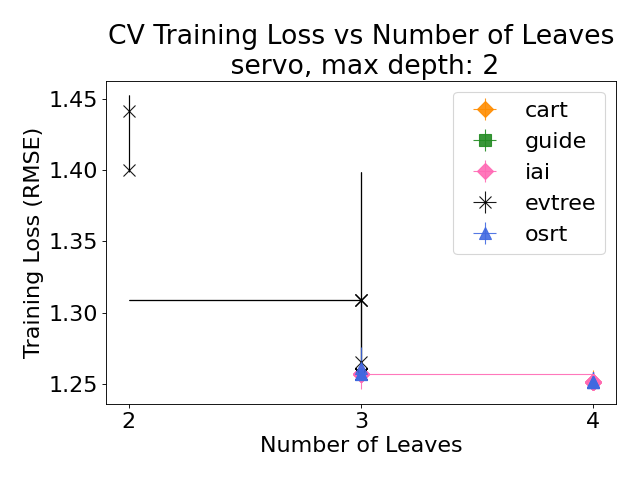}
    \includegraphics[width=0.4\textwidth]{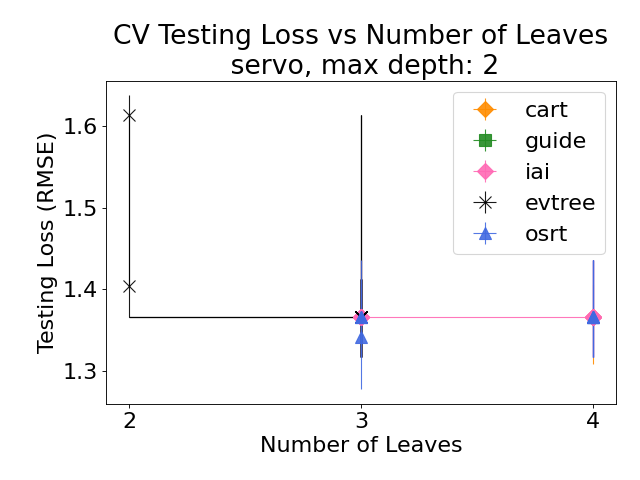}
    \includegraphics[width=0.4\textwidth]{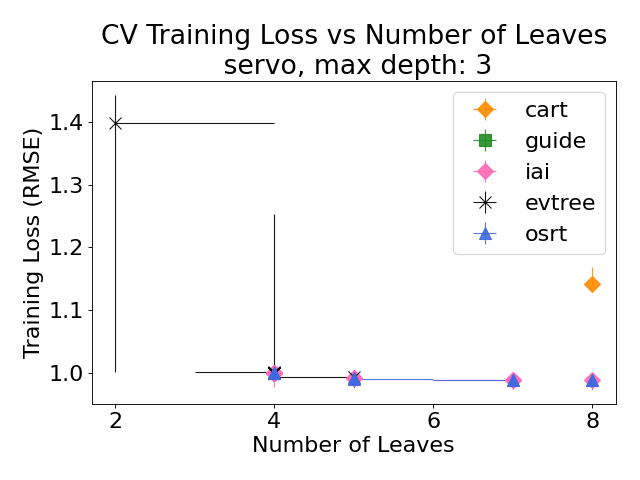}
    \includegraphics[width=0.4\textwidth]{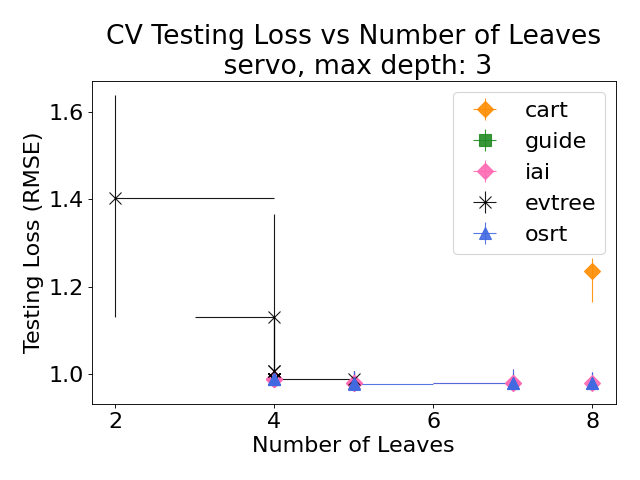}
    \includegraphics[width=0.4\textwidth]{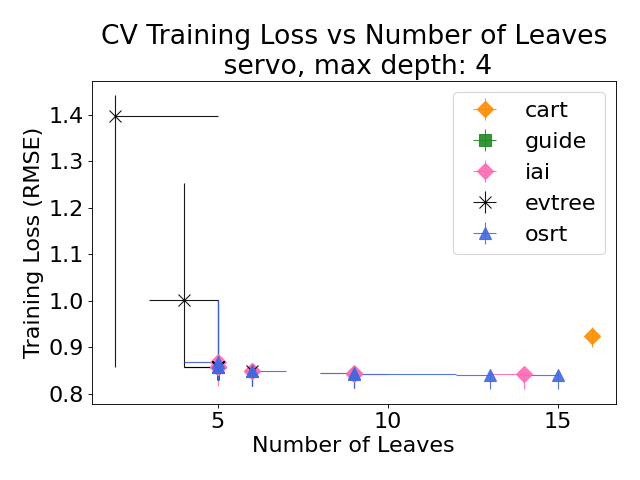}
    \includegraphics[width=0.4\textwidth]{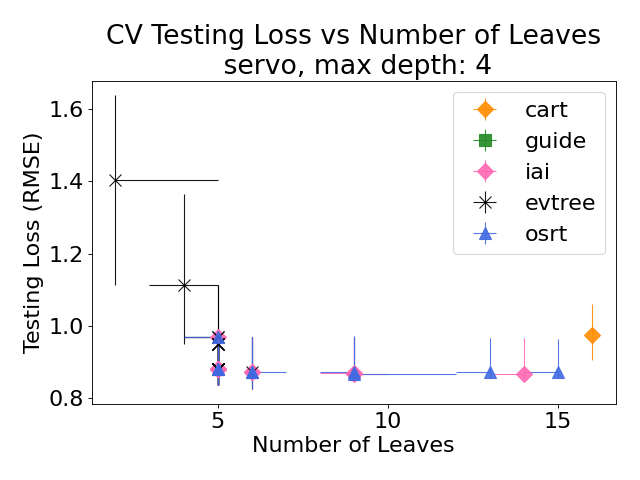}
    \includegraphics[width=0.4\textwidth]{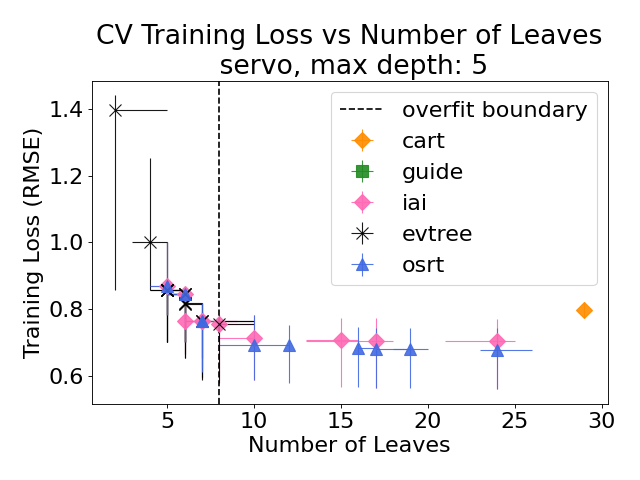}
    \includegraphics[width=0.4\textwidth]{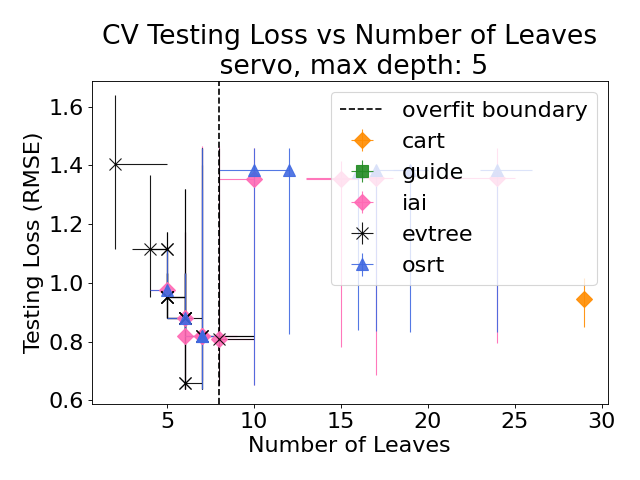}
    
    \caption{5-fold CV of OSRT, IAI, Evtree, CART, GUIDE as a function of number of leaves on dataset: servo}
    \label{fig:cv:servo}
\end{figure*}

\begin{figure*}[htbp]
    \centering
    \includegraphics[width=0.4\textwidth]{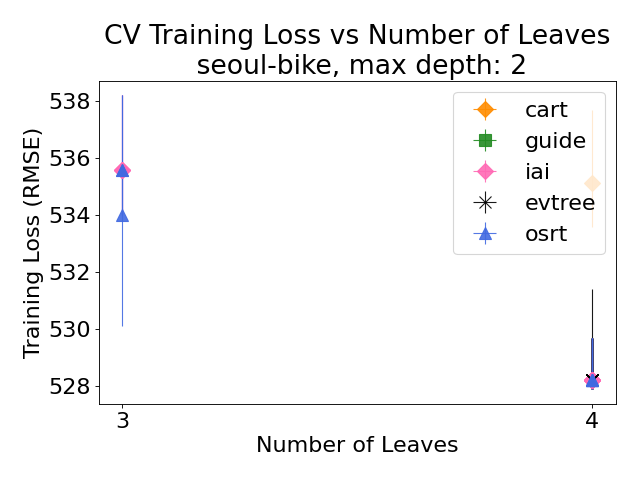}
    \includegraphics[width=0.4\textwidth]{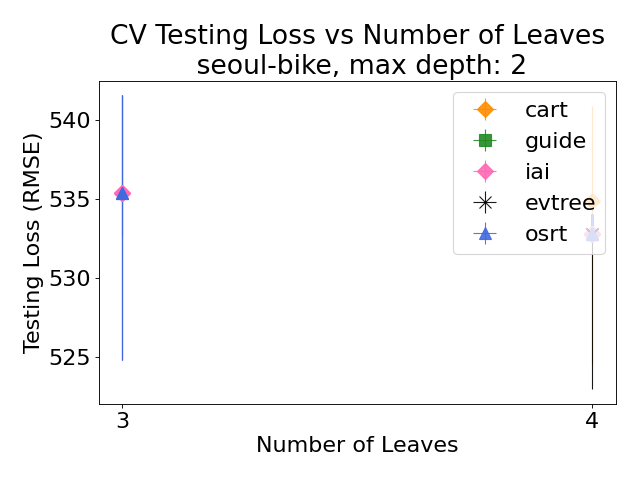}
    \includegraphics[width=0.4\textwidth]{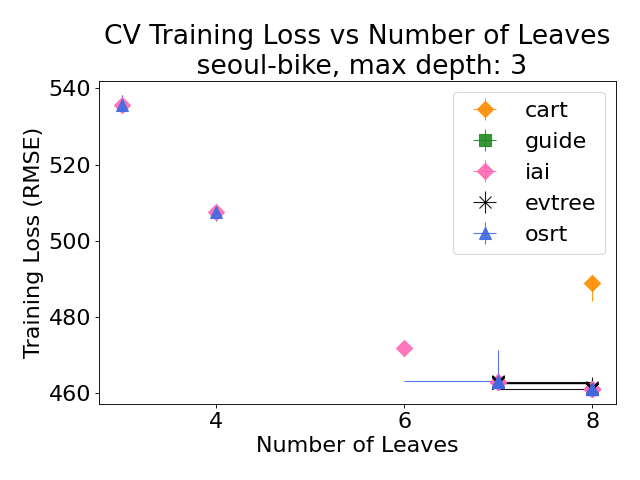}
    \includegraphics[width=0.4\textwidth]{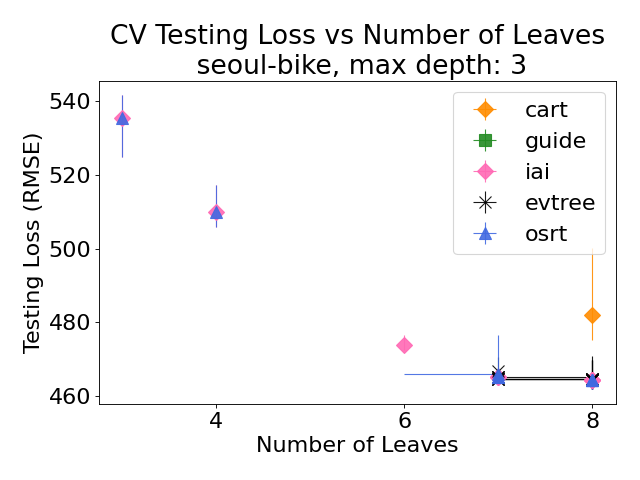}
    \includegraphics[width=0.4\textwidth]{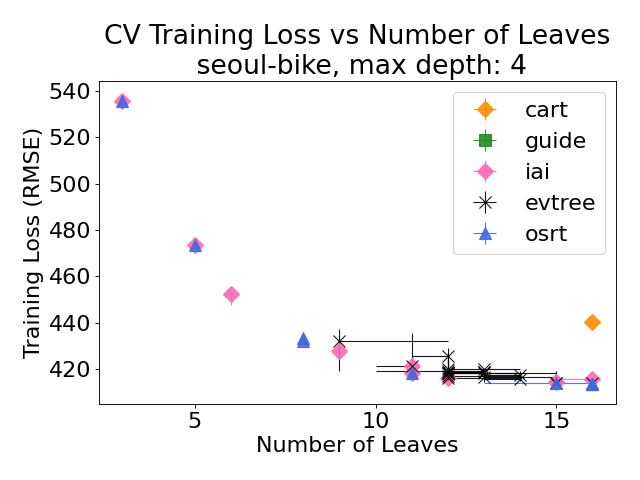}
    \includegraphics[width=0.4\textwidth]{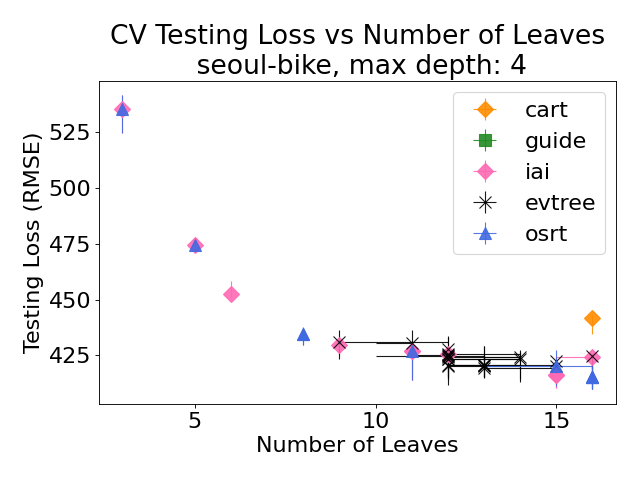}
    \includegraphics[width=0.4\textwidth]{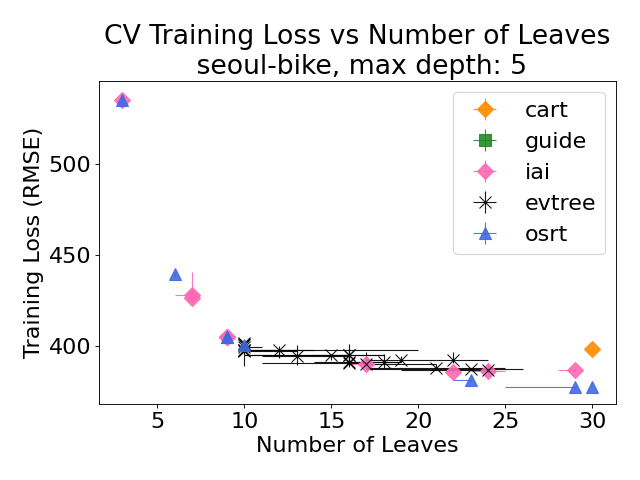}
    \includegraphics[width=0.4\textwidth]{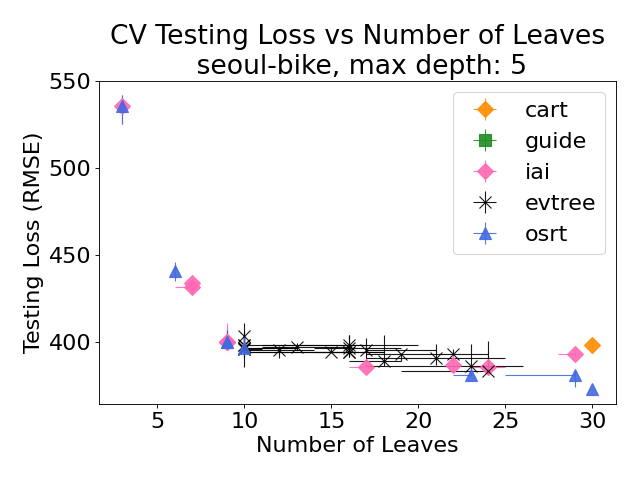}
    
    \caption{5-fold CV of OSRT, IAI, Evtree, CART, GUIDE as a function of number of leaves on dataset: seoul-bike}
    \label{fig:cv:seoul-bike}
\end{figure*}

\begin{figure*}[htbp]
    \centering
    \includegraphics[width=0.4\textwidth]{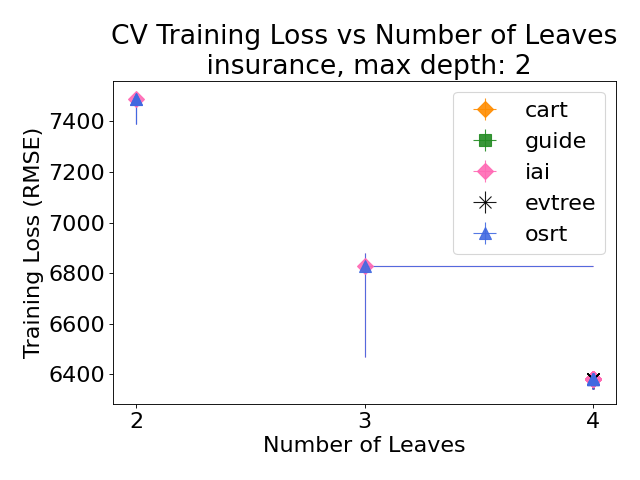}
    \includegraphics[width=0.4\textwidth]{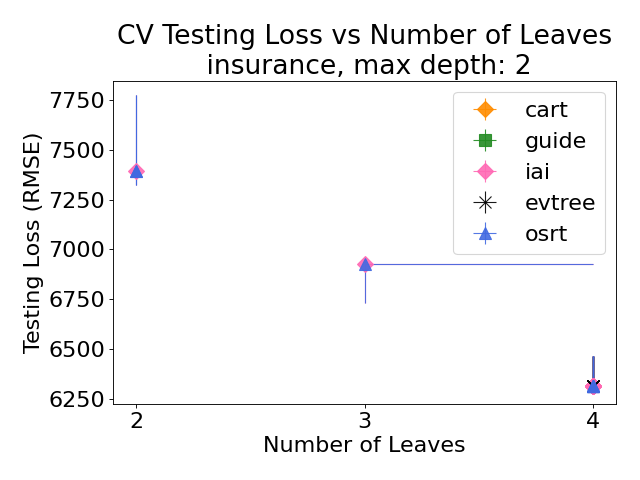}
    \includegraphics[width=0.4\textwidth]{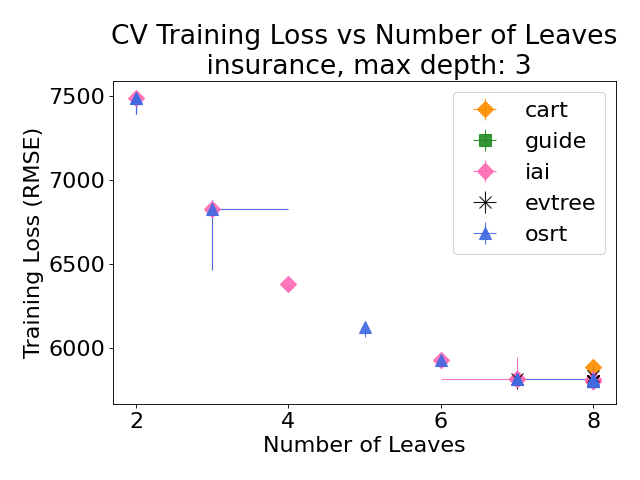}
    \includegraphics[width=0.4\textwidth]{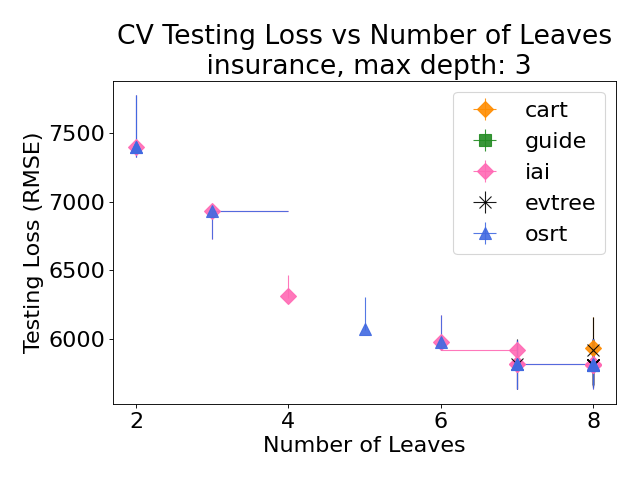}
    \includegraphics[width=0.4\textwidth]{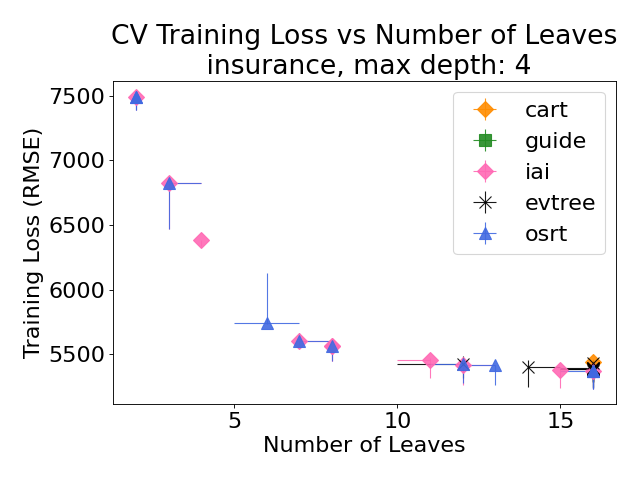}
    \includegraphics[width=0.4\textwidth]{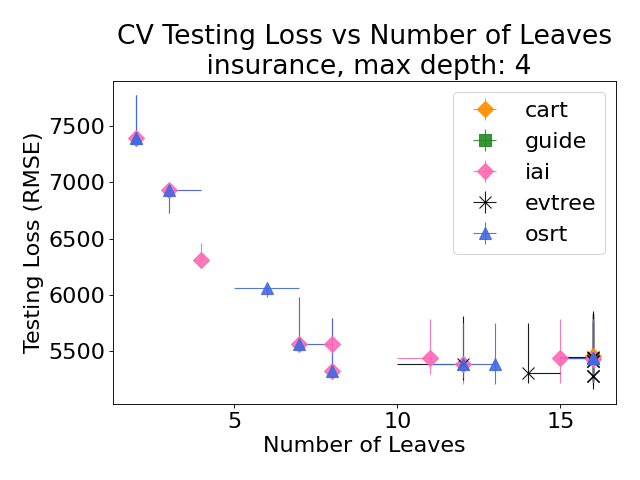}
    \includegraphics[width=0.4\textwidth]{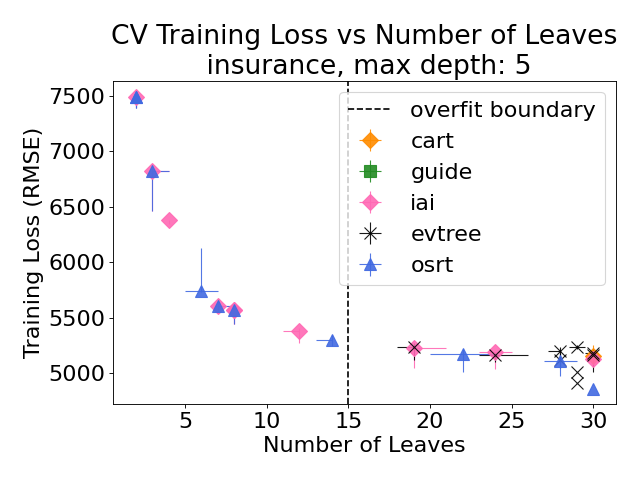}
    \includegraphics[width=0.4\textwidth]{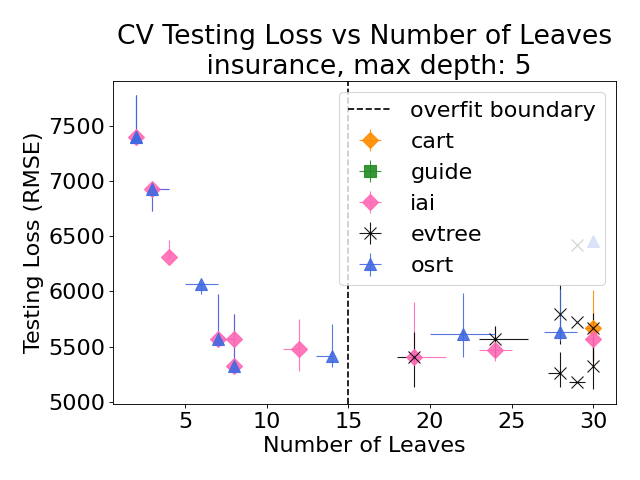}
    
    \caption{5-fold CV of OSRT, IAI, Evtree, CART, GUIDE as a function of number of leaves on dataset: insurance}
    \label{fig:cv:insurance}
\end{figure*}
    

%% file: exp_fig_code/ablation_lb.tex
\begin{figure*}[ht]
    \centering
    \includegraphics[width=0.45\textwidth]{figures/binary_settings/ablation/airfoil/depth6_reg_0.005.png}
    \includegraphics[width=0.45\textwidth]{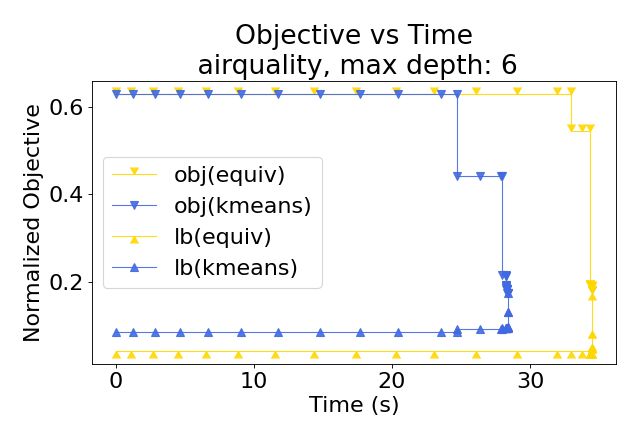}
    \includegraphics[width=0.45\textwidth]{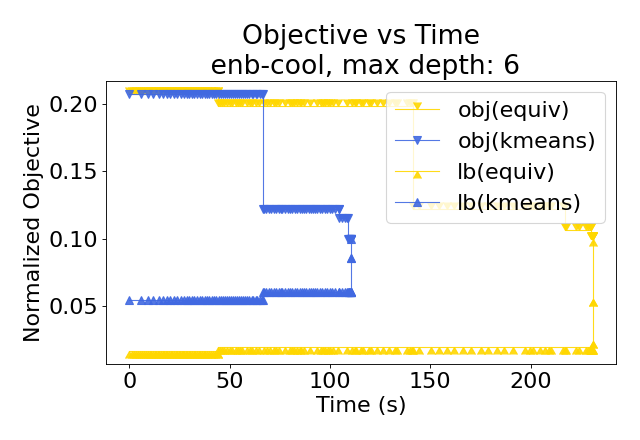}
    \includegraphics[width=0.45\textwidth]{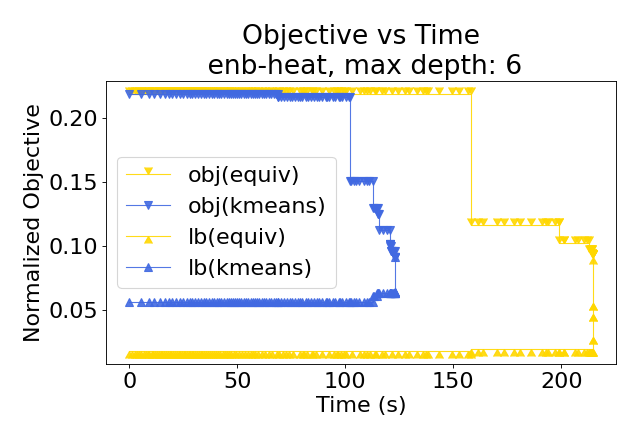}
    \includegraphics[width=0.45\textwidth]{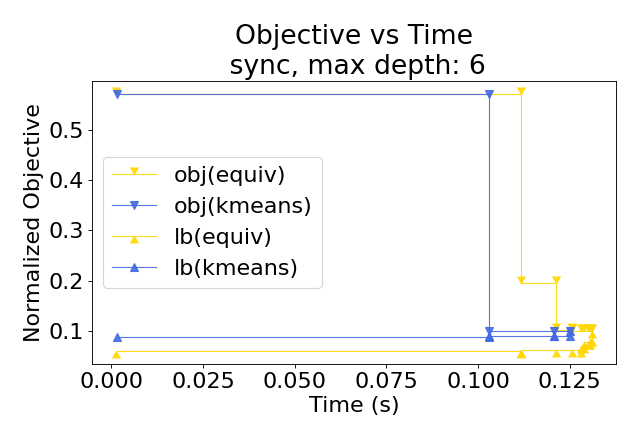}
    \includegraphics[width=0.45\textwidth]{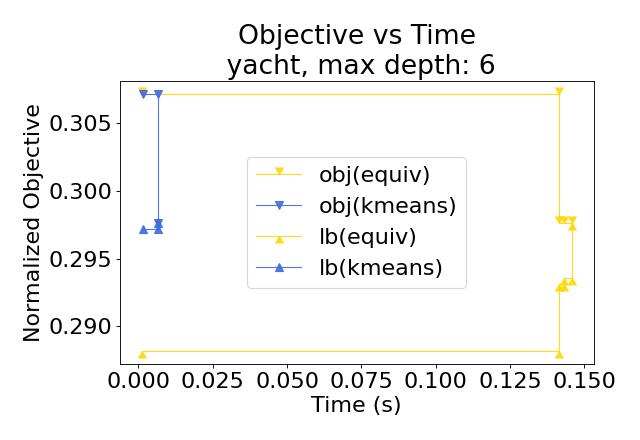}
    
    \caption{Convergence Time of OSRT variations: we verified \textit{k-Means lower bound }has fewer iterations to completion than the \textit{equivalent points lower bound} on Datasets \textit{sync, yacht}.}
    \label{fig:ablation:lb}
\end{figure*}

%% file: exp_fig_code/execution_trace.tex
\begin{figure*}[htbp]
    \centering
    \includegraphics[width=0.45\textwidth]{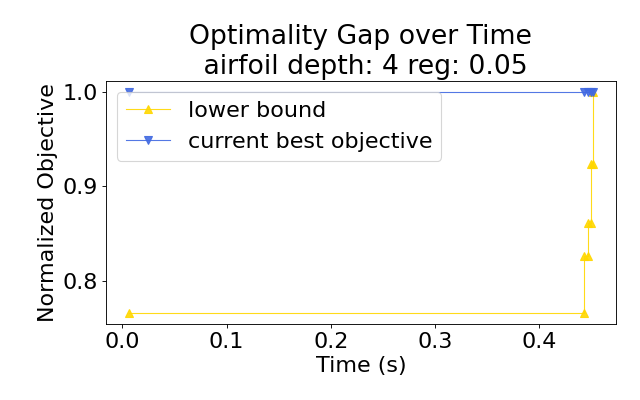}
    \includegraphics[width=0.45\textwidth]{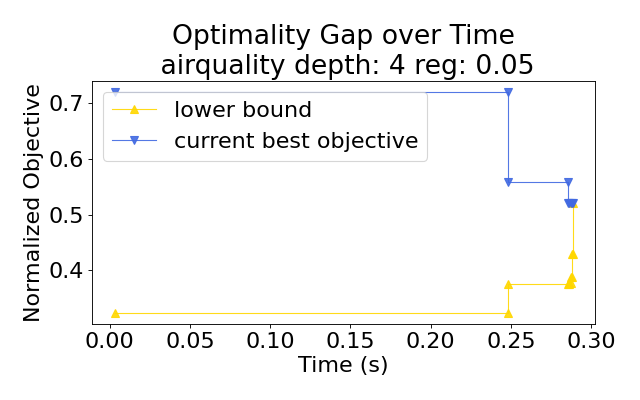}
    \includegraphics[width=0.45\textwidth]{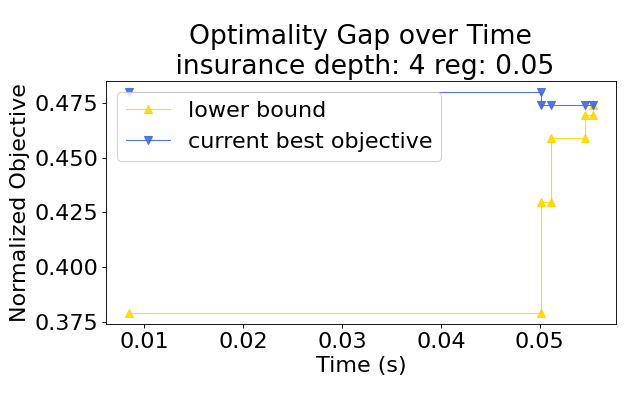}
    \includegraphics[width=0.45\textwidth]{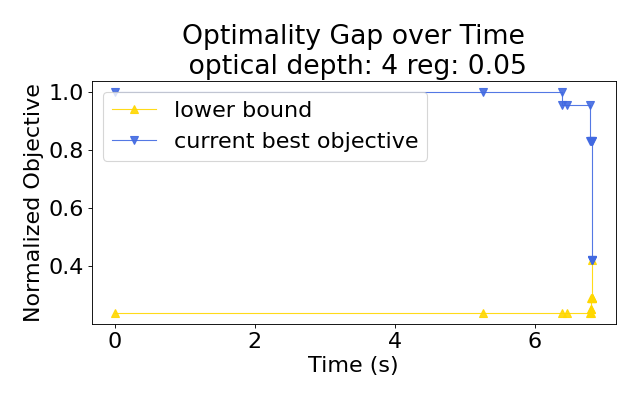}
    \includegraphics[width=0.45\textwidth]{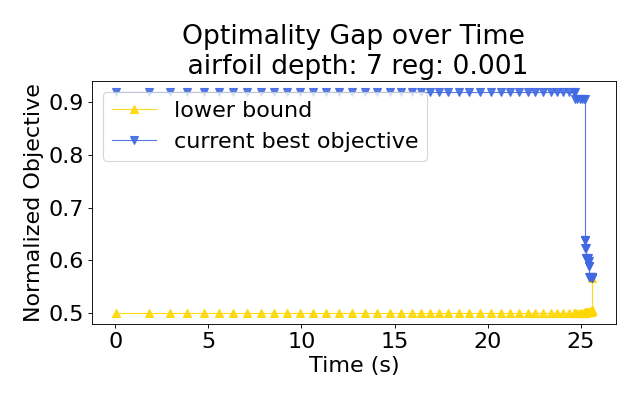}
    \includegraphics[width=0.45\textwidth]{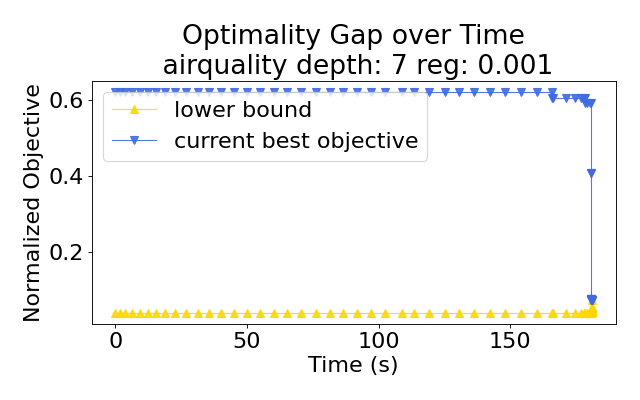}
    \includegraphics[width=0.45\textwidth]{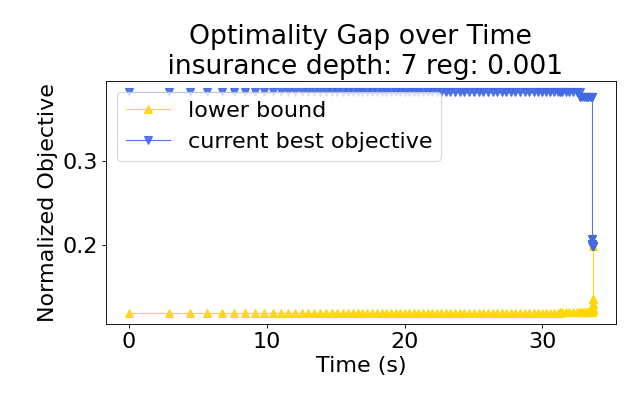}
    \includegraphics[width=0.45\textwidth]{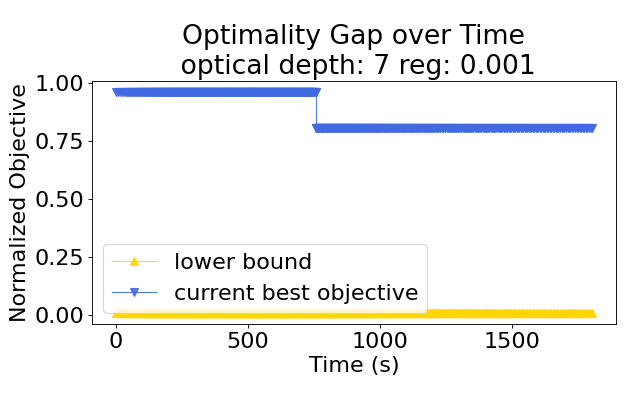}
    \caption{Execution Trace of OSRT over time, timed out on dataset \textit{optical} with small regularization.}
    \label{fig:trace}
\end{figure*}